\documentclass{article}

 \usepackage[main, final]{neurips_2026}

\usepackage[utf8]{inputenc} 
\usepackage[T1]{fontenc}    
\usepackage{hyperref}       
\usepackage{url}            
\usepackage{booktabs}       
\usepackage{amsfonts}       
\usepackage{nicefrac}       
\usepackage{microtype}      
\usepackage{xcolor}         

\usepackage{amsmath}
\usepackage{amssymb}
\usepackage{mathtools}
\usepackage{amsthm}

\usepackage[capitalize,noabbrev,nameinlink]{cleveref}

\theoremstyle{plain}

\newtheorem{lemma}{Lemma}

\theoremstyle{definition}

\newtheorem{definition}{Definition}

\theoremstyle{remark}
\newtheorem*{remark}{Remark}

\usepackage[textsize=tiny]{todonotes}

\usepackage{amsmath,amsfonts,bm}

\def\eqref#1{equation~\ref{#1}}

\def\1{\bm{1}}

\def\rvx{{\mathbf{x}}}

\def\rmA{{\mathbf{A}}}

\def\rmD{{\mathbf{D}}}

\def\rmI{{\mathbf{I}}}

\def\rmW{{\mathbf{W}}}
\def\rmX{{\mathbf{X}}}

\DeclareMathAlphabet{\mathsfit}{\encodingdefault}{\sfdefault}{m}{sl}
\SetMathAlphabet{\mathsfit}{bold}{\encodingdefault}{\sfdefault}{bx}{n}

\def\gA{{\mathcal{A}}}

\def\gD{{\mathcal{D}}}

\def\gF{{\mathcal{F}}}

\def\gI{{\mathcal{I}}}

\def\gL{{\mathcal{L}}}

\def\gR{{\mathcal{R}}}

\def\gW{{\mathcal{W}}}
\def\gX{{\mathcal{X}}}
\def\gY{{\mathcal{Y}}}
\def\gZ{{\mathcal{Z}}}

\newcommand{\E}{\mathbb{E}}

\DeclareMathOperator*{\argmax}{arg\,max}

\newcommand{\trainr}{R_{m,\gamma}(f\circ\phi;\pi)}
\newcommand{\testr}{R_{u}(f\circ\phi;\pi)}

\newcommand{\traini}{\gI_{\mathrm{train}}^{(\pi)}}
\newcommand{\testi}{\gI_{\mathrm{test}}^{(\pi)}}

\newcommand{\classWD}{\gW_C}
\newcommand{\gammaWD}{\gW_S}

\usepackage{amsmath,amssymb,amsfonts}
\usepackage{amsthm}
\usepackage{dsfont}
\usepackage{enumitem}
\usepackage{thm-restate}
\usepackage[most]{tcolorbox}
\usepackage{xcolor}
\usepackage{colortbl}
\usepackage{wrapfig}
\usepackage{pgfplots}
\usepackage{tikz}
\pgfplotsset{compat=1.18}
\usepgfplotslibrary{fillbetween, groupplots}
\usepackage{xcolor}
\tcbuselibrary{skins, breakable}

\definecolor{colorGlobal}{HTML}{4A90D9}
\definecolor{colorOracle}{HTML}{E85A5A}
\definecolor{colorSep}{HTML}{8E7CC3}
\definecolor{colorGen}{HTML}{27AE60}

\usepackage{wrapfig}

\usepackage{amsmath,amssymb,amsthm}
\usepackage{mathrsfs}
\usepackage{bbm}
\usepackage{enumitem}
\usepackage{graphicx}
\usepackage{hyperref}
\usepackage{booktabs}

\usepackage{mathtools}
\usepackage{duckuments}
\usepackage{soul}
\usepackage{multirow}
\usepackage{cancel}
\usepackage{tabularx}

\definecolor{BrickRed}{rgb}{0.6,0,0}
\definecolor{RoyalBlue}{rgb}{0,0,0.8}
\definecolor{Tdgreen}{rgb}{0,0.4,0.7}
\definecolor{cadmiumgreen}{rgb}{0.0, 0.42, 0.24}
\hypersetup{colorlinks, citecolor=RoyalBlue, linkcolor=RoyalBlue}

\title{Transductive Generalization for GNNs \\ via Optimal Transport}

\author{%
  MoonJeong Park$^{1}$\thanks{Equal contribution.} \quad
  Seungbeom Lee$^{1}$\footnotemark[1] \quad
  Kyungmin Kim$^{1}$ \quad
  Jaeseung Heo$^{1}$ \\
  \textbf{Seunghyuk Cho}$^{1}$ \quad
  \textbf{Shouheng Li}\thanks{Work done while at CSIRO’s Data61, Canberra, Australia.} \quad
  \textbf{Sangdon Park}$^{1,2}$ \quad
  \textbf{Dongwoo Kim}$^{1,2}$\thanks{Corresponding author: \texttt{dongwoo.kim@postech.ac.kr}} \\
  $^{1}$Graduate School of Artificial Intelligence, POSTECH, South Korea \\
  $^{2}$Department of Computer Science \& Engineering, POSTECH, South Korea
}

\begin{document}

\maketitle

\begin{abstract}
    Analyzing the generalization of graph neural networks (GNNs) on node classification is challenging: message passing induces dependencies that break the i.i.d. assumption underlying standard inductive generalization theory. While distribution-free transductive theory resolves this dependency issue, existing classical bounds are either intractable or correlate poorly with empirical generalization. 
    To fill this gap, we propose a transductive generalization bound based on optimal transport, connecting generalization to intra-class concentration and inter-class separation in the feature space. Empirically, our bound aligns strongly with the generalization gap. Building on this representation-based approach, we analyze how feature geometry evolves under repeated message passing. This explains the non-monotonic relationship between GNN depth and generalization, providing new insights into oversmoothing research. Practically, our bound serves as a geometry-aware regularizer that corrects the harmful behavior of message passing, yielding consistent performance improvements. The code is available at \url{https://github.com/ml-postech/Transductive-OT-Gen-Bound}.
\end{abstract}

\section{Introduction}

Graph neural networks (GNNs) achieve strong empirical performance across domains on graph-structured data, yet their generalization behavior remains insufficiently characterized~\citep{tang2023towards}. This challenge is particularly evident in node classification, where message passing induces dependencies among node representations through the graph structure. These dependencies break the standard inductive assumption that data points are independent and identically distributed (i.i.d.). Since the i.i.d. assumption underlies much of classical learning theory and many generalization bounds~\citep{bartlett2002rademacher,von2011statistical,chuang2021measuring}, their direct application to node classification is challenging.

A natural resolution lies in the transductive learning framework. In the transductive setting, the learner has access to the test features in the fixed dataset during training, and aims to predict the labels of the unlabeled test points. Rather than assuming a specific underlying data distribution, transductive theory yields distribution-free guarantees by operating directly on the observed dataset and, therefore, remains valid even when representations are dependent~\citep{el2009transductive}. Since node classification is typically transductive, bounds derived in this setting offer a principled lens for analyzing GNN generalization in the presence of graph-induced dependencies.

However, existing transductive bounds for GNNs predominantly rely on classical complexity measures, such as Rademacher complexity~\citep{el2009transductive}, VC-dimension~\citep{pmlr-v35-tolstikhin14}, and PAC-Bayesian analysis~\citep{pac2014begin}. 
These measures, in their general form (e.g., transductive Rademacher complexity), are computationally intractable for modern architectures. Even their computable specializations (e.g., the GNN-tailored Rademacher bound of~\citet{EsserNEURIPS2021_learning}) often correlate weakly or even negatively with the empirical generalization gap~\citep{jiang2020neuripscomp}.
While the broader learning-theory community has developed representation-based measures that track empirical generalization more reliably~\citep{natekar2020representationcompwin,chuang2021measuring}, these advances have remained largely confined to the inductive setting. 

We close this gap by proposing a distribution-free transductive generalization bound based on optimal transport. Our bound is formulated through the empirical $1$-Wasserstein distance between encoded feature distributions, extending the inductive concentration-separation principle of~\citet{chuang2021measuring} to the transductive regime where representations may be dependent. Beyond inheriting the empirical predictive power of representation-based measures, our bound is tailored to GNNs in a substantive sense: because message passing reshapes feature geometry through neighbor aggregation, the Wasserstein quantities in our bound directly capture how this aggregation affects intra-class concentration and inter-class separation.

Building on this geometric characterization, we use our bound to resolve the relationship between GNN depth and generalization. We derive depth-dependent upper bounds on the Wasserstein distance between feature distributions of arbitrary node subsets, characterizing how message passing transforms feature geometry as layers are stacked. The analysis shows that depth simultaneously strengthens intra-class concentration and weakens inter-class separation, producing a fundamental geometric trade-off. Our bound thereby links the well-studied oversmoothing phenomenon to generalization, distinguishing the beneficial and harmful effects of oversmoothing. Unlike prior depth-dependent transductive bounds based on uniform stability~\citep{congNEURIPS2021_provable}, which predict a generalization gap that grows monotonically with depth, our analysis explains the empirically observed non-monotonic relationship between depth and generalization.

We empirically validate our framework along three axes. First, our bound correlates with the empirical generalization gap on GNN node classification across nine datasets and five architectures, substantially outperforming classical transductive baselines. Second, we apply the bound as a training regularizer that directly encourages intra-class concentration and smoothness of the classifier in feature space, yielding consistent accuracy gains over standard cross-entropy training, with the largest improvements on heterophilic graphs where neighbor mixing is most damaging. Third, we empirically confirm the predicted depth trade-off: our bound tracks the non-monotonic generalization gap as GNN depth increases, while existing bounds do not.
\section{Related work}
\paragraph{Generalization bounds and representation measures}
Several works have studied generalization guarantees in the distribution-free transductive setting, typically deriving model-agnostic bounds using classical complexity tools such as transductive Rademacher complexity~\citep{el2009transductive}, stability~\citep{el2006stable, trans_stab_2008}, PAC-Bayesian frameworks~\citep{pac2014begin}, or VC dimension~\citep{pmlr-v35-tolstikhin14}. However, these classical measures are often computationally intractable and struggle to reliably track observed test errors in modern deep learning~\citep{jiang2020neuripscomp}. To address these limitations, representation-based complexity measures have emerged as powerful alternatives that better capture empirical generalization~\citep{natekar2020representationcompwin}. Most notably, \citet{chuang2021measuring} developed a rigorous margin-based bound using optimal transport and $k$-variance~\citep{solomon2022k} to account for the structural properties of learned feature distributions. Despite their empirical success, the theoretical advancement of these representation-based measures has been largely confined to inductive settings, leaving a critical gap for transductive learning scenarios.

\paragraph{Generalization in graph neural networks}
Because graph node classification is a canonical transductive problem, multiple studies have specialized classical transductive bounds for GNNs. These works typically upper-bound transductive Rademacher complexity or algorithmic stability using architecture- and training-dependent quantities, such as normalized adjacency matrices or diffusion operators~\citep{Oono2020Graph, EsserNEURIPS2021_learning, tang2023towards}. Notably, \citet{congNEURIPS2021_provable} utilized stability-based theory to bound generalization error with respect to GNN depth; however, their bound grows monotonically with depth, contradicting the non-monotonic and often decreasing behavior observed empirically. Another line of work analyzes node classification under independence assumptions, ignoring the dependencies among nodes~\citep{SCARSELLI2018248, verma2019stab, zhang20y_fast, ZHOU202197, pmlr-v119-garg20c}. 
A separate line of work studies the generalization of GNNs in the graph classification setting~\citep{liao2021pac, duNEURIPS2019fusing, maskey2022generalization, levieNEURIPS2023graphon, rauchwerger2024generalization, li2025towards}, which differs fundamentally from node classification and falls outside the scope of this work. Because individual graphs in graph classification are generally independent and identically distributed, inductive representation bounds can be naturally applied without the challenges posed by node dependencies.

\section{Preliminaries}

\paragraph{Generalization bound in transductive learning} 
In transductive learning, the learner has access to the features of both training and test points at training time, while labels are available only for the training points. This contrasts with standard inductive learning, where test features are not accessible during training. 
A formal definition of a distribution-free transductive setting is provided by \citet{vapnik2006estimation}.
Consider a fixed dataset $\gD=\{(\rvx_i,y_i)\}_{i=1}^{m+u}$ of $m+u$ data points $\rvx_i\in\mathbb{R}^F$ and labels $y_i\in\{1,\cdots,K\}$, where $K$ denotes the number of classes. Using a random permutation $\pi:\{1,\cdots,m+u\}\rightarrow\{1,\cdots,m+u\}$, the training set is determined as 
$
    \gD_{\mathrm{train}} 
    = 
    \{ (\rvx_i, y_i): i \in \gI_{\mathrm{train}}^{(\pi)} \}
$ 
and test set as 
$
    \gD_{\mathrm{test}} 
    = 
    \{ (\rvx_i, y_i): i \in \gI_{\mathrm{test}}^{(\pi)} \}
$.
Here, $\gI_{\mathrm{train}}^{(\pi)}\coloneq\{\pi(i)\}_{i=1}^{m}$ and $\gI_{\mathrm{test}}^{(\pi)}\coloneq\{\pi(i)\}_{i=m+1}^{m+u}$ are index sets for training and test sets, respectively.
During training, the learner has access to the full unlabeled samples $\{\rvx_i\}_{i=1}^{m+u}$ and the labels $\{y_i:i\in \gI_{\mathrm{train}}^{(\pi)}\}$. 
A transductive generalization bound aims to bound the test error $\frac{1}{u}\sum_{i \in \gI_{\mathrm{test}}^{(\pi)}}\ell(\rvx_i,y_i)$ for any permutation $\pi$, where $\ell$ is the loss function.

\paragraph{Graph neural networks}

Consider an undirected graph $\mathcal{G}=(\rmX,\rmA)$ with $N$ nodes, where $\rmX\in\mathbb{R}^{N \times F}$ denotes a node feature matrix whose $i$-th row $\rvx_i \in \mathbb{R}^F$ is the $F$-dimensional feature of node $i$. The adjacency matrix $\rmA \in \{0,1\}^{N \times N}$ encodes the graph structure, where $\rmA_{ij} = 1$ if an edge exists between nodes $i$ and $j$, and $\rmA_{ij} = 0$ otherwise.
GNNs are characterized by an aggregation process that updates each node feature as a weighted average of its neighbors’ features, inducing dependencies among node features. 
For example, Graph Convolutional Networks (GCNs)~\citep{kipf2016semi} use the aggregation operator $\hat{\rmA}\rmX$, where $\hat{\rmA}\coloneq \tilde{\rmD}^{-1/2}\tilde{\rmA}\tilde{\rmD}^{-1/2}$, $\tilde{\rmA}\coloneq \rmA+\rmI$, and $\tilde{\rmD}\coloneq \mathrm{diag}(\tilde{\rmA}\,\mathbf{1})$. The $\ell$-th layer ($0\leq\ell<L$) updates node features as $\rmX^{(\ell+1)} = \sigma(\hat{\rmA}\rmX^{(\ell)}\rmW^{(\ell)}),$
where $\rmW^{(\ell)}$ is the weight parameter, $\sigma(\cdot)$ is a non-linear activation, and $\rmX^{(0)}=\rmX$.
Simple Graph Convolution (SGC)~\citep{wu2019simplifying} removes activation, resulting in
$\rmX^{(L)}=\hat{\rmA}^L\rmX^{(0)}\rmW.$

\paragraph{Wasserstein distance}
Given two probability distributions $\mu$ and $\nu$ on $\mathbb{R}^F$, the $p$-Wasserstein distance between $\mu$ and $\nu$ with Euclidean cost is defined as:
\begin{equation*}
    \mathcal{W}_p(\mu, \nu):=\inf _{T \in \mathcal U(\mu, \nu)}\left(\mathbb{E}_{(x, y) \sim {T}}\|x-y\|^p\right)^{1 / p},
\end{equation*}
where ${\mathcal U}(\mu, \nu)$ denotes the set of all couplings of $\mu$ and $\nu$, i.e., joint distributions ${T}$ on $\mathbb{R}^F \times \mathbb{R}^F$ with $\mu$ and $\nu$ as marginals. Intuitively, the Wasserstein distance measures the minimal cost of transporting the mass from the distribution $\mu$ to $\nu$. Throughout this work, unless stated otherwise, $\gW(\cdot,\cdot)$ denotes the $1$-Wasserstein distance and $\|\cdot\|$ the Euclidean norm.

\section{Transductive generalization bound via optimal transport}\label{sec:WB-Transductive}

In this section, 
we derive a generalization error bound in the transductive setting based on optimal transport. 
Our main result upper-bounds the generalization gap by Wasserstein distances between class-wise feature distributions, revealing the central role of intra-class concentration and inter-class separation. 
We then specialize the bound to GNN encoders, deriving depth-dependent expressions that characterize how message passing reshapes the feature geometry.

\subsection{Setup}

Let $\gX$ denote the input space, $\gZ$ the embedding space, and $\gY \coloneq \{1,\cdots,K\}$ the output space. Consider a compositional hypothesis class $\gF\circ\Phi$, 
where each hypothesis $f\circ \phi$ decomposes into a feature encoder\footnote{ While our setup applies to general encoders $\phi: \gX \rightarrow \gZ$, our primary focus is on GNN encoders $\phi: \gX\times(\gX^N,\gA) \rightarrow \gZ$, where the output depends on the feature matrix $\rmX$ and adjacency matrix $\rmA$. For simplicity, we omit the dependency and write $\phi(\rvx)$ instead of $\phi(\rvx; \rmX, \rmA)$.} $\phi\in\Phi$ with $\phi:\gX\rightarrow\gZ$ and a score-based classifier $f\in\gF$ with $f=[f_1,\cdots,f_K]:\gZ\rightarrow\mathbb{R}^K$.
The label of $\rvx\in\gX$ is predicted by $\argmax_{y\in\gY}f_y(\phi(\rvx))$. The margin of classifier $f$ for a data point $(\rvx,y)$ is defined by $\rho_f(\phi(\rvx),y)\coloneq f_{y}(\phi(\rvx))-\max_{y'\neq y} f_{y'}(\phi(\rvx))$, so $\rvx$ is misclassified if $\rho_f(\phi(\rvx),y)\leq 0$. We aim to bound the zero-one loss over the $u$ test samples, $R_u(f\circ\phi;\pi)\coloneq \frac{1}{u}\sum_{i\in\testi}\mathbbm{1}_{\rho_f(\phi(\rvx_i),y_i)\leq 0}$, by the $\gamma$-margin loss over the $m$ training samples, $R_{m,\gamma}(f\circ\phi;\pi)\coloneq \frac{1}{m}\sum_{i\in\traini}\mathbbm{1}_{\rho_f(\phi(\rvx_i),y_i)\leq\gamma}$, where $\gamma>0$ is a hyperparameter that sets the margin threshold. The $\gamma$-margin loss\footnote{$\gamma$-margin loss is introduced solely for theoretical analysis, not as a training loss.} is stricter than the zero-one loss, treating training points whose margin falls below $\gamma$ as errors even when they are correctly classified.



\subsection{Main theorem}
\label{subse:class-wise OT}
Intuitively, a model generalizes well when the encoded features are tightly clustered for each class and well-separated from other classes. \citet{chuang2021measuring} formalized this intuition in the inductive setting. We extend it to the transductive setting through a permutation-based analysis, yielding distribution-free guarantees despite feature dependencies. For each class $c \in \gY$, let $\mathcal{I}_{\mathrm{train},c}^{(\pi)}$ and $\mathcal{I}_{\mathrm{test},c}^{(\pi)}$ denote the indices of training and test samples with label $c$, of sizes $m_c^{(\pi)}$ and $u_c^{(\pi)}$, respectively. For an index set $\gI$, we define the empirical distribution $\mu_\gI \coloneq \frac{1}{|\gI|} \sum_{i \in \gI} \delta_{\rvx_i}$, where $\delta_{\rvx}$ denotes the Dirac delta function, and write $\phi_\#\mu_\gI$ for its pushforward under $\phi$, i.e., the empirical distribution of $\{\phi(\rvx_i)\}_{i \in \gI}$. We state our main theorem as follows.
\begin{restatable}[Optimal transport bound]{theorem}{classwisebound}
\label{thm:classwise-ot}
  Let $\gamma>0$. Then, with probability at least $1-\delta$ over the random split $\pi$, for all $f\circ\phi\in F\circ\Phi$,
  \begin{align}
    \label{eq:main-OT-bound}
    &R_u(f\circ\phi;\pi) - R_{m,\gamma}(f\circ\phi;\pi) \\
    &\le 
    \frac{1}{\gamma}\sum_{c=1}^K M_c(f,\phi)\,
    \mathbb{E}_{\pi'}\!\left[
        \frac{m_c^{(\pi')}}{m}\,
        \mathcal{W}\!\left(\phi_\#\mu_{\gI_{\mathrm{train},c}^{(\pi')}}, \phi_\#\mu_{\gI_{\mathrm{test},c}^{(\pi')}}\right)
    \right]
    + \mathbb{E}_{\pi'}\!\left[\sum_{c=1}^K \biggl|\frac{u_c^{(\pi')}}{u} - \frac{m_c^{(\pi')}}{m}\biggl|\right]
    + \varepsilon_\delta, \notag
\end{align}


where $
    M_c(f,\phi) 
    \coloneq 
    \max_{i, j}~\frac{
        |\rho_f(\phi(\rvx_i),c) - \rho_f(\phi(\rvx_j),c)|
    }{
        \|\phi(\rvx_i)-\phi(\rvx_j)\|
    }
$ for $i \neq j$ such that $
    i, j \in \gI_{\mathrm{train},c}^{(\pi)} \cup \gI_{\mathrm{test}}^{(\pi)}.
$ \\
In addition, $\varepsilon_\delta \coloneq \left[\frac{m u \omega^2}{2(m+u-\tfrac12)} \left(1-\frac{1}{2\max\{m,u\}}\right)^{-1} \ln\frac{1}{\delta}\right]^{1/2}$
 with $\omega \coloneq \frac{1}{m} + \frac{1}{u}$.
\end{restatable}

We provide a proof of \cref{thm:classwise-ot} in \cref{apdx_sec:proof of class-wise OT}. We defer empirical estimation of the bound to \cref{apdx_sec:approximate}.

The Optimal Transport bound (\emph{OT bound}) in \cref{thm:classwise-ot} comprises two key components, $\mathbb{E}_{\pi'}\mathcal{W}$ and $M_c(f,\phi)$, each capturing a distinct aspect of $f\circ\phi$. 
The expected Wasserstein distance $\mathbb{E}_{\pi'}\mathcal{W}$ captures \emph{intra-class concentration}. The expectation over random splits $\pi'$ removes the train/test distinction, therefore, if within-class features are concentrated, the Wasserstein distance is small.      
$M_c(f,\phi)$ is the maximum rate of change of the class-$c$ margin across pairs of encoded data points; it is small when the classifier varies smoothly in feature space, playing the same role as the Lipschitz constant of the margin in inductive bounds. 

The hyperparameter $\gamma$ controls a trade-off in the bound: a larger $\gamma$ shrinks the Wasserstein term through the $1/\gamma$ factor but inflates $R_{m,\gamma}$. The tightest bound is obtained by balancing the two, rather than by taking $\gamma$ as large as possible. 
We further examine the relationship between $M_c(f,\phi)$, the margin $\gamma$, and the Wasserstein distance.
To assign different classes to features that lie close together while maintaining a margin $\gamma$, the classifier must vary sufficiently sharply in feature space. We provide a formal statement of this relationship, its proof, and an empirical analysis across encoder depths in
\cref{apdx_sec:mc_analysis}.
The second term of our bound measures the imbalance of class proportions between the train and test splits, and vanishes when the splits are class-balanced. The last term $\varepsilon_\delta$ is a model-agnostic constant.


The following remark shows that, as a special case, our bound recovers the classical concentration-separation principle of \citet{chuang2021measuring} from the inductive setting.

\begin{remark}[cf.\ Lemma~10 in~\citet{chuang2021measuring}; 
Proposition~5.2 in~\citet{li2025towards}]
Under the same setup as \cref{thm:classwise-ot}, assume additionally that the class-wise classifiers $f_c$ in $f$ are Lipschitz and that $\rho_f(\phi(\rvx_i), y_i) \geq \gamma$ for all $i \in [m]$. Then, the Wasserstein term in \Cref{thm:classwise-ot} is lower bounded by
\begin{align}
\frac{\mathrm{Lip}(f)  \sum_{c=1}^K M_c(f, \phi) \ \mathbb{E}_{\pi'} \left[  \frac{m_c^{(\pi')}}{m}\; \mathcal{W} \left( \phi_\# \mu_{\gI_{train,c}^{(\pi')}}, \phi_\# \mu_{\gI_{test,c}^{(\pi')}} \right) \right]}{ \min_{c \in \gY} \mathrm{Lip}(f_c) \ \min_{c \neq c'} \mathcal{W} \left( \phi_\# \mu_{\gI_c}, \phi_\# \mu_{\gI_{c'}} \right)}.
\end{align}
\end{remark}


Under the assumption of the remark, the Wasserstein term in \cref{thm:classwise-ot} is lower-bounded by a quantity inversely proportional to the inter-class Wasserstein distance $\mathcal{W}(\phi_\#\mu_{\gI_c}, \phi_\#\mu_{\gI_{c'}})$ between distinct classes $c \neq c'$. We refer to this property as \emph{inter-class separation}. 

Although stated as a special case, the same principle informs \cref{thm:classwise-ot} more broadly. When classes are well-separated, $\gamma$ can be chosen larger without inflating $R_{m,\gamma}$, tightening the bound through the $1/\gamma$ scaling. The concentration-separation principle that \citet{chuang2021measuring} derived in the inductive setting therefore carries over to our transductive formulation. Good generalization requires both intra-class concentration \emph{and} inter-class separation.

We further derive a variant of the bound based on the Wasserstein distance between the training and test feature distributions, without class-wise decomposition.
We call this the global bound, which can be evaluated without access to test labels and holds deterministically for any split $\pi$, as presented in \cref{apdx:sec:global}.
We defer its statement to \cref{apdx:sec:global}.  

\subsection{Depth-dependent specialization to GNNs}\label{subsec:gnn-depth}

GNN encoders shape feature geometry through message passing, raising the natural question of how message passing affects the concentration and separation captured by our bound. 
We specialize our bound to GNN encoders by deriving depth-dependent expressions for the Wasserstein terms and reveal a non-monotonic relationship between depth and generalization.

To state the propositions, let $\tilde{\rmD}$ be the diagonal degree matrix after adding self-loops to the graph, so that $\tilde{d}_i \coloneq (\tilde{\rmD})_{ii}$ denotes the resulting degree of node $i$. Define the degree statistic $d(\rvx_i) \coloneq \sqrt{\tilde{d}_i}$. 
Let $\{\lambda_k(\hat{\rmA})\}_{k=1}^N$ be the eigenvalues of $\hat{\rmA}$ in descending order, where $\lambda_1(\hat{\rmA})=1$, and define $\rho_\perp(\hat{\rmA}) \coloneq \max_{k\geq 2}|\lambda_k(\hat{\rmA})|$. Throughout this section, we assume the graph is connected, so that $\rho_\perp(\hat{\rmA}) < 1$. The following propositions characterize for arbitrary node subsets $\mathcal{S}$ and  $\mathcal{T}$, how the Wasserstein distance $\mathcal{W}(\phi_\#\mu_\mathcal{S}, \phi_\#\mu_{\mathcal{T}})$ evolves with depth for SGC and GCN encoders.

\begin{restatable}[SGC bound]{proposition}{sgcdepth}
\label{prop:w1-sgc}
    Consider the SGC encoder for node $i$ at depth $\ell$ defined by $\phi^{(\ell)}(\rvx_i) := (\hat{\rmA}^{\ell}\rmX)_{i\cdot}$. Then for any nonempty $\mathcal{S},\mathcal{T}\subseteq[N]$ and any $\ell\in\mathbb{N}$,
    \begin{equation}\label{eq:w1-sgc}
        \mathcal{W}\big(\phi^{(\ell)}_\#\mu_{\mathcal S},\,\phi^{(\ell)}_\#\mu_{\mathcal T}\big) 
        \leq
        C_1\, \mathcal{W}\big(d_\#\mu_{\mathcal S},\,d_\#\mu_{\mathcal T}\big)
        + C_2\, \rho_\perp(\hat{\rmA})^{\ell},
    \end{equation}
    where $C_1\coloneq \|\mathbf X\|_F/\|\mathbf u_1\|_2$ and $C_2\coloneq 2\,\|\mathbf P_{U^\perp}\mathbf X\|_F$ depend only on $\rmX$ and $\hat{\rmA}$.
\end{restatable}

\begin{restatable}[GCN bound]{proposition}{gcndepth}
\label{prop:w1-gcn}
    Consider a GCN encoder with ReLU activation
    $\rmX^{(\ell+1)} = \mathrm{ReLU}\big(\hat{\rmA}\rmX^{(\ell)}\rmW^{(\ell)}\big)$,
    and define $\phi^{(\ell)}(x_i) := (\rmX^{(\ell)})_{i\cdot}$.
    Assume $\|\rmW^{(\ell)}\|_2 \leq \beta$ for every layer.
    Then for any nonempty $\mathcal{S},\mathcal{T}\subseteq[N]$
    and any $\ell\in\mathbb{N}$,
    \begin{equation}\label{eq:w1-gcn}
        \mathcal{W}\big(\phi^{(\ell)}_\#\mu_{\mathcal{S}},\,\phi^{(\ell)}_\#\mu_{\mathcal{T}}\big)
        \leq
        \Big(C_1\,\mathcal{W}\big(d_\#\mu_{\mathcal{S}},\,d_\#\mu_{\mathcal{T}}\big)
        + C_2\,\rho_\perp(\hat{\rmA})^{\ell}\Big)\beta^{\ell},
    \end{equation}
    where $C_1\coloneq\|\mathbf X\|_F/\|\mathbf u_1\|_2$
    and $C_2\coloneq 2\,\|\mathbf P_{U^\perp}\mathbf X\|_F$ depend
    only on $\rmX$ and $\hat{\rmA}$.

\end{restatable}


\cref{prop:w1-sgc,prop:w1-gcn} provide depth-dependent upper bounds on the Wasserstein distance between feature distributions of arbitrary node subsets. For SGC, since $\rho_\perp(\hat{\rmA}) < 1$, this upper bound decays exponentially with depth, so the Wasserstein distance contracts. Through \cref{thm:classwise-ot}, this contraction has two opposing effects on the bound. Applied to the intra-class distance $\mathcal{W}(\phi_\#\mu_{\gI_{\mathrm{train},c}^{(\pi)}}, \phi_\#\mu_{\gI_{\mathrm{test},c}^{(\pi)}})$, the contraction tightens the Wasserstein term, reducing the bound. 
However, the contraction also reduces inter-class separation, so the largest $\gamma$ compatible with a given $R_{m,\gamma}$ shrinks, inflating the $1/\gamma$ factor and loosening the bound.
Depth, therefore, induces a fundamental trade-off between intra-class concentration and inter-class separation, allowing the bound to depend non-monotonically on depth for SGC.

For GCN, the upper bound on the Wasserstein distance may decay or grow with depth depending on the magnitude of $\beta$, the spectral norm of the learned weights. Empirically, regularizing the spectral behavior of GCN weights has been shown to improve deep-layer performance \citep{park2024mitigating}, consistent with the regime in which our upper bound decays. When the upper bound decays, the same trade-off applies and generalization on GCN may likewise be non-monotonic in depth.



\paragraph{Revisiting oversmoothing via generalization}
Our propositions connect to the phenomenon of \emph{oversmoothing}~\citep{li2018deeper,Oono2020Graph,cai2020note}: the exponential decay of the Wasserstein distance in \cref{prop:w1-sgc,prop:w1-gcn} formalizes the feature contraction this literature describes. Prior oversmoothing measures, such as the Dirichlet energy~\citep{cai2020note}, quantify oversmoothing through aggregate distances between node features and treat contraction itself as the diagnostic. However, contraction alone does not determine generalization. What matters is whether intra-class points are pulled together \emph{more} than inter-class points, an asymmetry that class-agnostic measures cannot resolve. \cref{thm:classwise-ot} makes this asymmetry explicit. Because the bound depends on both intra-class concentration and inter-class separation, the contraction described by \cref{prop:w1-sgc,prop:w1-gcn} translates into beneficial intra-class concentration on the one hand and harmful inter-class collapse on the other. 

This perspective helps explain why aggregate oversmoothing measures often correlate only weakly with downstream performance~\citep{rusch2023survey,heo2025influence}.
\citet{keriven2022not,wu2023non-asym} make a related observation through simplified linear GNNs on random graphs; we extend this perspective to non-linear GCN architectures on arbitrary graphs.

\section{Experiments}
\label{sec:exp}

We evaluate our bounds on GNN-based node classification, a representative transductive task in which the interdependence of node features is induced via message passing. Our experiments address three questions: (i) Does our OT bound align with the empirical generalization error? (ii) Beyond theory, can the bound serve as a regularizer to improve GNN node classification performance? (iii) Does the depth-dependent behavior derived in~\cref{subsec:gnn-depth} hold in practice? 
We examine each question in turn.

\subsection{Correlation with the generalization error}
\label{subsec:exp_corr_gnn}
\begin{table*}[t]
    \caption{Rank correlation between generalization bounds and the empirical error gap across nine datasets and five GNN architectures. \emph{OT-oracle} and \emph{OT bound} correspond to~\cref{thm:classwise-ot} with and without access to test labels, respectively. Cells are shaded green only for positive correlations, with deeper green indicating stronger alignment. Negative correlations are left unshaded. The RC bound is shown only for the applicable models.}
  \centering
  \resizebox{0.95\linewidth}{!}{%
  \begin{tabular}{ll *{9}{>{\centering\arraybackslash}p{0.075\linewidth}}}
    \toprule
      &  & Cora & Cite. & PubMed & Comp. & Photo & Squirrel & Chamel. & Roman. & Amazon. \\
    \midrule\midrule
    \multirow{4}{*}{SGC} & PAC bound & \cellcolor[rgb]{0.786,0.890,0.816}\phantom{-}0.61 & \cellcolor[rgb]{0.754,0.874,0.789}\phantom{-}0.70 & \cellcolor[rgb]{0.716,0.854,0.756}\phantom{-}0.81 & \cellcolor[rgb]{0.966,0.982,0.971}\phantom{-}0.10 & \cellcolor[rgb]{0.713,0.852,0.754}\phantom{-}0.82 & \cellcolor[rgb]{0.889,0.943,0.905}\phantom{-}0.32 & \cellcolor[rgb]{0.857,0.927,0.878}\phantom{-}0.41 & \cellcolor[rgb]{0.730,0.861,0.768}\phantom{-}0.77 & \cellcolor[rgb]{0.730,0.861,0.768}\phantom{-}0.77 \\
     & RC bound & -0.92 & -0.97 & -0.35 & -0.04 & -0.41 & -0.94 & -0.94 & -0.56 & -0.77 \\
     & OT-oracle& \cellcolor[rgb]{0.720,0.856,0.760}\phantom{-}0.80 & \cellcolor[rgb]{0.682,0.836,0.727}\phantom{-}0.91 & \cellcolor[rgb]{0.653,0.821,0.702}\phantom{-}0.99 & \cellcolor[rgb]{0.930,0.964,0.940}\phantom{-}0.20 & \cellcolor[rgb]{0.716,0.854,0.756}\phantom{-}0.81 & \cellcolor[rgb]{0.695,0.843,0.739}\phantom{-}0.87 & \cellcolor[rgb]{0.745,0.869,0.781}\phantom{-}0.73 & \cellcolor[rgb]{0.727,0.860,0.766}\phantom{-}0.78 & \cellcolor[rgb]{0.671,0.831,0.718}\phantom{-}0.94 \\
     & OT bound & \cellcolor[rgb]{0.723,0.857,0.762}\phantom{-}0.79 & \cellcolor[rgb]{0.688,0.840,0.733}\phantom{-}0.89 & \cellcolor[rgb]{0.657,0.824,0.706}\phantom{-}0.98 & \cellcolor[rgb]{0.941,0.970,0.949}\phantom{-}0.17 & \cellcolor[rgb]{0.716,0.854,0.756}\phantom{-}0.81 & \cellcolor[rgb]{0.709,0.850,0.751}\phantom{-}0.83 & \cellcolor[rgb]{0.752,0.872,0.787}\phantom{-}0.71 & \cellcolor[rgb]{0.727,0.860,0.766}\phantom{-}0.78 & \cellcolor[rgb]{0.671,0.831,0.718}\phantom{-}0.94 \\
    \midrule
    \multirow{4}{*}{GCN} & PAC bound & -0.78 & -0.65 & -0.91 & -0.90 & -0.79 & -0.96 & -0.87 & -0.87 & -0.92 \\
     & RC bound & \cellcolor[rgb]{0.898,0.948,0.913}\phantom{-}0.29 & \cellcolor[rgb]{0.833,0.914,0.856}\phantom{-}0.48 & \cellcolor[rgb]{0.878,0.937,0.895}\phantom{-}0.35 & -0.06 & -0.32 & -0.75 & -0.82 & \cellcolor[rgb]{0.878,0.937,0.895}\phantom{-}0.35 & \cellcolor[rgb]{0.930,0.964,0.940}\phantom{-}0.20 \\
    
     & OT-oracle& \cellcolor[rgb]{0.761,0.877,0.795}\phantom{-}0.68 & \cellcolor[rgb]{0.716,0.854,0.756}\phantom{-}0.81 & \cellcolor[rgb]{0.673,0.832,0.720}\phantom{-}0.93 & \cellcolor[rgb]{0.896,0.946,0.911}\phantom{-}0.30 & \cellcolor[rgb]{0.896,0.946,0.911}\phantom{-}0.30 & \cellcolor[rgb]{0.752,0.872,0.787}\phantom{-}0.71 & \cellcolor[rgb]{0.923,0.960,0.934}\phantom{-}0.22 & \cellcolor[rgb]{0.682,0.836,0.727}\phantom{-}0.91 & \cellcolor[rgb]{0.713,0.852,0.754}\phantom{-}0.82 \\
     & OT bound & \cellcolor[rgb]{0.808,0.901,0.835}\phantom{-}0.55 & \cellcolor[rgb]{0.723,0.857,0.762}\phantom{-}0.79 & \cellcolor[rgb]{0.677,0.834,0.724}\phantom{-}0.92 & \cellcolor[rgb]{0.920,0.959,0.932}\phantom{-}0.23 & \cellcolor[rgb]{0.930,0.964,0.940}\phantom{-}0.20 & \cellcolor[rgb]{0.786,0.890,0.816}\phantom{-}0.61 & \cellcolor[rgb]{0.968,0.984,0.973}\phantom{-}0.09 & \cellcolor[rgb]{0.684,0.838,0.729}\phantom{-}0.90 & \cellcolor[rgb]{0.713,0.852,0.754}\phantom{-}0.82 \\
    \midrule
    \multirow{3}{*}{GCNII} & PAC bound & -0.17 & \cellcolor[rgb]{0.973,0.986,0.976}\phantom{-}0.08 & -0.51 & -0.32 & -0.25 & \cellcolor[rgb]{0.892,0.944,0.907}\phantom{-}0.31 & \cellcolor[rgb]{0.966,0.982,0.971}\phantom{-}0.10 & -0.55 & -0.28 \\
     & OT-oracle& \cellcolor[rgb]{0.713,0.852,0.754}\phantom{-}0.82 & \cellcolor[rgb]{0.736,0.864,0.774}\phantom{-}0.75 & \cellcolor[rgb]{0.765,0.879,0.799}\phantom{-}0.67 & \cellcolor[rgb]{0.736,0.864,0.774}\phantom{-}0.75 & \cellcolor[rgb]{0.713,0.852,0.754}\phantom{-}0.82 & \cellcolor[rgb]{0.959,0.979,0.965}\phantom{-}0.12 & -0.23 & \cellcolor[rgb]{0.779,0.886,0.811}\phantom{-}0.63 & \cellcolor[rgb]{0.730,0.861,0.768}\phantom{-}0.77 \\
     & OT bound & \cellcolor[rgb]{0.727,0.860,0.766}\phantom{-}0.78 & \cellcolor[rgb]{0.736,0.864,0.774}\phantom{-}0.75 & \cellcolor[rgb]{0.761,0.877,0.795}\phantom{-}0.68 & \cellcolor[rgb]{0.752,0.872,0.787}\phantom{-}0.71 & \cellcolor[rgb]{0.720,0.856,0.760}\phantom{-}0.80 & \cellcolor[rgb]{0.962,0.980,0.967}\phantom{-}0.11 & -0.24 & \cellcolor[rgb]{0.783,0.888,0.814}\phantom{-}0.62 & \cellcolor[rgb]{0.730,0.861,0.768}\phantom{-}0.77 \\
    \midrule
    \multirow{3}{*}{GAT} & PAC bound & -0.81 & -0.61 & -0.90 & -0.90 & -0.91 & -0.91 & -0.91 & -0.70 & -0.93 \\
     & OT-oracle& \cellcolor[rgb]{0.857,0.927,0.878}\phantom{-}0.41 & \cellcolor[rgb]{0.702,0.847,0.745}\phantom{-}0.85 & \cellcolor[rgb]{0.705,0.848,0.747}\phantom{-}0.84 & \cellcolor[rgb]{0.702,0.847,0.745}\phantom{-}0.85 & \cellcolor[rgb]{0.723,0.857,0.762}\phantom{-}0.79 & \cellcolor[rgb]{0.720,0.856,0.760}\phantom{-}0.80 & \cellcolor[rgb]{0.698,0.845,0.741}\phantom{-}0.86 & \cellcolor[rgb]{0.709,0.850,0.751}\phantom{-}0.83 & \cellcolor[rgb]{0.666,0.828,0.714}\phantom{-}0.95 \\
     & OT bound & \cellcolor[rgb]{0.889,0.943,0.905}\phantom{-}0.32 & \cellcolor[rgb]{0.727,0.860,0.766}\phantom{-}0.78 & \cellcolor[rgb]{0.713,0.852,0.754}\phantom{-}0.82 & \cellcolor[rgb]{0.713,0.852,0.754}\phantom{-}0.82 & \cellcolor[rgb]{0.730,0.861,0.768}\phantom{-}0.77 & \cellcolor[rgb]{0.709,0.850,0.751}\phantom{-}0.83 & \cellcolor[rgb]{0.698,0.845,0.741}\phantom{-}0.86 & \cellcolor[rgb]{0.709,0.850,0.751}\phantom{-}0.83 & \cellcolor[rgb]{0.666,0.828,0.714}\phantom{-}0.95 \\
    \midrule
    \multirow{3}{*}{SAGE} & PAC bound & -0.36 & -0.38 & -0.66 & -0.89 & -0.79 & -0.75 & -0.13 & -0.25 & -0.81 \\
     & OT-oracle& \cellcolor[rgb]{0.671,0.831,0.718}\phantom{-}0.94 & \cellcolor[rgb]{0.673,0.832,0.720}\phantom{-}0.93 & \cellcolor[rgb]{0.684,0.838,0.729}\phantom{-}0.90 & \cellcolor[rgb]{0.677,0.834,0.724}\phantom{-}0.92 & \cellcolor[rgb]{0.709,0.850,0.751}\phantom{-}0.83 & \cellcolor[rgb]{0.817,0.906,0.844}\phantom{-}0.52 & \cellcolor[rgb]{0.776,0.885,0.808}\phantom{-}0.64 & \cellcolor[rgb]{0.871,0.934,0.889}\phantom{-}0.37 & \cellcolor[rgb]{0.671,0.831,0.718}\phantom{-}0.94 \\
     & OT bound & \cellcolor[rgb]{0.673,0.832,0.720}\phantom{-}0.93 & \cellcolor[rgb]{0.682,0.836,0.727}\phantom{-}0.91 & \cellcolor[rgb]{0.684,0.838,0.729}\phantom{-}0.90 & \cellcolor[rgb]{0.677,0.834,0.724}\phantom{-}0.92 & \cellcolor[rgb]{0.713,0.852,0.754}\phantom{-}0.82 & \cellcolor[rgb]{0.808,0.901,0.835}\phantom{-}0.55 & \cellcolor[rgb]{0.804,0.899,0.832}\phantom{-}0.56 & \cellcolor[rgb]{0.889,0.943,0.905}\phantom{-}0.32 & \cellcolor[rgb]{0.671,0.831,0.718}\phantom{-}0.94 \\
    \bottomrule
  \end{tabular}%
  }
  \label{tab:main_result}
\end{table*}

To assess whether our bound aligns with the empirical generalization error, we measure the rank correlation between our bound and the empirical generalization error, following the evaluation measure proposed in the Predicting Generalization in Deep Learning competition~\citep{jiang2020neuripscomp}. We generate $18$ configurations by varying the encoder $\phi$ over GNNs with depth $\{1,2,4,8,16,32\}$ and the classifier $f$ over one-, two-, and four-layer MLPs with ReLU activations. We report two variants of our bound: \emph{OT-oracle} computes the bound with access to all labels, serving as a reference; \emph{OT bound} estimates it from training labels only, reflecting realistic test-blind use. The estimation procedure of the OT bound is detailed in \cref{apdx_sec:approximate}. Training details are provided in \cref{apdx_sec:exp_gnn}. 

We use nine datasets, comprising five homophilic graphs (Cora, CiteSeer, PubMed, Computers, Photo)~\citep{sen2008collective, yang2016revisiting,mcauley2015image} and four heterophilic graphs (Squirrel, Chameleon, Roman-empire, Amazon-ratings)~\citep{musae,platonov2023critical}. Dataset statistics are summarized in \cref{apdx:datastat}. We evaluate five GNN architectures: SGC~\citep{wu2019simplifying}, GCN~\citep{kipf2016semi}, GCNII~\citep{chen2020simple}, GAT~\citep{velivckovic2017graph}, and GraphSAGE~\citep{hamilton2017inductive}.

We compare against two transductive generalization bounds: the PAC-Bayesian bound~\citep{pac2014begin} and a transductive Rademacher complexity (RC) bound~\citep{el2009transductive, EsserNEURIPS2021_learning}. 
Other transductive bounds based on VC dimension~\citep{pmlr-v35-tolstikhin14}, Permutational Rademacher Complexity~\citep{permute_RC_2015}, and stability~\citep{el2006stable, trans_stab_2008} are not included, as they are computationally intractable in practice.
Although the original RC bound~\citep{el2009transductive} is not computable in general, we use the GNN-specific upper bound from \citet{EsserNEURIPS2021_learning}, which applies only to GCN and SGC.
We attempted to include the stability-based bound in the same manner, using the GNN-specific bound of \citet{congNEURIPS2021_provable}. However, its value grows exponentially with depth, as $O((\max_i d_i)^L E)$ for SGC and $O((L+2)\,E(\max_i \sqrt{d_i}\,\beta)^{2L})$ for GCN, where $E$ is the number of training epochs, and its computation overflows at the depths we consider.

\cref{tab:main_result} shows the rank correlation between each bound and the empirical generalization error across nine datasets and five GNN architectures. Both OT-oracle and OT bound consistently show strong alignment with the empirical generalization error, providing experimental support for \cref{thm:classwise-ot}. 
Notably, OT bound performs comparably to OT-oracle, indicating that our bound remains effective in realistic settings without access to test labels. In contrast, the PAC and RC baselines largely fail to track the empirical error across most datasets and architectures, showing that these classical complexity measures are not reliable predictors for graph node classification. 

In addition to node classification, our bound also exhibits a strong correlation with the generalization error in GNN link prediction, another transductive task with dependent features (see \cref{appdx:subsec:link_prediction}).

\subsection{Application as a training regularizer}
\label{subsec:expo_reg_gnn}
Beyond theoretical analysis, our bound can serve a practical role during training. To explicitly encourage a generalization-friendly representation geometry, we add our bound as a regularizer $\gR$ to the cross-entropy loss $\gL_\mathrm{CE}$: $\gL=\gL_\mathrm{CE}+\lambda\gR$.  Setting $\gamma=1$, we define 
$\gR \coloneq \sum_{c=1}^K M_c(f,\phi) \mathbb{E}_{\pi'} \left[ \frac{m_c}{m} \gW_1(\phi_\#\mu_{\gI_{\mathrm{train},c}^{(\pi')}},\phi_\#\mu_{\gI_{\mathrm{test},c}^{(\pi')}}) \right]$. 
Since the original bound is non-differentiable, the surrogate regularizer $\mathcal{R}$ replaces the bound's two non-smooth components with well-known smooth approximations. The Wasserstein distance is approximated by the entropy-regularized Sinkhorn algorithm~\citep{cuturi2013sinkhorn}, and the maximum in $M_c(f,\phi) = \max_{i,j} |\rho_f(\phi(\rvx_i),c) - \rho_f(\phi(\rvx_j),c)| / \|\phi(\rvx_i)-\phi(\rvx_j)\|$ by a temperature-smoothed log-sum-exp.

We evaluate this strategy on the same datasets as in \cref{subsec:exp_corr_gnn} (five homophilic and four heterophilic graphs) using GCN, GCNII, GAT, and GraphSAGE. We report mean test accuracy and standard deviation over five runs. Full experimental details are provided in \cref{apdx_sec:exp_gnn}. 

\begin{table*}[t]
    \caption{Effect of using our OT bound as a regularizer on test accuracy across datasets and GNNs. Mean $\pm$ std over multiple runs; $\Delta$ is the gain over the unregularized baseline.}
    \centering
    \resizebox{1.0\linewidth}{!}{
    \begin{tabular}{lccccccccc}
         \toprule
          & Cora & CiteSeer & PubMed & Computers & Photo & Squirrel & Chamel. & Roman. & Amazon. \\
          Homophily  &    0.81    &    0.74    &    0.80    &   0.78     &    0.83    &    0.21    &   0.24   &   0.05     &    0.38  \\
         \midrule\midrule
         GCN & 83.70{\scriptsize $\pm$1.16} & 69.55{\scriptsize $\pm$1.19} & 87.56{\scriptsize $\pm$0.34} & 89.96{\scriptsize $\pm$0.25} & 93.25{\scriptsize $\pm$0.57} & 43.14{\scriptsize $\pm$1.32} & 43.54{\scriptsize $\pm$3.59} & 56.06{\scriptsize $\pm$2.75} & 49.38{\scriptsize $\pm$0.52} \\
         + OT reg. & 84.55{\scriptsize $\pm$1.00} & 71.89{\scriptsize $\pm$1.29} & 87.68{\scriptsize $\pm$0.24} & 90.76{\scriptsize $\pm$0.46} & 93.70{\scriptsize $\pm$0.60} & 45.27{\scriptsize $\pm$1.52} & 44.28{\scriptsize $\pm$3.31} & 64.68{\scriptsize $\pm$1.86} & 49.02{\scriptsize $\pm$0.66} \\
         $\Delta$ & +0.85 & +2.34 & +0.12 & +0.80 & +0.45 & +2.13 & +0.74 & +8.62 & $-$0.36 \\
         \midrule
         GCNII & 84.44{\scriptsize $\pm$1.01} & 71.02{\scriptsize $\pm$1.85} & 87.37{\scriptsize $\pm$0.71} & 89.10{\scriptsize $\pm$1.10} & 92.14{\scriptsize $\pm$1.72} & 41.26{\scriptsize $\pm$1.61} & 42.38{\scriptsize $\pm$3.32} & 70.88{\scriptsize $\pm$0.96} & 48.38{\scriptsize $\pm$0.63} \\
         + OT reg. & 84.44{\scriptsize $\pm$1.52} & 71.83{\scriptsize $\pm$1.11} & 88.98{\scriptsize $\pm$0.31} & 90.20{\scriptsize $\pm$0.40} & 94.09{\scriptsize $\pm$0.74} & 44.19{\scriptsize $\pm$2.24} & 45.40{\scriptsize $\pm$2.94} & 76.30{\scriptsize $\pm$0.94} & 47.92{\scriptsize $\pm$0.56} \\
         $\Delta$ & 0.00 & +0.81 & +1.61 & +1.10 & +1.95 & +2.93 & +3.02 & +5.42 & $-$0.46 \\
         \midrule
         GAT & 82.99{\scriptsize $\pm$1.28} & 69.97{\scriptsize $\pm$1.92} & 85.03{\scriptsize $\pm$0.81} & 90.49{\scriptsize $\pm$0.65} & 93.14{\scriptsize $\pm$0.64} & 38.63{\scriptsize $\pm$3.19} & 40.38{\scriptsize $\pm$2.48} & 57.83{\scriptsize $\pm$1.65} & 50.13{\scriptsize $\pm$0.49} \\
         + OT reg. & 84.51{\scriptsize $\pm$1.15} & 71.47{\scriptsize $\pm$1.42} & 86.55{\scriptsize $\pm$0.24} & 91.36{\scriptsize $\pm$0.51} & 93.73{\scriptsize $\pm$0.31} & 39.66{\scriptsize $\pm$1.95} & 41.50{\scriptsize $\pm$2.23} & 62.09{\scriptsize $\pm$1.24} & 49.84{\scriptsize $\pm$0.30} \\
         $\Delta$ & +1.52 & +1.50 & +1.52 & +0.87 & +0.59 & +1.03 & +1.12 & +4.26 & $-$0.29 \\
         \midrule
         SAGE & 84.66{\scriptsize $\pm$0.56} & 70.03{\scriptsize $\pm$2.42} & 87.06{\scriptsize $\pm$0.62} & 88.92{\scriptsize $\pm$0.14} & 94.13{\scriptsize $\pm$0.44} & 36.10{\scriptsize $\pm$1.49} & 39.79{\scriptsize $\pm$3.58} & 76.79{\scriptsize $\pm$0.89} & 47.03{\scriptsize $\pm$0.63} \\
         + OT reg. & 85.66{\scriptsize $\pm$0.95} & 72.61{\scriptsize $\pm$1.17} & 87.84{\scriptsize $\pm$0.40} & 89.81{\scriptsize $\pm$0.60} & 94.44{\scriptsize $\pm$0.60} & 38.49{\scriptsize $\pm$2.92} & 42.98{\scriptsize $\pm$4.33} & 78.20{\scriptsize $\pm$0.45} & 47.40{\scriptsize $\pm$0.43} \\
         $\Delta$ & +1.00 & +2.58 & +0.78 & +0.89 & +0.31 & +2.39 & +3.19 & +1.41 & +0.37 \\
         \bottomrule
    \end{tabular}
    }
    \label{tab:regularization_result}
\end{table*}

\cref{tab:regularization_result} shows that adding our regularizer improves over the baseline in $32$ out of $36$ setups. Improvement magnitude varies with dataset homophily: while gains are modest on homophilic graphs, they are substantially larger on heterophilic graphs, reaching up to $+8.62\%$ on Roman-empire with GCN. The exception is Amazon-ratings where the base model itself struggles to fit the training data, a regime in which regularization typically does not help. We provide a detailed analysis of this case in \cref{app:opt-stability}.

We attribute the regularizer's consistent gains to its role in correcting the GNN encoder's tendency to mix representations across classes. This mixing is inherently more problematic in heterophilic graphs, where nodes frequently aggregate features from neighbors with different labels. Consequently, encouraging intra-class concentration in the encoder yields larger benefits in heterophilic settings. Without our regularizer, the standard cross-entropy loss must compensate for poorly separated features with a sharp classifier alone. Our regularizer adds two complementary terms: one at the encoder level, drawing same-class features together, and one at the classifier level, penalizing large change rates $M_c$. Together, they encourage features to be well clustered at the encoder level, resulting in consistent performance gains.

\subsection{Validation of the depth-dependent analysis}\label{sec:depth-ana}
We validate the depth-dependent analysis in~\cref{subsec:gnn-depth}. We first verify~\cref{prop:w1-sgc,prop:w1-gcn}, and show that generalization exhibits a non-monotonic dependence on depth, consistent with what our \cref{thm:classwise-ot} suggests. 
Specifically, we track how the feature distributions evolve under message passing by measuring two key Wasserstein distances, capturing intra-class concentration and inter-class separation, respectively. To measure intra-class concentration, we compute the class-averaged expected Wasserstein distance $\classWD \coloneq \frac{1}{K}\sum_{c=1}^K \mathbb{E}_{\pi'}\big[\frac{m_c^{(\pi')}}{m}\, \mathcal{W}(\phi_\#\mu_{\gI_{\rm{train},c}^{(\pi')}}, \phi_\#\mu_{\gI_{\rm{test},c}^{(\pi')}})\big]$. To measure inter-class separation, we use an averaged Wasserstein distance between different classes, i.e., $\gammaWD \coloneq \frac{1}{K(K-1)} \sum_{c_1\neq c_2}\mathcal{W}(\phi_\#\mu_{\gI_{c_1}}, \phi_\#\mu_{\gI_{c_2}})$, as a proxy.
We also track how the generalization error varies with depth.

\begin{figure*}[t]
\centering
\resizebox{\textwidth}{!}{%
\begin{tikzpicture}

\pgfplotsset{
    every axis/.append style={
        width=8cm, height=6cm,
        scale only axis=true,
        axis line style={line width=1.5pt},
        ymajorgrids=true,
        grid style={line width=0.2pt, dashed, gray!50},
        label style={font=\huge},
        tick label style={font=\huge},
        legend style={font=\huge, draw=gray!40, rounded corners=1pt, inner sep=2pt},
        title style={at={(0.5,-0.40)}, anchor=north, font=\Huge, align=center},
        xtick={1,8,16,24,32},
        xmin=0, xmax=33,
        every axis plot/.append style={line width=2.0pt},
        major tick length=0pt,
        minor tick length=0pt,
        yminorgrids=false,
        scaled y ticks=false,
        legend pos=north east,
    },
}

\begin{groupplot}[
    group style={
        group size=2 by 2,
        horizontal sep=1.8cm,
        vertical sep=2.8cm,
    },
]

\nextgroupplot[
    ymin=0, ymax=0.16,
    ytick={0.04,0.08,0.12}, yticklabels={.04,.08,.12}]
\addplot[color=colorOracle, mark=square*, mark size=5.0pt] coordinates {
    (1, 0.019) (4, 0.00877) (8, 0.00609) (12, 0.00498) (16, 0.00435) (20, 0.00394) (24, 0.00366) (28, 0.00345) (32, 0.00329)
};
\addplot[color=colorGlobal, mark=triangle*, mark size=5.5pt] coordinates {
    (1, 0.132) (4, 0.082) (8, 0.064) (12, 0.0551) (16, 0.0494) (20, 0.0454) (24, 0.0423) (28, 0.0399) (32, 0.0378)
};
\legend{$\mathcal{W}_C$, $\mathcal{W}_S$}

\nextgroupplot[
    ymin=0.06, ymax=0.22,
    ytick={0.10,0.14,0.18}, yticklabels={.10,.14,.18}]
\addplot[name path=upB, draw=none, forget plot] coordinates {
    (1, 0.176) (4, 0.183) (8, 0.196) (12, 0.18) (16, 0.173) (20, 0.166) (24, 0.15) (28, 0.137) (32, 0.118)
};
\addplot[name path=loB, draw=none, forget plot] coordinates {
    (1, 0.17) (4, 0.17) (8, 0.171) (12, 0.167) (16, 0.155) (20, 0.139) (24, 0.121) (28, 0.103) (32, 0.0884)
};
\addplot[colorGen, fill opacity=0.18, forget plot] fill between[of=upB and loB];
\addplot[color=colorGen, mark=diamond*, mark size=5.5pt] coordinates {
    (1, 0.173) (4, 0.176) (8, 0.183) (12, 0.174) (16, 0.164) (20, 0.153) (24, 0.135) (28, 0.12) (32, 0.103)
};
\legend{Gen. error}

\nextgroupplot[
    title={(a) SGC: Wasserstein}, xlabel={Depth},
    ymin=0, ymax=0.12,
    ytick={0.03,0.06,0.09}, yticklabels={.03,.06,.09}]
\addplot[color=colorOracle, mark=square*, mark size=5.0pt] coordinates {
    (1, 0.0192) (4, 0.0124) (8, 0.0105) (12, 0.00974) (16, 0.00927) (20, 0.00895) (24, 0.00872) (28, 0.00855) (32, 0.00841)
};
\addplot[color=colorGlobal, mark=triangle*, mark size=5.5pt] coordinates {
    (1, 0.105) (4, 0.0809) (8, 0.0728) (12, 0.0688) (16, 0.0663) (20, 0.0645) (24, 0.0632) (28, 0.0621) (32, 0.0612)
};
\legend{$\mathcal{W}_C$, $\mathcal{W}_S$}

\nextgroupplot[
    title={(b) SGC: Gen. error}, xlabel={Depth},
    ymin=0.26, ymax=0.34,
    ytick={0.28,0.30,0.32}, yticklabels={.28,.30,.32}]
\addplot[name path=upF, draw=none, forget plot] coordinates {
    (1, 0.308) (4, 0.322) (8, 0.328) (12, 0.323) (16, 0.323) (20, 0.322) (24, 0.314) (28, 0.306) (32, 0.299)
};
\addplot[name path=loF, draw=none, forget plot] coordinates {
    (1, 0.286) (4, 0.309) (8, 0.303) (12, 0.304) (16, 0.303) (20, 0.308) (24, 0.3) (28, 0.283) (32, 0.28)
};
\addplot[colorGen, fill opacity=0.18, forget plot] fill between[of=upF and loF];
\addplot[color=colorGen, mark=diamond*, mark size=5.5pt] coordinates {
    (1, 0.297) (4, 0.315) (8, 0.315) (12, 0.314) (16, 0.313) (20, 0.315) (24, 0.307) (28, 0.295) (32, 0.289)
};
\legend{Gen. error}

\end{groupplot}


\begin{axis}[name=panelC,
    at={($(group c2r1.east)+(3.5cm,0)$)}, anchor=west,
    ymode=log, ymin=0.01, ymax=100,
    ytick={0.1,1,10}, yticklabels={.1,1,10}]
\addplot[color=colorOracle, mark=square*, mark size=5.0pt] coordinates {
    (1, 0.042) (4, 0.49) (8, 0.553) (12, 0.301) (16, 0.349) (20, 0.155) (24, 0.0355) (28, 0.085) (32, 0.0926)
};
\addplot[color=colorGlobal, mark=triangle*, mark size=5.5pt] coordinates {
    (1, 0.761) (4, 16.7) (8, 22.3) (12, 11.3) (16, 12.1) (20, 7.09) (24, 0.974) (28, 2.42) (32, 2.71)
};
\legend{$\mathcal{W}_C$, $\mathcal{W}_S$}
\end{axis}

\begin{axis}[name=panelD,
    at={($(panelC.east)+(1.8cm,0)$)}, anchor=west,
    ymin=0, ymax=0.24,
    ytick={0.06,0.12,0.18}, yticklabels={.06,.12,.18}]
\addplot[name path=upD, draw=none, forget plot] coordinates {
    (1, 0.2) (4, 0.225) (8, 0.179) (12, 0.16) (16, 0.14) (20, 0.072) (24, 0.0327) (28, 0.0457) (32, 0.0583)
};
\addplot[name path=loD, draw=none, forget plot] coordinates {
    (1, 0.183) (4, 0.195) (8, 0.134) (12, 0.0703) (16, 0.0539) (20, 0.0252) (24, 0.0046) (28, 0.0217) (32, 0.0239)
};
\addplot[colorGen, fill opacity=0.18, forget plot] fill between[of=upD and loD];
\addplot[color=colorGen, mark=diamond*, mark size=5.5pt] coordinates {
    (1, 0.192) (4, 0.21) (8, 0.157) (12, 0.115) (16, 0.097) (20, 0.0486) (24, 0.0187) (28, 0.0337) (32, 0.0411)
};
\legend{Gen. error}
\end{axis}

\begin{axis}[name=panelG,
    at={($(group c2r2.east)+(3.5cm,0)$)}, anchor=west,
    title={(c) GCN: Wasserstein}, xlabel={Depth},
    ymode=log, ymin=0.01, ymax=100,
    ytick={0.1,1,10}, yticklabels={.1,1,10}]
\addplot[color=colorOracle, mark=square*, mark size=5.0pt] coordinates {
    (1, 0.0405) (4, 0.385) (8, 0.58) (12, 0.537) (16, 0.118) (20, 0.0123) (24, 0.0228) (28, 0.0138) (32, 0.0215)
};
\addplot[color=colorGlobal, mark=triangle*, mark size=5.5pt] coordinates {
    (1, 0.583) (4, 11.3) (8, 23.6) (12, 19.4) (16, 4.3) (20, 0.14) (24, 0.304) (28, 0.172) (32, 0.304)
};
\legend{$\mathcal{W}_C$, $\mathcal{W}_S$}
\end{axis}

\begin{axis}[name=panelH,
    at={($(panelG.east)+(1.8cm,0)$)}, anchor=west,
    title={(d) GCN: Gen. error}, xlabel={Depth},
    ymin=0, ymax=0.4,
    ytick={0.10,0.20,0.30}, yticklabels={.10,.20,.30}]
\addplot[name path=upH, draw=none, forget plot] coordinates {
    (1, 0.325) (4, 0.37) (8, 0.348) (12, 0.297) (16, 0.159) (20, 0.0271) (24, 0.0356) (28, 0.0355) (32, 0.0453)
};
\addplot[name path=loH, draw=none, forget plot] coordinates {
    (1, 0.294) (4, 0.314) (8, 0.242) (12, 0.2) (16, -0.0203) (20, 0.00909) (24, 0.0105) (28, 0.00223) (32, 0.00307)
};
\addplot[colorGen, fill opacity=0.18, forget plot] fill between[of=upH and loH];
\addplot[color=colorGen, mark=diamond*, mark size=5.5pt] coordinates {
    (1, 0.31) (4, 0.342) (8, 0.295) (12, 0.248) (16, 0.0694) (20, 0.0181) (24, 0.023) (28, 0.0189) (32, 0.0242)
};
\legend{Gen. error}
\end{axis}

\node[rotate=90, font=\Huge]
    at ($(group c1r1.outer west |- group c1r1.center)+(-0.7cm,0)$) {Cora};
\node[rotate=90, font=\Huge]
    at ($(group c1r2.outer west |- group c1r2.center)+(-0.7cm,0)$) {CiteSeer};

\node[rotate=90, font=\Huge, opacity=0]
    at ($(panelD.outer east |- panelD.center)+(0.7cm,0)$) {Cora};
\node[rotate=90, font=\Huge, opacity=0]
    at ($(panelH.outer east |- panelH.center)+(0.7cm,0)$) {CiteSeer};

\end{tikzpicture}
}
\caption{Depth analysis on Cora (top row) and CiteSeer (bottom row) for SGC and GCN, showing Wasserstein distances $\gW_C, \gW_S$ and the empirical generalization error denoted Gen. error in the figure at varying depths. Shaded regions indicate $\pm 1$ standard deviation across runs.}
\label{fig:depth_analysis_2x4}
\end{figure*}

\cref{fig:depth_analysis_2x4} illustrates the depth analysis on Cora (top row) and CiteSeer (bottom row), each with SGC and GCN. 
For SGC, the Wasserstein distances ($\gW_C, \gW_S$) exponentially decay with depth (\cref{fig:depth_analysis_2x4}(a)), consistent with \cref{prop:w1-sgc}. As depth increases, intra-class concentration strengthens ($\gW_C \downarrow$), benefiting generalization, while inter-class separation weakens ($\gW_S \downarrow$), harming it.
The trade-off is clearly reflected in the empirical generalization error shown in \cref{fig:depth_analysis_2x4}(b), which exhibits a non-monotonic dependence on depth: the error initially increases as the loss of inter-class separation dominates, but eventually decreases as the strengthening of intra-class concentration outweighs the loss of separation. The same pattern holds for GCN: although the trajectories of $\gW_C, \gW_S$ differ due to the spectral norm of the learned weights (i.e., $\beta$ in~\cref{prop:w1-gcn}), the fundamental interpretation remains unchanged. 
Such non-monotonic behavior contrasts with the stability-based bound of~\citet{congNEURIPS2021_provable}, which grows monotonically and exponentially with depth for GCN-type models. 
Our representation-based bound, by contrast, provides a more reliable basis for understanding the depth-generalization relationship. Additional results on heterophilic graph datasets are provided in~\cref{apdx_sec:depth_hetero}.

\section{Beyond graph node classification}\label{sec:beyond}
While the OT bound is most naturally suited to GNN-based node classification, it is not inherently restricted to the graph domain. To examine the generality and robustness of our bound across transductive settings, we evaluate it on image classification, with both supervised and semi-supervised algorithms, both for rank correlation with the generalization error and as a training regularizer.

We evaluate our bound on two image datasets, CIFAR-10 and SVHN, in a transductive setup with limited labels. For supervised learning, we train a CNN with cross-entropy loss with $30\%$ of the data labeled. For semi-supervised learning, we use $2\%$ labeled data with four representative algorithms that actively leverage unlabeled features: LP-DeepSSL~\citep{iscen2019label}, Mean Teacher~\citep{tarvainen2017mean}, FixMatch~\citep{sohn2020fixmatch}, and SimMatch~\citep{zheng2022simmatch}. More details of the implementation are provided in \cref{apdx_sec:exp-image}. Additional supervised classification experiments on image and tabular data are provided in~\cref{appdx:sec:experiment_others}.

\begin{table*}[t]
  \caption{Rank correlation between empirical error gaps and generalization bounds on image classification across supervised and semi-supervised settings.}
  \centering
  \small
  \resizebox{1\linewidth}{!}{%
  \begin{tabular}{l *{10}{>{\centering\arraybackslash}p{0.085\linewidth}}}
    \toprule
                              & \multicolumn{2}{c}{\raisebox{-2.0ex}[0pt][0pt]{Supervised}}
                              & \multicolumn{8}{c}{Semi-supervised} \\
    \cmidrule(lr){4-11}
                              & \multicolumn{2}{c}{}
                              & \multicolumn{2}{c}{LP-DeepSSL}
                              & \multicolumn{2}{c}{Mean Teacher}
                              & \multicolumn{2}{c}{FixMatch}
                              & \multicolumn{2}{c}{SimMatch} \\
    \cmidrule(lr){2-3} \cmidrule(lr){4-5} \cmidrule(lr){6-7} \cmidrule(lr){8-9} \cmidrule(lr){10-11}
                              & \mbox{CIFAR-10} & SVHN
                              & \mbox{CIFAR-10} & SVHN
                              & \mbox{CIFAR-10} & SVHN
                              & \mbox{CIFAR-10} & SVHN
                              & \mbox{CIFAR-10} & SVHN \\
    \midrule
    PAC bound                 & -0.14  & \phantom{-}0.00
                              & \cellcolor[rgb]{0.835,0.915,0.859}\phantom{-}0.47   & -0.94
                              & \cellcolor[rgb]{0.966,0.982,0.971}\phantom{-}0.10   & \cellcolor[rgb]{0.864,0.930,0.884}\phantom{-}0.39
                              & -0.73  & -0.44
                              & \cellcolor[rgb]{0.720,0.856,0.760}\phantom{-}0.80   & \cellcolor[rgb]{0.727,0.860,0.766}\phantom{-}0.78 \\
    OT-oracle                 & \cellcolor[rgb]{0.698,0.845,0.741}\phantom{-}0.86   & \cellcolor[rgb]{0.736,0.864,0.774}\phantom{-}0.75
                              & \cellcolor[rgb]{0.709,0.850,0.751}\phantom{-}0.83   & \cellcolor[rgb]{0.671,0.831,0.718}\phantom{-}0.94
                              & \cellcolor[rgb]{0.804,0.899,0.832}\phantom{-}0.56   & \cellcolor[rgb]{0.734,0.863,0.772}\phantom{-}0.76
                              & \cellcolor[rgb]{0.776,0.885,0.808}\phantom{-}0.64   & \cellcolor[rgb]{0.716,0.854,0.756}\phantom{-}0.81
                              & \cellcolor[rgb]{0.730,0.861,0.768}\phantom{-}0.77   & \cellcolor[rgb]{0.695,0.843,0.739}\phantom{-}0.87 \\
    OT bound                  & \cellcolor[rgb]{0.698,0.845,0.741}\phantom{-}0.86   & \cellcolor[rgb]{0.736,0.864,0.774}\phantom{-}0.75
                              & \cellcolor[rgb]{0.730,0.861,0.768}\phantom{-}0.77   & \cellcolor[rgb]{0.673,0.832,0.720}\phantom{-}0.93
                              & \cellcolor[rgb]{0.793,0.893,0.822}\phantom{-}0.59   & \cellcolor[rgb]{0.695,0.843,0.739}\phantom{-}0.87
                              & \cellcolor[rgb]{0.801,0.898,0.829}\phantom{-}0.57   & \cellcolor[rgb]{0.713,0.852,0.754}\phantom{-}0.82
                              & \cellcolor[rgb]{0.761,0.877,0.795}\phantom{-}0.68   & \cellcolor[rgb]{0.688,0.840,0.733}\phantom{-}0.89 \\
    \bottomrule
  \end{tabular}%
  }
  \label{tab:combined-corr}
\end{table*}

\paragraph{Correlation}
\label{subsec:corr_image}
We evaluate our bound as in~\cref{subsec:exp_corr_gnn}, measuring the rank correlation between the bound and the empirical generalization error. The results are shown in~\cref{tab:combined-corr}. OT-oracle and OT bound consistently align with the generalization error across all algorithms and datasets, achieving strong positive correlations ($\geq 0.7$) in $75\%$ of the cases, whereas the PAC baseline largely fails except for SimMatch. These results demonstrate that our bound generalizes well to image classification tasks, including the low-label regime.

\begin{table*}[t]
\caption{Effect of our OT regularizer on image classification accuracy. $\mathcal{L}_\mathrm{orig}$ is the unregularized baseline; $+$ OT reg.\ adds our regularizer. Mean $\pm$ std over five runs; $\Delta$ is the gain over the baseline.}
\centering
\resizebox{1.0\linewidth}{!}{%
\begin{tabular}{l cc cc cc cc cc}
  \toprule
                            & \multicolumn{2}{c}{\raisebox{-2.0ex}[0pt][0pt]{Supervised}}
                            & \multicolumn{8}{c}{Semi-supervised} \\
  \cmidrule(lr){4-11}
                            & \multicolumn{2}{c}{}
                            & \multicolumn{2}{c}{LP-DeepSSL}
                            & \multicolumn{2}{c}{Mean Teacher}
                            & \multicolumn{2}{c}{FixMatch}
                            & \multicolumn{2}{c}{SimMatch} \\
  \cmidrule(lr){2-3} \cmidrule(lr){4-5} \cmidrule(lr){6-7} \cmidrule(lr){8-9} \cmidrule(lr){10-11}
                            & CIFAR-10 & SVHN
                            & CIFAR-10 & SVHN
                            & CIFAR-10 & SVHN
                            & CIFAR-10 & SVHN
                            & CIFAR-10 & SVHN \\
  \midrule\midrule
  $\mathcal{L}_\mathrm{orig}$                      & 76.81{\scriptsize $\pm$0.24}
                            & 88.15{\scriptsize $\pm$1.86}
                            & 70.97{\scriptsize $\pm$0.10}     & 90.71{\scriptsize $\pm$0.06}
                            & 80.60{\scriptsize $\pm$0.33}     & 97.47{\scriptsize $\pm$0.04}
                            & 93.91{\scriptsize $\pm$0.07}     & 97.37{\scriptsize $\pm$0.01}
                            & 90.58{\scriptsize $\pm$0.67}     & 97.19{\scriptsize $\pm$0.05} \\
  + OT reg.                 & 77.71{\scriptsize $\pm$0.27}
                            & 91.00{\scriptsize $\pm$0.37}
                            & 71.17{\scriptsize $\pm$0.17}     & 90.92{\scriptsize $\pm$0.07}
                            & 81.00{\scriptsize $\pm$0.28}     & 97.50{\scriptsize $\pm$0.01}
                            & 93.97{\scriptsize $\pm$0.04}     & 97.42{\scriptsize $\pm$0.02}
                            & 92.41{\scriptsize $\pm$0.31}     & 97.26{\scriptsize $\pm$0.05} \\
  $\Delta$                  & +0.90 & +2.85
                            & +0.20 & +0.21
                            & +0.41 & +0.03
                            & +0.06 & +0.05
                            & +1.83 & +0.07 \\
  \bottomrule
\end{tabular}%
}
\label{tab:reg-others}
\end{table*}

\paragraph{Regularization}
\label{subsec:reg_image}
We use our OT bound as a regularizer in the training loss: $\mathcal{L} = \mathcal{L}_\mathrm{orig} + \lambda \mathcal{R}$, where $\mathcal{L}_\mathrm{orig}$ is the original loss of each algorithm and $\mathcal{R}$ is the surrogate of the OT bound, as explained in \cref{subsec:expo_reg_gnn}. 
\cref{tab:reg-others} reports the accuracy with and without our OT regularizer. The results show that adding the regularizer consistently improves over the baseline across all datasets, algorithms, and architectures, with the gains larger than the baseline standard deviation in eight out of ten cases. While the gains are smaller than in the GNN setting, where the regularizer most directly corrects the message-passing geometry, the consistent improvements across diverse algorithms confirm the bound's broader applicability. 
\section{Conclusion}
In this work, we establish an OT-based generalization bound for distribution-free transductive learning. Our bound connects generalization to the geometric properties of the feature space and is tractable. We then specialize the bound to GNNs, where the resulting depth-dependent expressions capture the non-monotonic dependence of generalization on depth and reveal a connection to oversmoothing. Experiments on GNN node classification show that our bound consistently aligns with the empirical generalization error and serves as an effective regularizer during training.

\paragraph{Limitations}
\cref{thm:classwise-ot} holds uniformly over $\mathcal{F}\circ\Phi$ and provides an a posteriori bound for a trained encoder and classifier, leading to access to richer information than the RC bound.
The bound shows how within-class concentration of the learned features helps control the generalization gap, but does not by itself predict whether a given architecture or optimizer reaches such a favorable geometry.


\begin{ack}
  This work was supported by the National Research Foundation of Korea (NRF) grants funded by the Korea government (MSIT) (RS-2024-00337955; RS-2025-00560062); by the Institute of Information \& Communications Technology Planning \& Evaluation (IITP) grants funded by the Korea government (MSIT) (RS-2024-00457882, Artificial Intelligence Research Hub Project; RS-2019-II191906, Artificial Intelligence Graduate School Program (POSTECH)); and by IITP grants funded by the Korea government (MSIT) under the Leading Generative AI Human Resources Development Program (IITP-2026-RS-2026-25546560) and the AI Star Fellowship Support Program (IITP-2026-RS-2026-25616313).
\end{ack}

\bibliography{icml2026}

@article{musae,
          author = {Rozemberczki, Benedek and Allen, Carl and Sarkar, Rik},
          title = {{Multi-Scale Attributed Node Embedding}},
          journal = {Journal of Complex Networks},
          volume = {9},
          number = {2},
          year = {2021},
}

@article{krogh1991simple,
  title={A simple weight decay can improve generalization},
  author={Krogh, Anders and Hertz, John},
  journal={Advances in neural information processing systems},
  volume={4},
  year={1991}
}

@article{srivastava2014dropout,
  title={Dropout: a simple way to prevent neural networks from overfitting},
  author={Srivastava, Nitish and Hinton, Geoffrey and Krizhevsky, Alex and Sutskever, Ilya and Salakhutdinov, Ruslan},
  journal={The journal of machine learning research},
  volume={15},
  number={1},
  pages={1929--1958},
  year={2014},
  publisher={JMLR. org}
}

@inproceedings{Rong2020DropEdge,
title={DropEdge: Towards Deep Graph Convolutional Networks on Node Classification},
author={Yu Rong and Wenbing Huang and Tingyang Xu and Junzhou Huang},
booktitle={International Conference on Learning Representations},
year={2020},
url={https://openreview.net/forum?id=Hkx1qkrKPr}
}

@inproceedings{park2024mitigating,
  title={Mitigating oversmoothing through reverse process of GNNs for heterophilic graphs},
  author={Park, MoonJeong and Heo, Jaeseung and Kim, Dongwoo},
  booktitle={International conference on machine learning},
  year={2024}
}

@inproceedings{cuturi2013sinkhorn,
 author = {Cuturi, Marco},
 booktitle = {Advances in Neural Information Processing Systems},
 editor = {C.J. Burges and L. Bottou and M. Welling and Z. Ghahramani and K. Weinberger},
 pages = {},
 publisher = {Curran Associates, Inc.},
 title = {Sinkhorn Distances: Lightspeed Computation of Optimal Transport},
 volume = {26},
 year = {2013}
}

@article{rusch2023survey,
  title={A survey on oversmoothing in graph neural networks},
  author={Rusch, T Konstantin and Bronstein, Michael M and Mishra, Siddhartha},
  journal={arXiv preprint arXiv:2303.10993},
  year={2023}
}

@incollection{von2011statistical,
  title={Statistical learning theory: Models, concepts, and results},
  author={Von Luxburg, Ulrike and Sch{\"o}lkopf, Bernhard},
  booktitle={Handbook of the History of Logic},
  volume={10},
  pages={651--706},
  year={2011},
  publisher={Elsevier}
}

@article{bartlett2002rademacher,
  title={Rademacher and gaussian complexities: Risk bounds and structural results},
  author={Bartlett, Peter L and Mendelson, Shahar},
  journal={Journal of machine learning research},
  volume={3},
  number={Nov},
  pages={463--482},
  year={2002}
}

@article{sen2008collective,
  title={Collective classification in network data},
  author={Sen, Prithviraj and Namata, Galileo and Bilgic, Mustafa and Getoor, Lise and Galligher, Brian and Eliassi-Rad, Tina},
  journal={AI magazine},
  volume={29},
  number={3},
  pages={93--93},
  year={2008}
}

@inproceedings{yang2016revisiting,
  title={Revisiting semi-supervised learning with graph embeddings},
  author={Yang, Zhilin and Cohen, William and Salakhudinov, Ruslan},
  booktitle={International conference on machine learning},
  pages={40--48},
  year={2016},
  organization={PMLR}
}

@inproceedings{mcauley2015image,
  title={Image-based recommendations on styles and substitutes},
  author={McAuley, Julian and Targett, Christopher and Shi, Qinfeng and Van Den Hengel, Anton},
  booktitle={Proceedings of the 38th international ACM SIGIR conference on research and development in information retrieval},
  pages={43--52},
  year={2015}
}

@article{solomon2022k,
  title={k-variance: A clustered notion of variance},
  author={Solomon, Justin and Greenewald, Kristjan and Nagaraja, Haikady},
  journal={SIAM Journal on Mathematics of Data Science},
  volume={4},
  number={3},
  pages={957--978},
  year={2022},
  publisher={SIAM}
}

@book{vapnik2006estimation,
  title={Estimation of dependences based on empirical data},
  author={Vapnik, Vladimir},
  year={2006},
  publisher={Springer Science \& Business Media}
}

@inproceedings{platonov2023critical,
  title={A critical look at evaluation of GNNs under heterophily: Are we really making progress?},
  author={Platonov, Oleg and Kuznedelev, Denis and Diskin, Michael and Babenko, Artem and Prokhorenkova, Liudmila},
  booktitle={The Eleventh International Conference on Learning Representations},
  year={2023}
}

@inproceedings{
Oono2020Graph,
title={Graph Neural Networks Exponentially Lose Expressive Power for Node Classification},
author={Kenta Oono and Taiji Suzuki},
booktitle={International Conference on Learning Representations},
year={2020},
}

@misc{natekar2020representationcompwin,
      title={Representation Based Complexity Measures for Predicting Generalization in Deep Learning}, 
      author={Parth Natekar and Manik Sharma},
      year={2020},
      eprint={2012.02775},
      archivePrefix={arXiv},
      primaryClass={cs.LG}
}

@article{jiang2020neuripscomp,
  title={Neurips 2020 competition: Predicting generalization in deep learning},
  author={Jiang, Yiding and Foret, Pierre and Yak, Scott and Roy, Daniel M and Mobahi, Hossein and Dziugaite, Gintare Karolina and Bengio, Samy and Gunasekar, Suriya and Guyon, Isabelle and Neyshabur, Behnam},
  journal={arXiv preprint arXiv:2012.07976},
  year={2020}
}

@inproceedings{
maskey2022generalization,
title={Generalization Analysis of Message Passing Neural Networks on Large Random Graphs},
author={Sohir Maskey and Ron Levie and Yunseok Lee and Gitta Kutyniok},
booktitle={Advances in Neural Information Processing Systems},
editor={Alice H. Oh and Alekh Agarwal and Danielle Belgrave and Kyunghyun Cho},
year={2022},
}

@inproceedings{duNEURIPS2019fusing,
 author = {Du, Simon S and Hou, Kangcheng and Salakhutdinov, Russ R and Poczos, Barnabas and Wang, Ruosong and Xu, Keyulu},
 booktitle = {Advances in Neural Information Processing Systems},
 editor = {H. Wallach and H. Larochelle and A. Beygelzimer and F. d\textquotesingle Alch\'{e}-Buc and E. Fox and R. Garnett},
 pages = {},
 publisher = {Curran Associates, Inc.},
 title = {Graph Neural Tangent Kernel: Fusing Graph Neural Networks with Graph Kernels},
 volume = {32},
 year = {2019}
}

@inproceedings{
liao2021pac,
title={A {\{}PAC{\}}-Bayesian Approach to Generalization Bounds for Graph Neural Networks},
author={Renjie Liao and Raquel Urtasun and Richard Zemel},
booktitle={International Conference on Learning Representations},
year={2021},
}

@inproceedings{levieNEURIPS2023graphon,
 author = {Levie, Ron},
 booktitle = {Advances in Neural Information Processing Systems},
 editor = {A. Oh and T. Naumann and A. Globerson and K. Saenko and M. Hardt and S. Levine},
 pages = {64482--64525},
 publisher = {Curran Associates, Inc.},
 title = {A graphon-signal analysis of graph neural networks},
 volume = {36},
 year = {2023}
}

@inproceedings{trans_stab_2008,
author = {Cortes, Corinna and Mohri, Mehryar and Pechyony, Dmitry and Rastogi, Ashish},
title = {Stability of transductive regression algorithms},
year = {2008},
isbn = {9781605582054},
publisher = {Association for Computing Machinery},
address = {New York, NY, USA},
doi = {10.1145/1390156.1390179},
booktitle = {Proceedings of the 25th International Conference on Machine Learning},
pages = {176–183},
numpages = {8},
location = {Helsinki, Finland},
series = {ICML '08}
}

@InProceedings{permute_RC_2015,
author="Tolstikhin, Ilya
and Zhivotovskiy, Nikita
and Blanchard, Gilles",
editor="Chaudhuri, Kamalika
and GENTILE, CLAUDIO
and Zilles, Sandra",
title="Permutational Rademacher Complexity",
booktitle="Algorithmic Learning Theory",
year="2015",
publisher="Springer International Publishing",
address="Cham",
pages="209--223",
isbn="978-3-319-24486-0"
}

@InProceedings{pmlr-v35-tolstikhin14,
  title = 	 {Localized Complexities for Transductive Learning},
  author = 	 {Tolstikhin, Ilya and Blanchard, Gilles and Kloft, Marius},
  booktitle = 	 {Proceedings of The 27th Conference on Learning Theory},
  pages = 	 {857--884},
  year = 	 {2014},
  editor = 	 {Balcan, Maria Florina and Feldman, Vitaly and Szepesvári, Csaba},
  volume = 	 {35},
  series = 	 {Proceedings of Machine Learning Research},
  address = 	 {Barcelona, Spain},
  month = 	 {13--15 Jun},
  publisher =    {PMLR},
}

@inproceedings{congNEURIPS2021_provable,
 author = {Cong, Weilin and Ramezani, Morteza and Mahdavi, Mehrdad},
 booktitle = {Advances in Neural Information Processing Systems},
 editor = {M. Ranzato and A. Beygelzimer and Y. Dauphin and P.S. Liang and J. Wortman Vaughan},
 pages = {9936--9949},
 publisher = {Curran Associates, Inc.},
 title = {On Provable Benefits of Depth in Training Graph Convolutional Networks},
 volume = {34},
 year = {2021}
}

@inproceedings{EsserNEURIPS2021_learning,
 author = {Esser, Pascal and Chennuru Vankadara, Leena and Ghoshdastidar, Debarghya},
 booktitle = {Advances in Neural Information Processing Systems},
 editor = {M. Ranzato and A. Beygelzimer and Y. Dauphin and P.S. Liang and J. Wortman Vaughan},
 pages = {27043--27056},
 publisher = {Curran Associates, Inc.},
 title = {Learning Theory Can (Sometimes) Explain Generalisation in Graph Neural Networks},
 volume = {34},
 year = {2021}
}

@inproceedings{tang2023towards,
  title={Towards understanding generalization of graph neural networks},
  author={Tang, Huayi and Liu, Yong},
  booktitle={International Conference on Machine Learning},
  pages={33674--33719},
  year={2023},
  organization={PMLR}
}

@InProceedings{pmlr-v119-garg20c,
  title = 	 {Generalization and Representational Limits of Graph Neural Networks},
  author =       {Garg, Vikas and Jegelka, Stefanie and Jaakkola, Tommi},
  booktitle = 	 {Proceedings of the 37th International Conference on Machine Learning},
  pages = 	 {3419--3430},
  year = 	 {2020},
  editor = 	 {III, Hal Daumé and Singh, Aarti},
  volume = 	 {119},
  series = 	 {Proceedings of Machine Learning Research},
  month = 	 {13--18 Jul},
  publisher =    {PMLR}
}

@article{ZHOU202197,
title = {The generalization error of graph convolutional networks may enlarge with more layers},
journal = {Neurocomputing},
volume = {424},
pages = {97-106},
year = {2021},
issn = {0925-2312},
doi = {https://doi.org/10.1016/j.neucom.2020.10.109},
author = {Xianchen Zhou and Hongxia Wang}
}

@InProceedings{zhang20y_fast,
  title = 	 {Fast Learning of Graph Neural Networks with Guaranteed Generalizability: One-hidden-layer Case},
  author =       {Zhang, Shuai and Wang, Meng and Liu, Sijia and Chen, Pin-Yu and Xiong, Jinjun},
  booktitle = 	 {Proceedings of the 37th International Conference on Machine Learning},
  pages = 	 {11268--11277},
  year = 	 {2020},
  editor = 	 {III, Hal Daumé and Singh, Aarti},
  volume = 	 {119},
  series = 	 {Proceedings of Machine Learning Research},
  month = 	 {13--18 Jul},
  publisher =    {PMLR}
}

@inproceedings{verma2019stab,
author = {Verma, Saurabh and Zhang, Zhi-Li},
title = {Stability and Generalization of Graph Convolutional Neural Networks},
year = {2019},
isbn = {9781450362016},
publisher = {Association for Computing Machinery},
address = {New York, NY, USA},
doi = {10.1145/3292500.3330956},
booktitle = {Proceedings of the 25th ACM SIGKDD International Conference on Knowledge Discovery \& Data Mining},
pages = {1539–1548},
numpages = {10},
location = {Anchorage, AK, USA},
series = {KDD '19}
}

@article{SCARSELLI2018248,
title = {The Vapnik–Chervonenkis dimension of graph and recursive neural networks},
journal = {Neural Networks},
volume = {108},
pages = {248-259},
year = {2018},
issn = {0893-6080},
doi = {https://doi.org/10.1016/j.neunet.2018.08.010},
author = {Franco Scarselli and Ah Chung Tsoi and Markus Hagenbuchner}
}

@InProceedings{pac2014begin,
  title = 	 {{PAC-Bayesian Theory for Transductive Learning}},
  author = 	 {Bégin, Luc and Germain, Pascal and Laviolette, François and Roy, Jean-Francis},
  booktitle = 	 {Proceedings of the Seventeenth International Conference on Artificial Intelligence and Statistics},
  pages = 	 {105--113},
  year = 	 {2014},
  editor = 	 {Kaski, Samuel and Corander, Jukka},
  volume = 	 {33},
  series = 	 {Proceedings of Machine Learning Research},
  address = 	 {Reykjavik, Iceland},
  month = 	 {22--25 Apr},
  publisher =    {PMLR}
}

@inproceedings{el2006stable,
  title={Stable transductive learning},
  author={El-Yaniv, Ran and Pechyony, Dmitry},
  booktitle={International Conference on Computational Learning Theory},
  pages={35--49},
  year={2006},
  organization={Springer}
}

@inproceedings{
wu2023non-asym,
title={A Non-Asymptotic Analysis of Oversmoothing in Graph Neural Networks},
author={Xinyi Wu and Zhengdao Chen and William Wei Wang and Ali Jadbabaie},
booktitle={The Eleventh International Conference on Learning Representations },
year={2023},
}

@inproceedings{
li2025towards,
title={Towards Bridging Generalization and Expressivity of Graph Neural Networks},
author={Shouheng Li and Floris Geerts and Dongwoo Kim and Qing Wang},
booktitle={The Thirteenth International Conference on Learning Representations},
year={2025},
}

@article{el2009transductive,
  title={Transductive rademacher complexity and its applications},
  author={El-Yaniv, Ran and Pechyony, Dmitry},
  journal={Journal of Artificial Intelligence Research},
  volume={35},
  pages={193--234},
  year={2009}
}

@article{chuang2021measuring,
  title={Measuring generalization with optimal transport},
  author={Chuang, Ching-Yao and Mroueh, Youssef and Greenewald, Kristjan and Torralba, Antonio and Jegelka, Stefanie},
  journal={Advances in neural information processing systems},
  volume={34},
  pages={8294--8306},
  year={2021}
}

@article{cai2020note,
title={A note on over-smoothing for graph neural networks},
year={2020},
author={Cai, Chen and Wang, Yusu},
booktitle={ICML Graph Representation Learning and Beyond (GRL+) Workshop},
}

@inproceedings{li2018deeper,
  title={Deeper insights into graph convolutional networks for semi-supervised learning},
  author={Li, Qimai and Han, Zhichao and Wu, Xiao-Ming},
  booktitle={Proceedings of the AAAI conference on artificial intelligence},
  volume={32},
  year={2018}
}

@inproceedings{wu2019simplifying,
  title={Simplifying graph convolutional networks},
  author={Wu, Felix and Souza, Amauri and Zhang, Tianyi and Fifty, Christopher and Yu, Tao and Weinberger, Kilian},
  booktitle={International conference on machine learning},
  pages={6861--6871},
  year={2019},
  organization={Pmlr}
}

@article{velivckovic2017graph,
title={Graph Attention Networks},
author={Petar Veličković and Guillem Cucurull and Arantxa Casanova and Adriana Romero and Pietro Liò and Yoshua Bengio},
booktitle={International Conference on Learning Representations},
year={2018},
}

@article{hamilton2017inductive,
  title={Inductive representation learning on large graphs},
  author={Hamilton, Will and Ying, Zhitao and Leskovec, Jure},
  journal={Advances in neural information processing systems},
  volume={30},
  year={2017}
}

@article{kipf2016semi,
title={Semi-Supervised Classification with Graph Convolutional Networks},
author={Thomas N. Kipf and Max Welling},
booktitle={International Conference on Learning Representations},
year={2017},
}

@inproceedings{chen2020simple,
  title={Simple and deep graph convolutional networks},
  author={Chen, Ming and Wei, Zhewei and Huang, Zengfeng and Ding, Bolin and Li, Yaliang},
  booktitle={International conference on machine learning},
  pages={1725--1735},
  year={2020},
  organization={PMLR}
}

@article{heo2025influence,
  title={Influence Functions for Edge Edits in Non-Convex Graph Neural Networks},
  author={Heo, Jaeseung and Yun, Kyeongheung and Yoon, Seokwon and Park, MoonJeong and Ok, Jungseul and Kim, Dongwoo},
  journal={arXiv preprint arXiv:2506.04694},
  year={2025}
}

@inproceedings{
keriven2022not,
title={Not too little, not too much: a theoretical analysis of graph (over)smoothing},
author={Nicolas Keriven},
booktitle={The First Learning on Graphs Conference},
year={2022},
}

@article{sohn2020fixmatch,
  title={Fixmatch: Simplifying semi-supervised learning with consistency and confidence},
  author={Sohn, Kihyuk and Berthelot, David and Carlini, Nicholas and Zhang, Zizhao and Zhang, Han and Raffel, Colin A and Cubuk, Ekin Dogus and Kurakin, Alexey and Li, Chun-Liang},
  journal={Advances in neural information processing systems},
  volume={33},
  pages={596--608},
  year={2020}
}

@inproceedings{iscen2019label,
  title={Label propagation for deep semi-supervised learning},
  author={Iscen, Ahmet and Tolias, Giorgos and Avrithis, Yannis and Chum, Ondrej},
  booktitle={Proceedings of the IEEE/CVF conference on computer vision and pattern recognition},
  pages={5070--5079},
  year={2019}
}

@article{tarvainen2017mean,
  title={Mean teachers are better role models: Weight-averaged consistency targets improve semi-supervised deep learning results},
  author={Tarvainen, Antti and Valpola, Harri},
  journal={Advances in neural information processing systems},
  volume={30},
  year={2017}
}

@inproceedings{zheng2022simmatch,
  title={Simmatch: Semi-supervised learning with similarity matching},
  author={Zheng, Mingkai and You, Shan and Huang, Lang and Wang, Fei and Qian, Chen and Xu, Chang},
  booktitle={Proceedings of the IEEE/CVF conference on computer vision and pattern recognition},
  pages={14471--14481},
  year={2022}
}

@article{wang2022usb,
  title={Usb: A unified semi-supervised learning benchmark for classification},
  author={Wang, Yidong and Chen, Hao and Fan, Yue and Sun, Wang and Tao, Ran and Hou, Wenxin and Wang, Renjie and Yang, Linyi and Zhou, Zhi and Guo, Lan-Zhe and others},
  journal={Advances in Neural Information Processing Systems},
  volume={35},
  pages={3938--3961},
  year={2022}
}

@article{rauchwerger2024generalization,
  title={Generalization, expressivity, and universality of graph neural networks on attributed graphs},
  author={Rauchwerger, Levi and Jegelka, Stefanie and Levie, Ron},
  journal={arXiv preprint arXiv:2411.05464},
  year={2024}
}
\bibliographystyle{plainnat}

\appendix

\newpage
\appendix
\onecolumn

\section{Additional theoretical results on generalization bounds}
\subsection{Additional analysis of \texorpdfstring{$M_c$}{Mc}}
\label{apdx_sec:mc_analysis}

The quantity $M_c(f,\phi)$ is a key term in the bound of
\cref{thm:classwise-ot}.
We derive a lower bound on $M_c$ in terms of the margin $\gamma$ and
the Wasserstein distance.
This result also complements our depth-dependent analysis of the
Wasserstein distance. We illustrate the connection empirically for SGC.

\begin{restatable}[Lower bound on $M_c$]{lemma}{Mclower} 
\label{apdx_lem:Mc_lower}
Fix a split $\pi$, a feature encoder $\phi$, and a classifier $f$
that assigns class scores to the encoded features.
Let $\gamma>0$ and let $c,c'$ be distinct classes. Define
\begin{align*}
S &\coloneq \left\{i\in\gI_{\mathrm{train},c}^{(\pi)}
: \rho_f(\phi(\rvx_i),c)\geq\gamma\right\},\\
T &\coloneq \left\{j\in\testi
: \arg\max_k f_k(\phi(\rvx_j))=c'\right\}.
\end{align*}
If $S$ and $T$ are nonempty, then
\begin{equation}
\label{apdx_eq:Mc_lower}
M_c(f,\phi)\geq
\frac{\gamma}{\mathcal W(\phi_\#\mu_S,\phi_\#\mu_T)}\;.
\end{equation}
\end{restatable}



\paragraph{Connection to SGC depth.}
For SGC, \cref{prop:w1-sgc} provides a depth-dependent upper bound on the
Wasserstein distance between the subsets $S$ and $T$ defined above.
When this distance decreases with depth, maintaining the same margin
$\gamma$ and different predictions imposes a larger lower bound on $M_c$.
\Cref{apdx_tab:Mc_depth} illustrates this relationship for SGC on Cora
with $\gamma=1$, reporting the Wasserstein distance, the resulting
lower bound, and the measured $M_c$ across encoder depths.
The distance decreases and the lower bound increases, while the measured
$M_c$ remains above the lower bound at every depth and increases overall. This trend is consistent with learning a sharper classifier to
maintain the margin $\gamma$ and assign different classes to $S$ and $T$
as their embeddings become closer.

\begin{table}[h]
\centering
\caption{Lower bound and measured $M_c$ for SGC on Cora with $\gamma=1$.}
\label{apdx_tab:Mc_depth}
\small
\setlength{\tabcolsep}{6pt}
\begin{tabular}{lrrrrrr}
\toprule
Depth $\ell$ & 1 & 2 & 4 & 8 & 16 & 32 \\
\midrule
$\mathcal W$ & 0.182 & 0.144 & 0.113 & 0.088 & 0.070 & 0.055 \\
Lower bound $\gamma/\mathcal W$ & 5.50 & 6.93 & 8.88 & 11.43 & 14.37 & 18.23 \\
Measured $M_c$ & 440 & 700 & 1238 & 1450 & 1398 & 1603 \\
\bottomrule
\end{tabular}
\end{table}

\subsection{Global bound}\label{apdx:sec:global}

We provide another generalization bound based on the Wasserstein distance.
Unlike \cref{thm:classwise-ot}, which connects generalization to the class-wise feature distributions, this bound explains generalization in terms of the distance between the overall feature distributions of the training and test sets.

\begin{restatable}[Global bound in the transductive setting]{theorem}{globalbound} 
\label{thm:global-ot} 
    Let $\gamma>0$. For any random split $\pi$, and all $f\circ\phi\in F\circ\Phi$, 
        \begin{equation}\label{eq:main-global-ot}
            R_u (f\circ\phi;\pi) - R_{m,\gamma}(f\circ\phi;\pi) \nonumber
             ~ \le ~
            \frac{M(f,\phi)}{\gamma}\,\mathcal{W}\big(\phi_\#\mu_{\gI_{\rm{train}}^{(\pi)}},\,\phi_\#\mu_{\gI_{\rm{test}}^{(\pi)}}\big)\;, 
        \end{equation}
    where $
        M(f,\phi)
        \coloneqq
        \max_{i,j,y}
        \frac{
            |\rho_f(\phi(\mathbf x_i),y_i) - \rho_f(\phi(\mathbf x_j),y)|
        }{
            \|\phi(\mathbf x_i)-\phi(\mathbf x_j)\|
        }
    $ for $
        i \in \traini, ~ j \in \testi, ~\text{and}~ y \in \gY.
    $
\end{restatable}

We provide a proof of \cref{thm:global-ot} in \cref{apdx_sec:proof of global OT} and experimental results
following the protocol of \cref{sec:exp} in \cref{appdx:sec:gb_results}.
\cref{thm:global-ot} demonstrates that the generalization error is small under three conditions: 1) the distance between the feature distributions of the training and test sets is small, 2) the change rate of the margin of classifier $f$ is small, or 3) the margin of the classifier is large. 
We note that in the transductive setting, the Wasserstein distance between the feature distributions of the training and test sets can be computed in practice, as the entire test set is accessible, except for its labels. 
Furthermore, given that the label space $\gY$ is discrete and finite, we can compute all possible values of $|\rho_f(\phi(x_i),y_i) - \rho_f(\phi(x_j),y)|/\lVert\phi(x_i)-\phi(x_j)\rVert$ where $i\in\traini$, $j\in\testi$ and $y\in\gY$. This ensures we can obtain a valid upper bound $M(f,\phi)$.

\section{Proofs}

\subsection{Proof of OT bound}\label{apdx_sec:proof of class-wise OT}
\classwisebound*

\begin{proof}
The goal of the theorem is to bound the test loss over the $u$ test data, $R_u(f\circ\phi;\pi)$ by the $\gamma$-margin loss over the $m$ training set $R_{m,\gamma}(f\circ\phi;\pi)$ of a model $f\circ\phi$ and permutation $\pi$, where:

    \begin{equation}
    \label{eq:apdx:test_loss}
        R_u(f\circ\phi;\pi)\coloneq \frac{1}{u}\sum_{i\in\testi}
        \mathbbm{1}_{\rho_f(\phi(\mathbf{x}_i),y_i)\leq0}\;,
    \end{equation}
    and
    \begin{equation}
    \label{eq:apdx:margin_loss}
        R_{m,\gamma}(f\circ\phi;\pi)\coloneq \frac{1}{m}\sum_{i\in\traini}
        \mathbbm{1}_{\rho_f(\phi(\mathbf{x}_i),y_i)\leq\gamma}\;,
    \end{equation}
    with $\gamma>0$.

We introduce a margin loss $L_\gamma$ with $\gamma>0$ defined by $L_\gamma(u)=\mathbbm{1}_{u\leq0}+(1-\frac{u}{\gamma})\mathbbm{1}_{0<u\leq\gamma}$, which is $\frac{1}{\gamma}$-Lipschitz and satisfies $\mathbbm{1}_{u\leq0}\leq L_\gamma(u)\leq\mathbbm{1}_{u\leq\gamma}$ for all $u$.
By simplifying $\ell_{\gamma,f}(z,y)\coloneq L_\gamma(\rho_f(z,y))$, we have: 

\begin{align}\label{apdx_eq:BasicGap}
    &\testr
    \leq \frac{1}{u}\sum_{i\in\testi}\ell_{\gamma,f}(\phi(\rvx_i),y_i)\nonumber\\
    =& \frac{1}{m}\sum_{{j}\in\traini}\ell_{\gamma,f}(\phi(\rvx_{j}),y_{j})
    +
    \frac{1}{u}\sum_{i\in\testi}\ell_{\gamma,f}(\phi(\rvx_i){,y_i}{})
    -
    \frac{1}{m}\sum_{{j}\in\traini}\ell_{\gamma,f}(\phi(\rvx_{j}){,y_{j}}{})\nonumber\\
    \leq& \frac{1}{m}\sum_{{j}\in\traini}\ell_{\gamma,f}(\phi(\rvx_{j}),y_{j})
    +\sup_{f\in\gF}\left(
    \frac{1}{u}\sum_{i\in\testi}\ell_{\gamma,f}(\phi(\rvx_i){,y_i}{})
    -
    \frac{1}{m}\sum_{{j}\in\traini}\ell_{\gamma,f}(\phi(\rvx_{j}){,y_{j}}{})
    \right)\nonumber\\
    \leq& \trainr 
    +\sup_{f\in\gF}\left(
    \frac{1}{u}\sum_{i\in\testi}\ell_{\gamma,f}(\phi(\rvx_i){,y_i}{})
    -
    \frac{1}{m}\sum_{{j}\in\traini}\ell_{\gamma,f}(\phi(\rvx_{j}){,y_{j}}{})
    \right)\;.
\end{align}

Define
\begin{equation}
    \Delta^{(\pi)}(f)\coloneq 
    \frac{1}{u}\sum_{i\in\testi}\ell_{\gamma,f}(\phi(\rvx_i),y_i) 
    - \frac{1}{m}\sum_{i\in\traini}\ell_{\gamma,f}(\phi(\rvx_i),y_i)
\end{equation}

We apply the concentration inequality of \citet{el2009transductive} (\cref{apdx_lem:trans-mcdiarmid}) to $g(\pi)\coloneq \sup_{f\in\gF}\Delta^{(\pi)}(f)$.
First, $g$ is $(m,u)$-permutation symmetric in the sense of \cref{apdx_def:def-symmetry}: $\Delta^{(\pi)}(f)$ depends on $\pi$ only through the train/test partition $(\gI_{\mathrm{train}}^{(\pi)},\gI_{\mathrm{test}}^{(\pi)})$, not on the ordering within each block, so $g(\pi)=g(\sigma\circ\pi)$ for all $\sigma\in H_{m,u}$.
Second, for any transposition $\tau_{ij}$ with $i\le m<j$, swapping a single train/test pair changes each empirical mean in $\Delta^{(\pi)}(f)$ by at most $\frac{1}{m}+\frac{1}{u}$, so $|g(\pi)-g(\pi^{(ij)})|\le \omega$ with $\omega\coloneq\frac{1}{m}+\frac{1}{u}$.
Applying \cref{apdx_lem:trans-mcdiarmid} with this $\omega$, we have with probability at least $1-\delta$,
\begin{equation}\label{apdx_eq:concentration}
\sup_{f\in\gF}\left[\Delta^{(\pi)}(f)\right]
\leq \mathbb E_{\pi'}\sup_{f\in\gF}\left[ \Delta^{(\pi')}(f)\right] 
+ \sqrt{\frac{m\,u\,(\frac{1}{m}+\frac{1}{u})^2}{2\,(m+u-\tfrac12)}\left(1-\frac{1}{2\max\{m,u\}}\right)^{-1}\ln\frac1\delta}\;.
\end{equation}

To decompose the $\Delta^{(\pi')}(f)$ into each class-$c$ with $m_c^{(\pi)}\coloneq|\gI_{\mathrm{train},c}^{(\pi)}|$ and $u_c^{(\pi)}\coloneq|\gI_{\mathrm{test},c}^{(\pi)}|$, we get:

\begin{align}\label{apdx_eq:class-wise}
    \sup_{f\in\gF} \left[\Delta^{(\pi')}(f)\right]=
    &\sup_{f\in\gF}\Bigg[
    \frac{1}{u}\sum_{i\in\gI_{\mathrm{test}}^{(\pi')}}\ell_{\gamma,f}(\phi(\rvx_i),y_i) 
    - \frac{1}{m}\sum_{i\in\gI_{\mathrm{train}}^{(\pi')}}\ell_{\gamma,f}(\phi(\rvx_i),y_i)
    \Bigg]\nonumber\\
    =&\sup_{f\in\gF}\Bigg[
    \sum_{c=1}^K\underbrace{\frac{m_c^{(\pi')}}{m}\Big(\frac{1}{u_c^{(\pi')}}\sum_{i\in\gI_{\mathrm{test},c}^{(\pi')}}\ell_{\gamma,f}(\phi(\rvx_i),c)-\frac{1}{m_c^{(\pi')}}\sum_{i\in\gI_{\mathrm{train},c}^{(\pi')}}\ell_{\gamma,f}(\phi(\rvx_i),c)\Big)}_{\coloneq\Delta_c^{(\pi')}(f)}\nonumber\\
    &~~~~~~~~~~~~~~~~~~~~~~~~~~~~~~~~~~~~~~~~~~~~~~~
    +\sum_{c=1}^K(\frac{u_c^{(\pi')}}{u}-\frac{m_c^{(\pi')}}{m})\Big(\frac{1}{u_c^{(\pi')}}\sum_{i\in\gI_{\mathrm{test},c}^{(\pi')}}\ell_{\gamma,f}(\phi(\rvx_i),c)\Big)
    \Bigg]\nonumber\\
    \leq& \sup_{f\in\gF}\left(\sum_{c=1}^K \Delta_c^{(\pi')}(f)\right) + \sup_{f\in\gF}\Bigg[\sum_{c=1}^K(\frac{u_c^{(\pi')}}{u}-\frac{m_c^{(\pi')}}{m})\Big(\frac{1}{u_c^{(\pi')}}\sum_{i\in\gI_{\mathrm{test},c}^{(\pi')}}\ell_{\gamma,f}(\phi(\rvx_i),c)\Big)
    \Bigg] \nonumber\\
    \leq& \sup_{f\in\gF}\left(\sum_{c=1}^K \Delta_c^{(\pi')}(f)\right) + \sup_{f\in\gF}\Bigg[\sum_{c=1}^K\biggl|\frac{u_c^{(\pi')}}{u}-\frac{m_c^{(\pi')}}{m}\biggl|\Big(\frac{1}{u_c^{(\pi')}}\sum_{i\in\gI_{\mathrm{test},c}^{(\pi')}}\ell_{\gamma,f}(\phi(\rvx_i),c)\Big)\Bigg]\nonumber \\
    \leq& \sup_{f\in\gF}\left(\sum_{c=1}^K \Delta_c^{(\pi')}(f)\right) + \sum_{c=1}^K\biggl|\frac{u_c^{(\pi')}}{u}-\frac{m_c^{(\pi')}}{m}\biggl|
\end{align}

Then, we have:
\begin{align}\label{apdx_eq:sumout}
    \mathbb E_{\pi'}\sup_{f\in\gF}\left[ \Delta^{(\pi')}(f)\right] 
    &\leq \mathbb E_{\pi'}\left[\,\sup_{f\in\gF}\sum_{c=1}^K\left(\Delta_c^{(\pi')}(f)\right) + \sum_{c=1}^K \,\biggl|\frac{u_c^{(\pi')}}{u}-\frac{m_c^{(\pi')}}{m}\biggl|\,\right] \nonumber\\
    &\leq \sum_{c=1}^K\, \mathbb E_{\pi'}\left[\,\sup_{f\in\gF}\left(\Delta_c^{(\pi')}(f)\right)\,\right]+\sum_{c=1}^K\, \mathbb E_{\pi'}\left[\,\biggl|\frac{u_c^{(\pi')}}{u}-\frac{m_c^{(\pi')}}{m}\biggl|\,\right]
\end{align}

We now bound $\sup_{f\in\gF}[\Delta_c^{(\pi')}(f)]$ via optimal transport. Define:
\begin{align*}
M_c(f,\phi) &\coloneq \max_{i, j}\frac{|\rho_f(\phi(\rvx_i),c) - \rho_f(\phi(\rvx_j),c)|}{\|\phi(\rvx_i)-\phi(\rvx_j)\|} \quad \text{for}\ i\neq j\ \text{and}\ i,j\in\gI_{\mathrm{train},c}^{(\pi)}\cup\testi.
\end{align*}
By the definition of $M_c(f,\phi)$, for all $i,j\in\gI_{\mathrm{train},c}^{(\pi)}\cup\testi$,
\begin{equation}\label{apdx_eq:Mc_bound}
|\rho_f(\phi(\rvx_i),c) - \rho_f(\phi(\rvx_j),c)|\leq M_c(f,\phi)\|\phi(\rvx_i)-\phi(\rvx_j)\|\;.
\end{equation}

Using the empirical distribution $\mu_\gI\coloneq\frac{1}{|\gI|}\sum_{i\in\gI}\delta(\rvx_i)$, where $\delta(\cdot)$ denotes the Dirac delta function, the empirical mean can be represented as the expectation with respect to the empirical distribution. This allows us to rewrite $\Delta_c^{(\pi')}(f)$ as:
\begin{align}\label{apdx_eq:expectation_form}
\Delta_c^{(\pi')}(f) 
=& \frac{m_c^{(\pi')}}{m}\Big(
    \mathbb{E}_{\rvx\sim\mu_{\gI_{\mathrm{test},c}^{(\pi')}}}[\ell_{\gamma,f}(\phi(\rvx),c)]
    -\mathbb{E}_{\rvx'\sim\mu_{\gI_{\mathrm{train},c}^{(\pi')}}}[\ell_{\gamma,f}(\phi(\rvx'),c)]
    \Big) \nonumber\\
=& \frac{m_c^{(\pi')}}{m}\Big(
    \int \ell_{\gamma,f}(\phi(\rvx),c)\,d\mu_{\gI_{\mathrm{test},c}^{(\pi')}}
    -\int \ell_{\gamma,f}(\phi(\rvx'),c)\,d\mu_{\gI_{\mathrm{train},c}^{(\pi')}}
    \Big).
\end{align}

For any coupling $T\in\mathcal U(\mu_{\gI_{\mathrm{train},c}^{(\pi')}},\mu_{\gI_{\mathrm{test},c}^{(\pi')}})$, where $\mathcal U(\mu,\nu)$ denotes the set of joint distributions on $\gZ\times\gZ$ with marginals $\mu$ and $\nu$, the difference of expectations can be expressed as an integral with respect to $T$:
\begin{align}\label{apdx_eq:coupling_bound}
&\int \ell_{\gamma,f}(\phi(\rvx),c)\,d\mu_{\gI_{\mathrm{test},c}^{(\pi')}}
-\int \ell_{\gamma,f}(\phi(\rvx'),c)\,d\mu_{\gI_{\mathrm{train},c}^{(\pi')}} \nonumber\\
=& \int \big( \ell_{\gamma,f}(\phi(\rvx),c) - \ell_{\gamma,f}(\phi(\rvx'),c) \big)\, d\,T(\rvx,\rvx') \nonumber\\
\leq& \int \big| \ell_{\gamma,f}(\phi(\rvx),c) - \ell_{\gamma,f}(\phi(\rvx'),c) \big|\, d\,T(\rvx,\rvx') \nonumber\\
=& \int \big| L_\gamma(\rho_f(\phi(\rvx),c)) - L_\gamma(\rho_f(\phi(\rvx'),c)) \big|\, d\,T(\rvx,\rvx') \nonumber\\
\leq& \frac{1}{\gamma} \int \big| \rho_f(\phi(\rvx),c) - \rho_f(\phi(\rvx'),c) \big|\, d\,T(\rvx,\rvx') \nonumber\\
\leq& \frac{M_c(f,\phi)}{\gamma} \int \|\phi(\rvx)-\phi(\rvx')\|\, d\,T(\rvx,\rvx'),
\end{align}
where the second-to-last inequality uses the fact that $L_\gamma$ is $\frac{1}{\gamma}$-Lipschitz, and the last inequality follows from~\cref{apdx_eq:Mc_bound}.

Since~\cref{apdx_eq:coupling_bound} holds for any coupling $T\in\mathcal U(\mu_{\gI_{\mathrm{train},c}^{(\pi')}}, \mu_{\gI_{\mathrm{test},c}^{(\pi')}})$, taking the infimum over all such couplings yields:
\begin{align}\label{apdx_eq:KRDual}
\sup_{f\in\gF}[\Delta_c^{(\pi')}(f)] 
=& \sup_{f\in\gF}\Bigg[
\frac{m_c^{(\pi')}}{m}\Big(
\mathbb{E}_{\rvx\sim\mu_{\gI_{\mathrm{test},c}^{(\pi')}}}\ell_{\gamma,f}(\phi(\rvx),c)
-\mathbb{E}_{\rvx'\sim\mu_{\gI_{\mathrm{train},c}^{(\pi')}}}\ell_{\gamma,f}(\phi(\rvx'),c)
\Big)
\Bigg]\nonumber\\
\leq& \frac{m_c^{(\pi')}}{m}\frac{M_c(f,\phi)}{\gamma}\,\inf_{T\in\mathcal U}\int \|\phi(\rvx)-\phi(\rvx')\|\,d\,T(\rvx,\rvx')\nonumber\\
=&\frac{m_c^{(\pi')}}{m}\frac{M_c(f,\phi)}{\gamma}\gW(\phi_\#\mu_{\gI_{\mathrm{test},c}^{(\pi')}},\phi_\#\mu_{\gI_{\mathrm{train},c}^{(\pi')}})\;.
\end{align}

Putting together \cref{apdx_eq:class-wise,apdx_eq:concentration,apdx_eq:sumout,apdx_eq:KRDual}, with probability at least $1-\delta$, we have:
\begin{align}\label{apdx_eq:puttogether}
    &\sup_{f\in\gF}\left[
    \frac{1}{u}\sum_{i\in\testi}\ell_{\gamma,f}(\phi(\rvx_i),y_i)
    -
    \frac{1}{m}\sum_{j\in\traini}\ell_{\gamma,f}(\phi(\rvx_j),y_j)
    \right]\nonumber\\
    &\leq
    \sum_{c=1}^K \mathbb{E}_{\pi'}\frac{m_c^{(\pi')}}{m}\frac{M_c(f,\phi)}{\gamma}\gW(\phi_\#\mu_{\gI_{\mathrm{test},c}^{(\pi')}},\phi_\#\mu_{\gI_{\mathrm{train},c}^{(\pi')}})
    +\mathbb E_{\pi'}\left[\,\sum_{c=1}^K\biggl|\frac{u_c^{(\pi')}}{u}-\frac{m_c^{(\pi')}}{m}\biggl|\,\right]+\varepsilon_\delta\;,
\end{align}
where $\varepsilon_\delta=\sqrt{\frac{m\,u\,(\frac{1}{m}+\frac{1}{u})^2}{2\,(m+u-\tfrac12)}\left(1-\frac{1}{2\max\{m,u\}}\right)^{-1}\ln\frac1\delta}$.
Plugging \cref{apdx_eq:puttogether} into \cref{apdx_eq:BasicGap}, we obtain the bound stated in the theorem, which completes the proof.
\end{proof}

\begin{definition}[$(m,u)$-permutation symmetry \citep{el2009transductive}]\label{apdx_def:def-symmetry}
Let $[m+u]=\{1,\dots, m, m+1, \dots, m+u\}$ and $[m+1,m+u] = \{m+1,\dots, m+u\}$, and $S_{m+u}$ be the set of permutation functions on $[m+u]$.
Define the block-preserving permutation function.
\[
H_{m,u}\;=\;\bigl\{\sigma\in S_{m+u}\,:\,\sigma([m])=[m]\ \text{and}\ \sigma([m+1,m+u])=[m+1,m+u]\bigr\}.
\]
A function $g:S_{m+u}\to\mathbb R$ is called \emph{$(m,u)$-permutation symmetric} if
\[
g(\pi)\;=\;g(\sigma\circ \pi)\qquad\text{for all }\pi\in S_{m+u}\ \text{and all }\sigma\in H_{m,u}.
\]
Equivalently, $g(\pi)$ depends only on the unordered split $\bigl(\pi([m]),\,\pi([m+1,m+u])\bigr)$, i.e., the training/test partition, and not on the ordering within each block.
\end{definition}

\begin{lemma}[Concentration Inequality on Transductive Setting~\citep{el2009transductive}]\label{apdx_lem:trans-mcdiarmid}
Let $S_{m+u}$  be the set of permutation functions on $[m+u]=\{1,\dots,m+u\}$.
Let $\pi\in S_{m+u}$ and $g:S_{m+u}\to\mathbb R$ be $(m,u)$-permutation symmetric.
For $i\in\{1,\dots,m\}$ and $j\in\{m+1,\dots,m+u\}$, let $\tau_{ij}\in S_{m+u}$ be the transposition of positions $i$ and $j$,
and write $\pi^{(ij)}\coloneq\tau_{ij}\circ\pi$.
If for some $\omega>0$,
\[
\bigl|\,g(\pi)-g(\pi^{(ij)})\,\bigr|\;\le\;\omega
\qquad\text{for all }i\le m<j,
\]
then for every $\varepsilon>0$,
\[
\Pr\!\Big\{\,g(\pi)-\mathbb E_{\pi'}[g(\pi')]\ge \varepsilon\,\Big\}
\;\le\;
\exp\!\left(
-\frac{2\,\varepsilon^2\,\bigl(m+u-\tfrac12\bigr)}{m\,u\,\omega^2\left(1-\tfrac{1}{2\max\{m,u\}}\right)}
\right),
\]
where $\pi'\in \mathrm S_{m+u}$ is an independent to $\pi$.
Equivalently, with probability at least $1-\delta$,
\[
g(\pi)\;\le\;\mathbb E_{\pi'}[g(\pi')]\;+\;\varepsilon_\delta,
\qquad
\varepsilon_\delta
=
\sqrt{\frac{m\,u\,\omega^2}{2\,(m+u-\tfrac12)}\left(1-\tfrac{1}{2\max\{m,u\}}\right)^{-1}\!\ln\frac1\delta}\,.
\]
\end{lemma}

\subsection{Proof of wasserstein distance on SGC}\label{apdx_sec:Pf-w1=}

\setcounter{theorem}{1}
\sgcdepth*

\begin{proof}
Fix $\ell\in\mathbb N$ and denote the depth-$\ell$ SGC embedding by
\begin{equation*}\label{eq:sgc-z-def}
\mathbf z_i^{(\ell)} \coloneq \phi^{(\ell)}(x_i;\mathbf X,\mathbf A)
= (\hat{\mathbf A}^{\ell}\mathbf X)_{i\cdot}\in\mathbb R^{F}.
\end{equation*}

Let $\mathbf u_1\coloneq\tilde{\mathbf D}^{1/2}\mathbf 1$. A direct calculation gives $\hat{\mathbf A}\mathbf u_1=\mathbf u_1$:
\begin{equation*}\label{eq:sgc-u1-eig}
\hat{\mathbf A}\mathbf u_1
=\tilde{\mathbf D}^{-1/2}\tilde{\mathbf A}\tilde{\mathbf D}^{-1/2}\tilde{\mathbf D}^{1/2}\mathbf 1
=\tilde{\mathbf D}^{-1/2}\tilde{\mathbf A}\mathbf 1
=\tilde{\mathbf D}^{-1/2}\tilde{\mathbf D}\mathbf 1
=\tilde{\mathbf D}^{1/2}\mathbf 1
=\mathbf u_1.
\end{equation*}
Define $U\coloneq\mathrm{span}\{\mathbf u_1\}$ and the orthogonal projectors
\begin{equation}\label{eq:sgc-projectors}
\mathbf P_U\coloneq\frac{\mathbf u_1\mathbf u_1^\top}{\|\mathbf u_1\|_2^2},
\qquad
\mathbf P_{U^\perp}\coloneq\mathbf I-\mathbf P_U.
\end{equation}
Decompose the input features as
\begin{equation*}\label{eq:sgc-X-decomp}
\mathbf X=\mathbf X_U+\mathbf X_{U^\perp},
\qquad
\mathbf X_U\coloneq\mathbf P_U\mathbf X,\quad \mathbf X_{U^\perp}\coloneq\mathbf P_{U^\perp}\mathbf X.
\end{equation*}
Since $\hat{\mathbf A}\mathbf u_1=\mathbf u_1$, we have $\hat{\mathbf A}^{\ell}\mathbf X_U=\mathbf X_U$, hence
\begin{equation}\label{eq:sgc-X-split}
\hat{\mathbf A}^{\ell}\mathbf X=\hat{\mathbf A}^{\ell}\mathbf X_U+\hat{\mathbf A}^{\ell}\mathbf X_{U^\perp}
=\mathbf X_U+\hat{\mathbf A}^{\ell}\mathbf X_{U^\perp}.
\end{equation}

We first control the $U^\perp$ component via spectral contraction.
Because the graph is undirected, $\hat{\mathbf A}$ is symmetric and admits an orthonormal eigenbasis
$\{\mathbf v_k\}_{k=1}^N$ with $\hat{\mathbf A}\mathbf v_k=\lambda_k(\hat{\mathbf A})\mathbf v_k$.
Choose $\mathbf v_1=\mathbf u_1/\|\mathbf u_1\|_2$ and define
$\rho_\perp(\hat{\mathbf A})\coloneq\max_{k\ge2}|\lambda_k(\hat{\mathbf A})|$.
For any $\mathbf h\in U^\perp$, we can write $\mathbf h=\sum_{k=2}^N \alpha_k\mathbf v_k$, and thus
\begin{equation}\label{eq:sgc-spectral-contraction}
\|\hat{\mathbf A}\mathbf h\|_2^2
=
\Big\|\sum_{k=2}^N \alpha_k\lambda_k(\hat{\mathbf A})\mathbf v_k\Big\|_2^2
=
\sum_{k=2}^N \alpha_k^2\lambda_k(\hat{\mathbf A})^2
\le
\rho_\perp(\hat{\mathbf A})^2\sum_{k=2}^N \alpha_k^2
=
\rho_\perp(\hat{\mathbf A})^2\|\mathbf h\|_2^2,
\end{equation}
so $\|\hat{\mathbf A}\mathbf h\|_2\le \rho_\perp(\hat{\mathbf A})\|\mathbf h\|_2$ for all $\mathbf h\in U^\perp$.
Since $\hat{\mathbf A}$ is symmetric and $U$ is an eigenspace of $\hat{\mathbf A}$, the orthogonal complement $U^\perp$ is also $\hat{\mathbf A}$-invariant, i.e., $\hat{\mathbf A}\mathbf h\in U^\perp$ for all $\mathbf h\in U^\perp$.
Iterating yields $\|\hat{\mathbf A}^{\ell}\mathbf h\|_2\le \rho_\perp(\hat{\mathbf A})^{\ell}\|\mathbf h\|_2$,
and applying this column-wise gives
\begin{equation}
\label{eq:sgc-Uperp}
\|\hat{\mathbf A}^{\ell}\mathbf X_{U^\perp}\|_F
\le
\rho_\perp(\hat{\mathbf A})^{\ell}\|\mathbf X_{U^\perp}\|_F
=
\rho_\perp(\hat{\mathbf A})^{\ell}\|\mathbf P_{U^\perp}\mathbf X\|_F.
\end{equation}

Next, since $\mathbf P_U$ has rank one, there exists a vector $\mathbf a\in\mathbb R^{F}$ such that
\begin{equation}\label{eq:sgc-XU-a}
\mathbf X_U=\mathbf u_1\mathbf a^\top,
\qquad
\mathbf a=\frac{1}{\|\mathbf u_1\|_2^2}\mathbf X^\top\mathbf u_1.
\end{equation}
Then, using $\|\mathbf X^\top\mathbf u_1\|_2\le\|\mathbf X\|_2\|\mathbf u_1\|_2\le\|\mathbf X\|_F\|\mathbf u_1\|_2$,
\begin{equation}\label{eq:sgc-a-bound}
\|\mathbf a\|_2
=
\frac{\|\mathbf X^\top\mathbf u_1\|_2}{\|\mathbf u_1\|_2^2}
\le
\frac{\|\mathbf X\|_F\|\mathbf u_1\|_2}{\|\mathbf u_1\|_2^2}
=
\frac{\|\mathbf X\|_F}{\|\mathbf u_1\|_2}.
\end{equation}

Now note that $(\mathbf X_U)_{i\cdot}=\mathbf u_1(i)\mathbf a^\top$.
Recalling that $d(x_i)=\sqrt{\tilde d_i}=\mathbf u_1(i)$, we obtain for any $i,j$,
\begin{align}
\|\mathbf z_i^{(\ell)}-\mathbf z_j^{(\ell)}\|_2
&=
\big\|(\mathbf X_U)_{i\cdot}-(\mathbf X_U)_{j\cdot}
+\big(\hat{\mathbf A}^{\ell}\mathbf X_{U^\perp}\big)_{i\cdot}
-\big(\hat{\mathbf A}^{\ell}\mathbf X_{U^\perp}\big)_{j\cdot}\big\|_2 \notag \\
&\le
|d(x_i)-d(x_j)|\,\|\mathbf a\|_2
+\big\|\big(\hat{\mathbf A}^{\ell}\mathbf X_{U^\perp}\big)_{i\cdot}\big\|_2
+\big\|\big(\hat{\mathbf A}^{\ell}\mathbf X_{U^\perp}\big)_{j\cdot}\big\|_2 \notag \\
&\le
\frac{\|\mathbf X\|_F}{\|\mathbf u_1\|_2}\,|d(x_i)-d(x_j)|
+2\,\|\hat{\mathbf A}^{\ell}\mathbf X_{U^\perp}\|_F \notag \\
&\le
\frac{\|\mathbf X\|_F}{\|\mathbf u_1\|_2}\,|d(x_i)-d(x_j)|
+2\,\|\mathbf P_{U^\perp}\mathbf X\|_F\,\rho_\perp(\hat{\mathbf A})^{\ell},
\label{eq:sgc-pointwise-z}
\end{align}

Let $\mathbf \Gamma$ be any coupling between $\mu_S$ and $\mu_T$ (i.e., $\Gamma_{ij}\ge 0$ and
$\sum_{j\in T}\Gamma_{ij}=\frac{1}{|S|}$ for $i\in S$, $\sum_{i\in S}\Gamma_{ij}=\frac{1}{|T|}$ for $j\in T$). By definition of the empirical $1$-Wasserstein distance,
\begin{equation}\label{eq:sgc-W1-def}
\mathcal W\!\big(\phi^{(\ell)}_\#\mu_{\mathcal S},\,\phi^{(\ell)}_\#\mu_{\mathcal T}\big)
=
\inf_{\mathbf \Gamma}\sum_{i\in \mathcal S}\sum_{j\in \mathcal T}
\Gamma_{ij}\,\|\mathbf z_i^{(\ell)}-\mathbf z_j^{(\ell)}\|_2.
\end{equation}
Applying the pointwise estimate inside the transport cost and using $\sum_{i,j}\Gamma_{ij}=1$ yields
\begin{align}
\sum_{i,j}\Gamma_{ij}\,\|\mathbf z_i^{(\ell)}-\mathbf z_j^{(\ell)}\|_2
&\le
\frac{\|\mathbf X\|_F}{\|\mathbf u_1\|_2}\sum_{i,j}\Gamma_{ij}|d(x_i)-d(x_j)|
+2\,\|\mathbf P_{U^\perp}\mathbf X\|_F\,\rho_\perp(\hat{\mathbf A})^{\ell}.
\label{eq:sgc-cost-bound}
\end{align}
Taking the infimum over $\mathbf \Gamma$ gives
\begin{equation}\label{eq:sgc-W1-final}
\mathcal W\!\big(\phi^{(\ell)}_\#\mu_{\mathcal S},\,\phi^{(\ell)}_\#\mu_{\mathcal T}\big)
\le
\Big(\frac{\|\mathbf X\|_F}{\|\mathbf u_1\|_2}\Big)\,
\mathcal W\!\big(d_\#\mu_{\mathcal S},\,d_\#\mu_{\mathcal T}\big)
+
\Big(2\,\|\mathbf P_{U^\perp}\mathbf X\|_F\Big)\rho_\perp(\hat{\mathbf A})^{\ell},
\end{equation}
which proves \eqref{eq:w1-sgc}. In particular, we take
$C_{1}\coloneq \|\mathbf X\|_F/\|\mathbf u_1\|_2$ and
$C_{2}\coloneq 2\,\|\mathbf P_{U^\perp}\mathbf X\|_F$.
\end{proof}

\subsection{Proof of the wasserstein distance on GCN}\label{apdx_sec:Pf-W1GCN}

\setcounter{theorem}{2}
\gcndepth*
\begin{proof}
Recall from the proof of \cref{prop:w1-sgc} that $\mathbf u_1\coloneq\tilde{\mathbf D}^{1/2}\mathbf 1$ satisfies
$\hat{\mathbf A}\mathbf u_1=\mathbf u_1$, and let $U\coloneq\mathrm{span}\{\mathbf u_1\}$ with orthogonal projectors
\begin{equation*}\label{eq:gcn-projectors}
\mathbf P_U=\frac{\mathbf u_1\mathbf u_1^\top}{\|\mathbf u_1\|_2^2},
\qquad
\mathbf P_{U^\perp}=\mathbf I-\mathbf P_U.
\end{equation*}
Write the decomposition
\begin{equation*}\label{eq:gcn-X-decomp}
\mathbf X_U^{(\ell)}\coloneq\mathbf P_U \mathbf X^{(\ell)},\qquad
\mathbf X_{U^\perp}^{(\ell)}\coloneq\mathbf P_{U^\perp} \mathbf X^{(\ell)},\qquad
\mathbf X^{(\ell)}=\mathbf X_U^{(\ell)}+\mathbf X_{U^\perp}^{(\ell)}.
\end{equation*}

We first control the growth of $\|\mathbf X^{(\ell)}\|_F$.
Since $\sigma=\mathrm{ReLU}$ is $1$-Lipschitz elementwise and $|\mathrm{ReLU}(a)|\le |a|$,
\begin{equation}\label{eq:gcn-growth-step}
\|\mathbf X^{(\ell+1)}\|_F
=\|\sigma(\hat{\mathbf A}\,\mathbf X^{(\ell)}\mathbf W^{(\ell)})\|_F
\le \|\hat{\mathbf A}\,\mathbf X^{(\ell)}\mathbf W^{(\ell)}\|_F
\le \|\hat{\mathbf A}\|_2 \|\mathbf X^{(\ell)}\|_F \|\mathbf W^{(\ell)}\|_2
\le \beta \|\mathbf X^{(\ell)}\|_F,
\end{equation}
where $\|\hat{\mathbf A}\|_2=1$. Iterating yields
\begin{equation}
\label{eq:gcn-growth}
\|\mathbf X^{(\ell)}\|_F \le \beta^\ell \|\mathbf X\|_F.
\end{equation}

Next we note that one GCN layer maps $U$ into itself.
If $\mathbf Y\in U$, then $\mathbf Y=\mathbf u_1 \mathbf a^\top$ for some vector $\mathbf a$ and
\begin{equation}\label{eq:gcn-U-map}
\hat{\mathbf A}\,\mathbf Y\, \mathbf W^{(\ell)}
= \hat{\mathbf A}\,\mathbf u_1 \mathbf a^\top \mathbf W^{(\ell)}
= \mathbf u_1(\mathbf a^\top \mathbf W^{(\ell)}).
\end{equation}
Since $\mathbf u_1(i)\ge 0$ for all $i$ and ReLU is positively homogeneous on $\mathbb R_+$,
$\mathrm{ReLU}(\mathbf u_1(i)\mathbf b)=\mathbf u_1(i)\mathrm{ReLU}(\mathbf b)$, hence
$\sigma(\hat{\mathbf A}\,\mathbf Y\,\mathbf W^{(\ell)})\in U$ and therefore
\begin{equation}
\label{eq:gcn-U-inv}
\mathbf P_{U^\perp}\sigma(\hat{\mathbf A}\,\mathbf X_U^{(\ell)}\mathbf W^{(\ell)})=0.
\end{equation}

We now bound the $U^\perp$ component.
Using \eqref{eq:gcn-U-inv} and $\|\mathbf P_{U^\perp}\|_2=1$,
\begin{align}\label{eq:gcn-Uperp-step1}
\|\mathbf X_{U^\perp}^{(\ell+1)}\|_F
&=\Big\|\mathbf P_{U^\perp}\big(\sigma(\hat{\mathbf A}\,\mathbf X^{(\ell)}\mathbf W^{(\ell)})
-\sigma(\hat{\mathbf A}\,\mathbf X_U^{(\ell)}\mathbf W^{(\ell)})\big)\Big\|_F \notag \\
&\le \|\sigma(\hat{\mathbf A}\,\mathbf X^{(\ell)}\mathbf W^{(\ell)})
-\sigma(\hat{\mathbf A}\,\mathbf X_U^{(\ell)}\mathbf W^{(\ell)})\|_F.
\end{align}
By $1$-Lipschitzness of $\sigma$,
\begin{equation}\label{eq:gcn-Uperp-step2}
\|\mathbf X_{U^\perp}^{(\ell+1)}\|_F
\le \|\hat{\mathbf A}(\mathbf X^{(\ell)}-\mathbf X_U^{(\ell)})\mathbf W^{(\ell)}\|_F
= \|\hat{\mathbf A}\,\mathbf X_{U^\perp}^{(\ell)}\mathbf W^{(\ell)}\|_F
\le \|\mathbf W^{(\ell)}\|_2\ \|\hat{\mathbf A}\,\mathbf X_{U^\perp}^{(\ell)}\|_F.
\end{equation}

To control $\|\hat{\mathbf A}\,\mathbf X_{U^\perp}^{(\ell)}\|_F$, use that $\hat{\mathbf A}$ is symmetric, hence admits an orthonormal eigenbasis
$\{\mathbf v_k\}_{k=1}^N$ with $\hat{\mathbf A}\mathbf v_k = \lambda_k(\hat{\mathbf A})\mathbf v_k$ and $\mathbf v_1=\mathbf u_1/\|\mathbf u_1\|_2$.
For any $\mathbf h\in U^\perp$ we have $\mathbf h=\sum_{k=2}^N \alpha_k \mathbf v_k$ and therefore
\begin{equation}\label{eq:gcn-spectral-contraction}
\|\hat{\mathbf A}\mathbf h\|_2^2
=\Big\|\sum_{k=2}^N \alpha_k \lambda_k(\hat{\mathbf A})\mathbf v_k\Big\|_2^2
=\sum_{k=2}^N \alpha_k^2 \lambda_k(\hat{\mathbf A})^2
\le \rho_\perp(\hat{\mathbf A})^2 \sum_{k=2}^N \alpha_k^2
= \rho_\perp(\hat{\mathbf A})^2 \|\mathbf h\|_2^2,
\end{equation}
which gives $\|\hat{\mathbf A}\mathbf h\|_2 \le \rho_\perp(\hat{\mathbf A})\|\mathbf h\|_2$ for all $\mathbf h\in U^\perp$.
Applying this columnwise to $\mathbf X_{U^\perp}^{(\ell)}$ yields
\begin{equation}\label{eq:gcn-AUperp-F}
\|\hat{\mathbf A}\,\mathbf X_{U^\perp}^{(\ell)}\|_F \le \rho_\perp(\hat{\mathbf A})\|\mathbf X_{U^\perp}^{(\ell)}\|_F,
\end{equation}
and thus
\begin{equation}\label{eq:gcn-Uperp-rec}
\|\mathbf X_{U^\perp}^{(\ell+1)}\|_F \le \beta\rho_\perp(\hat{\mathbf A})\,\|\mathbf X_{U^\perp}^{(\ell)}\|_F.
\end{equation}
Iterating gives
\begin{equation}
\label{eq:gcn-Uperp}
\|\mathbf X_{U^\perp}^{(\ell)}\|_F \le (\beta\rho_\perp(\hat{\mathbf A}))^\ell \|\mathbf P_{U^\perp}\mathbf X\|_F.
\end{equation}

Since $\mathbf P_U$ has rank one, we can write $\mathbf X_U^{(\ell)}=\mathbf u_1 \mathbf a^{(\ell)\top}$ with
\begin{equation}\label{eq:gcn-a-def}
\mathbf a^{(\ell)}=\frac{1}{\|\mathbf u_1\|_2^2}\,\mathbf X^{(\ell)\top}\mathbf u_1.
\end{equation}
Then
\begin{equation}\label{eq:gcn-a-bound}
\|\mathbf a^{(\ell)}\|_2
=\frac{\|\mathbf X^{(\ell)\top}\mathbf u_1\|_2}{\|\mathbf u_1\|_2^2}
\le \frac{\|\mathbf X^{(\ell)}\|_F\|\mathbf u_1\|_2}{\|\mathbf u_1\|_2^2}
= \frac{\|\mathbf X^{(\ell)}\|_F}{\|\mathbf u_1\|_2}
\le \beta^\ell \frac{\|\mathbf X\|_F}{\|\mathbf u_1\|_2},
\end{equation}
where we used \eqref{eq:gcn-growth}.

Now let $\mathbf z_i^{(\ell)}\coloneq(\mathbf X^{(\ell)})_{i\cdot}
=\phi_{\mathrm{GCN}}^{(\ell)}(x_i;\mathbf X,\mathbf A)$.
Because $(\mathbf X_U^{(\ell)})_{i\cdot}=\mathbf u_1(i)\,\mathbf a^{(\ell)\top}
=d(x_i)\,\mathbf a^{(\ell)\top}$, for any $i,j$,
\begin{align}\label{eq:gcn-pointwise-pre}
\|\mathbf z_i^{(\ell)}-\mathbf z_j^{(\ell)}\|_2
&\le |d(x_i)-d(x_j)|\,\|\mathbf a^{(\ell)}\|_2
+ \|(\mathbf X_{U^\perp}^{(\ell)})_{i\cdot}\|_2 + \|(\mathbf X_{U^\perp}^{(\ell)})_{j\cdot}\|_2 \notag \\
&\le |d(x_i)-d(x_j)|\,\|\mathbf a^{(\ell)}\|_2 + 2\,\|\mathbf X_{U^\perp}^{(\ell)}\|_F.
\end{align}
Combining with the bounds above and \eqref{eq:gcn-Uperp} yields the pointwise estimate
\begin{equation}\label{eq:gcn-pointwise}
\|\mathbf z_i^{(\ell)}-\mathbf z_j^{(\ell)}\|_2
\le
\Big(\beta^\ell\frac{\|\mathbf X\|_F}{\|\mathbf u_1\|_2}\Big)\,|d(x_i)-d(x_j)|
+
\Big(2\,\|\mathbf P_{U^\perp}\mathbf X\|_F\Big)\,(\beta\rho_\perp(\hat{\mathbf A}))^\ell.
\end{equation}

Let $\mathbf \Gamma$ be any coupling between $\mu_S$ and $\mu_T$.
By definition of the empirical $1$-Wasserstein distance,
\begin{equation}\label{eq:gcn-W1-def}
W\!\Big( \phi^{(\ell)}_\#\mu_S,\ \phi^{(\ell)}_\#\mu_T \Big)
=
\inf_{\mathbf \Gamma}\sum_{i\in S}\sum_{j\in T}\Gamma_{ij}\,\|\mathbf z_i^{(\ell)}-\mathbf z_j^{(\ell)}\|_2.
\end{equation}
Applying the pointwise estimate inside the transport cost and using $\sum_{i,j}\Gamma_{ij}=1$ gives
\begin{align}
\sum_{i,j}\Gamma_{ij}\,\|\mathbf z_i^{(\ell)}-\mathbf z_j^{(\ell)}\|_2
&\le
\Big(\beta^\ell\frac{\|\mathbf X\|_F}{\|\mathbf u_1\|_2}\Big)\sum_{i,j}\Gamma_{ij}|d(x_i)-d(x_j)| \notag \\
&\quad
+\Big(2\,\|\mathbf P_{U^\perp}\mathbf X\|_F\Big)(\beta\rho_\perp(\hat{\mathbf A}))^\ell \sum_{i,j}\Gamma_{ij} \notag \\
&=
\Big(\beta^\ell\frac{\|\mathbf X\|_F}{\|\mathbf u_1\|_2}\Big)\sum_{i,j}\Gamma_{ij}|d(x_i)-d(x_j)| \notag \\
&\quad
+\Big(2\,\|\mathbf P_{U^\perp}\mathbf X\|_F\Big)(\beta\rho_\perp(\hat{\mathbf A}))^\ell .
\label{eq:gcn-cost-bound}
\end{align}
Taking the infimum over $\mathbf \Gamma$ yields
\begin{equation}\label{eq:gcn-W1-final}
W\!\Big( \phi^{(\ell)}_\#\mu_S,\ \phi^{(\ell)}_\#\mu_T \Big)
\le
\left[\Big(\frac{\|\mathbf X\|_F}{\|\mathbf u_1\|_2}\Big)\,
W\!\big(d_\#\mu_S,\ d_\#\mu_T\big)
+
\Big(2\,\|\mathbf P_{U^\perp}\mathbf X\|_F\Big)\,(\rho_\perp(\hat{\mathbf A}))^\ell\right]\beta^\ell.
\end{equation}
Thus the theorem holds with $C_1\coloneq\|\mathbf X\|_F/\|\mathbf u_1\|_2$ and
$C_2\coloneq 2\,\|\mathbf P_{U^\perp}\mathbf X\|_F$.
\end{proof}

\subsection{Proof of lower bound on \texorpdfstring{$M_c$}{Mc}}
\Mclower*
\begin{proof}
For every $i\in S$, we have $\rho_f(\phi(\rvx_i),c)\geq\gamma$.
For every $j\in T$, the predicted class $c'\neq c$ attains the
maximum score over all classes. Hence,
\begin{equation}
\label{apdx_eq:Mc_test_margin}
\rho_f(\phi(\rvx_j),c)
= f_c(\phi(\rvx_j))-f_{c'}(\phi(\rvx_j))\leq0\;.
\end{equation}
Thus, for every $i\in S$ and $j\in T$,
\begin{equation*}
\gamma
\leq \rho_f(\phi(\rvx_i),c)-\rho_f(\phi(\rvx_j),c)\;.
\end{equation*}
Since $S\cup T\subseteq\gI_{\mathrm{train},c}^{(\pi)}\cup\testi$,
the definition of $M_c(f,\phi)$ then gives
\begin{equation*}
\gamma
\leq M_c(f,\phi)\|\phi(\rvx_i)-\phi(\rvx_j)\|\;.
\end{equation*}

Let $\Gamma$ be any coupling of $\phi_\#\mu_S$ and $\phi_\#\mu_T$.
For a pair $(z,z')$ in its support, $z=\phi(\rvx_i)$ for some $i\in S$
and $z'=\phi(\rvx_j)$ for some $j\in T$.
Integrating the above inequality and using that $\Gamma$ has total mass
one, we obtain
\begin{equation}
\label{apdx_eq:Mc_transport}
\gamma=\int\gamma\,d\Gamma(z,z')
\leq M_c(f,\phi)\int\|z-z'\|\,d\Gamma(z,z')\;.
\end{equation}
Since this holds for every coupling, taking the infimum over $\Gamma$
yields
\begin{equation}
\label{apdx_eq:Mc_transport_inf}
\gamma\leq M_c(f,\phi)\mathcal W(\phi_\#\mu_S,\phi_\#\mu_T)\;.
\end{equation}
The margin conditions give disjoint finite supports, so the Wasserstein
distance is positive. Dividing by this distance completes the proof.
\end{proof}

\subsection{Proof of global bound}\label{apdx_sec:proof of global OT}

\globalbound*

\begin{proof}
The goal of the theorem is 
bounding the generalization gap between the zero-one loss of test set $R_u(f\circ\phi;\pi)$ and the $\gamma$-margin loss of train set $R_{m,\gamma}(f\circ\phi;\pi)$ of a model $f\circ\phi$ and permutation $\pi$, where:
\begin{equation}
\label{eq:apdx:test_loss}
    R_u(f\circ\phi;\pi)\coloneq \frac{1}{u}\sum_{i\in\testi}
    \mathbbm{1}_{\rho_f(\phi(\mathbf{x}_i),y_i)\leq0}\;,
\end{equation}
and
\begin{equation}
\label{eq:apdx:margin_loss}
    R_{m,\gamma}(f\circ\phi;\pi)\coloneqq \frac{1}{m}\sum_{i\in\traini}
    \mathbbm{1}_{\rho_f(\phi(\mathbf{x}_i),y_i)\leq\gamma}\;,
\end{equation}
with $\gamma>0$.
We introduce a margin loss $L_\gamma$ with $\gamma>0$ defined by $L_\gamma(u)\coloneq\mathbbm{1}_{u\leq0}+(1-\frac{u}{\gamma})\mathbbm{1}_{0< u\leq\gamma}$. 
We first derive the upper bound on the difference of the margin loss on the test and train sets, i.e.:
\begin{equation}\label{apdx:eq:marginloss_diff}
    \frac{1}{u}\sum_{i\in\gI_{\mathrm{test}}^{(\pi)}}L_\gamma(\rho_f(\phi(\rvx_i),y_i))
    -\frac{1}{m}\sum_{{j}\in\gI_{\mathrm{train}}^{(\pi)}}L_\gamma(\rho_f(\phi(\rvx_j),y_j))\;.
\end{equation}
By using the empirical distribution $\mu_{\gI}\coloneq \frac{1}{|\gI|}\sum_{i\in\gI}\delta(\rvx_i)$, 
where $\delta(\cdot)$ denotes the Dirac delta function, the empirical mean can be represented as the expectation with respect to the empirical distribution.
This allows us to rewrite \cref{apdx:eq:marginloss_diff} as:

\begin{align}\label{apdx_eq:diff_of_margin}
    \mathbb{E}_{\rvx\sim\mu_{\mathcal{I}^{(\pi)}_{\mathrm{test}}}}[L_\gamma(\rho_f(\phi(\rvx),y_\rvx))] 
    &- \mathbb{E}_{\rvx'\sim\mu_{\mathcal{I}^{(\pi)}_{\mathrm{train}}}}[L_\gamma(\rho_f(\phi(\rvx'),y_{\rvx'}))] \nonumber\\
    =& \
    \int L_\gamma(\rho_f(\phi(\rvx),y_\rvx))d\mu_{\mathcal{I}^{(\pi)}_{\mathrm{test}}}
    - \int L_\gamma(\rho_f(\phi(\rvx'),y_{\rvx'}))d\mu_{\mathcal{I}^{(\pi)}_{\mathrm{train}}} \nonumber \\
    =&\int \Big(L_\gamma\big(\rho_f(\phi(\rvx), y_{\rvx})\big)
    - L_\gamma \big(\rho_f(\phi(\rvx'),y_{\rvx'})\big)\Big)d\,{T}(\rvx, \rvx') \nonumber \\
    \leq &\int |L_\gamma\big(\rho_f(\phi(\rvx), y_{\rvx})\big)
    - L_\gamma \big(\rho_f(\phi(\rvx'),y_{\rvx'})\big)|d\,{T}(\rvx, \rvx') \nonumber\\
    \leq& \,\frac{1}{\gamma}\,\int |\rho_f(\phi(\rvx), y_\rvx) - \rho_f(\phi(\rvx'),y_{\rvx'}) | d\,{T}(\rvx,\rvx')
\end{align}

where $y_\rvx$ and $y_{\rvx'}$ denote the labels of $\rvx$ and $\rvx'$, respectively. The last inequality in~\cref{apdx_eq:diff_of_margin} is based on the fact that $L_\gamma$ is $\frac{1}{\gamma}$-Lipschitz.

Now, we will bound $\|\rho_f(\phi(\rvx), y_{\rvx}) - \rho_f(\phi(\rvx'), y_{\rvx'}) \|$ from our samples using an advantage of transductive settings. Since we cannot access the label of a test sample ($y_{\rvx})$, we define $M$ with all samples $\rvx, \rvx' \in \{\rvx_i\}^{m+u}_{i=1}$ and $y, y' \in \mathcal Y$.
Define: 
\begin{equation}
    M(f,\phi)\coloneq \max_{\rvx, \rvx', y'}\frac{|\rho_f(\phi(\rvx), y_{\rvx}) - \rho_f(\phi(\rvx'),y')|}{\|\phi(\rvx)-\phi(\rvx')\|} \qquad \text{for}\ \rvx\neq \rvx'
\end{equation}
Then,  
\begin{equation} \label{apdx_eq:boudnM}
|\rho_f(\phi(\rvx),y_\rvx) - \rho_f(\phi(\rvx'), y_{\rvx'}) | \leq M(f,\phi) \,\|\phi(\rvx)-\phi(\rvx')\|
\end{equation}

From the~\cref{apdx_eq:diff_of_margin} and~\cref{apdx_eq:boudnM}, we have:

\begin{align}\label{apdx_eq:appliedbd}
    \mathbb{E}_{\rvx\sim\mu_{\mathcal{I}^{(\pi)}_{\mathrm{test}}}}[L_\gamma(\rho_f(\phi(\rvx),y))]&
    - \mathbb{E}_{\rvx'\sim\mu_{\mathcal{I}^{(\pi)}_{\mathrm{train}}}}[L_\gamma(\rho_f(\phi(\rvx'),y'))] \nonumber\\
    &\leq \int \Big(L_\gamma\big(\rho_f(\phi(\rvx),y)\big)
    - L_\gamma \big(\rho_f(\phi(\rvx'),y')\big)\Big)d\,{T}(\rvx,\rvx') \nonumber \\
    &\leq\, \frac{1}{\gamma}\,\int |\rho_f(\phi(\rvx),y) - \rho_f(\phi(\rvx'),y') | d\,{T}(\rvx,\rvx') \nonumber\\
    &\leq\,\frac{M(f,\phi)}{\gamma}\,\int \|\phi(\rvx)-\phi(\rvx')\|d\,{T}(\rvx,\rvx')
\end{align}

Since \cref{apdx_eq:appliedbd} holds for any couplings ${T}\in {\mathcal U(\mu_{\mathcal{I}^{(\pi)}_{\mathrm{train}}}, \mu_{\mathcal{I}^{(\pi)}_{\mathrm{test}}})}$, we have:
\begin{align}\label{apdx_eq:ineqWD}
    &\mathbb{E}_{\rvx\sim\mu_{\mathcal{I}^{(\pi)}_{\mathrm{test}}}}[L_\gamma(\rho_f(\phi(\rvx),y))]
    - \mathbb{E}_{\rvx'\sim\mu_{\mathcal{I}^{(\pi)}_{\mathrm{train}}}}[L_\gamma(\rho_f(\phi(\rvx'),y'))] \nonumber\\
    &\le \inf_{{T}\in {\mathcal U}}\frac{M(f,\phi)}{\gamma}\,\int_{T}\|\phi(\rvx)-\phi(\rvx')\|d\,{T}(\rvx,\rvx') \nonumber \\
    &=\frac{M(f,\phi)}{\gamma}\, \mathcal W(\phi_{\#\mu_{\mathcal I_\mathrm{test}}},\phi_{\#\mu_{\mathcal I_\mathrm{train}}})
\end{align}

Putting together \cref{apdx:eq:marginloss_diff} and \cref{apdx_eq:ineqWD}, then
\begin{equation*}
    \frac{1}{u}\sum_{i\in\gI_{\mathrm{test}}^{(\pi)}}L_\gamma(\rho_f(\phi(\rvx_i),y_i))
    \leq\frac{1}{m}\sum_{i\in\gI_{\mathrm{train}}^{(\pi)}}L_\gamma(\rho_f(\phi(\rvx_j),y_j))
    + \frac{M(f,\phi)}{\gamma} \mathcal{W}\big(\phi_\#\mu_{\gI_{\rm{test}}^{(\pi)}},\,\phi_\#\mu_{\gI_{\rm{train}}^{(\pi)}}\big)\;.
\end{equation*}
Since $\mathbbm{1}_{u\leq0}\leq L_\gamma(u)\leq\mathbbm{1}_{u\leq\gamma}$ for all $u$, we have:
\begin{align*}
    \testr\leq&\frac{1}{u}\sum_{i\in\gI_{\mathrm{test}}^{(\pi)}}L_\gamma(\rho_f(\phi(\rvx_i),y_i))\\
    \leq&\frac{1}{m}\sum_{i\in\gI_{\mathrm{train}}^{(\pi)}}L_\gamma(\rho_f(\phi(\rvx_j),y_j))
    + \frac{M(f,\phi)}{\gamma} \mathcal{W}\big(\phi_\#\mu_{\gI_{\rm{test}}^{(\pi)}},\,\phi_\#\mu_{\gI_{\rm{train}}^{(\pi)}}\big)\;\\
    \leq&\trainr
    + \frac{M(f,\phi)}{\gamma} \mathcal{W}\big(\phi_\#\mu_{\gI_{\rm{test}}^{(\pi)}},\,\phi_\#\mu_{\gI_{\rm{train}}^{(\pi)}}\big)\;,
\end{align*}
which completes the proof.
\end{proof}

\newpage


\section{Summary of the bounds used in this work}
\label{apdx_sec:approximate}
This section summarizes the bounds used in this work and clarifies
the purpose of each. We distinguish the theoretical bounds
(the bound in \cref{thm:classwise-ot}, evaluated as \emph{OT-oracle},
and the global bound), the practical estimate (\emph{OT bound}),
and the surrogate regularizer $\gR$ used during training.
\begin{enumerate}[label=(\arabic*), leftmargin=*, itemsep=2pt, topsep=2pt]
    \item The bound in \cref{thm:classwise-ot} and the global bound
    in \cref{thm:global-ot} provide theoretical guarantees on the
    generalization gap. The former compares training and test feature
    distributions within each class, while the latter compares the
    full training and test feature distributions without class-wise
    decomposition and can be evaluated without access to test labels.

    \item \emph{OT-oracle} evaluates the bound in \cref{thm:classwise-ot} with access to both training and test labels.

    \item \emph{OT bound} estimates the bound in \cref{thm:classwise-ot} without access to test labels for practical use.

    \item The surrogate regularizer $\gR$ motivated by class-wise bound in \cref{subsec:expo_reg_gnn}
    is a differentiable surrogate added
    to the training loss. Its value is not itself guaranteed to
    upper-bound the generalization gap.
\end{enumerate}
The remainder of this section describes how OT-oracle and OT bound
are computed.



\paragraph{Setup}
The original split is given by $\pi$, which divides nodes into training and test sets, $\gI_{\rm{train}}^{(\pi)}$ and $\gI_{\rm{test}}^{(\pi)}$. The expectations in~\cref{eq:main-OT-bound} are then taken over an additional permutation $\pi'$, sampled independently from the uniform distribution over all $(m+u)!$ permutations of the nodes.
We consider two estimators: \textit{OT-oracle}, which assumes access to all node labels, and \textit{OT bound}, which uses only the training labels under $\pi$.

\paragraph{OT-oracle}
Given a sampled $\pi'$, every quantity in~\cref{eq:main-OT-bound}, including $u_c^{(\pi')}$, $m_c^{(\pi')}$, and $\mathcal{W}\big(\phi_\#\mu_{\gI_{\rm{train},c}^{(\pi')}},\,\phi_\#\mu_{\gI_{\rm{test},c}^{(\pi')}}\big)$, can be computed directly using the true labels.
\paragraph{OT bound}
Without access to the test labels under $\pi$, OT bound restricts all class-wise computations to nodes whose labels are known. Concretely, given a sampled $\pi'$, we proceed as follows:
\begin{enumerate}
  \item The permutation $\pi'$ divides all nodes into new train/test sets, $\gI_{\rm{train}}^{(\pi')}$ and $\gI_{\rm{test}}^{(\pi')}$.
  \item Any node belonging to the original test set, $\gI_{\rm{test}}^{(\pi)}$, is removed from these sets to avoid test label leakage under the given split $\pi$.
  \item All remaining nodes have known labels, allowing class-wise quantities to be computed directly. Letting $\gI_{c}$ denote all nodes of class $c$, the resulting class-$c$ index sets are
\begin{equation*}
\gI_{\rm{train},c}^{(\pi';\pi)} = (\gI_{\rm{train}}^{(\pi')} \setminus \gI_{\rm{test}}^{(\pi)}) \cap \gI_{c}, \quad
\gI_{\rm{test},c}^{(\pi';\pi)} = (\gI_{\rm{test}}^{(\pi')} \setminus \gI_{\rm{test}}^{(\pi)}) \cap \gI_{c}.
\end{equation*}
\end{enumerate}
The class-wise terms are then approximated by:
\begin{equation*}
u_c^{(\pi')}\approx|\gI_{\rm{test},c}^{(\pi';\pi)}|, \quad
m_c^{(\pi')}\approx|\gI_{\rm{train},c}^{(\pi';\pi)}|, \quad
\mathcal{W}\big(\phi_\#\mu_{\gI_{\rm{train},c}^{(\pi')}},\,\phi_\#\mu_{\gI_{\rm{test},c}^{(\pi')}}\big) \approx \mathcal{W}\big(\phi_\#\mu_{\gI_{\rm{train},c}^{(\pi';\pi)}},\,\phi_\#\mu_{\gI_{\rm{test},c}^{(\pi';\pi)}}\big).
\end{equation*}

\paragraph{Sampling permutations}
Both estimators approximate the expectation over $\pi'$ by averaging over $T$ sampled permutations: $\E_{\pi'}[\cdot]\approx \frac{1}{T}\sum_{t=1}^T (\cdot)$. We use $T=4$, as larger sample sizes ($16$ or $64$) do not yield consistent gains while smaller sizes ($1$) sometimes produce noticeably lower correlations. The corresponding results for SGC and GCN on Cora and Amazon-Ratings are reported in~\cref{tab:sample_efficiency}.
\begin{table}[h]
    \caption{Correlation values (with computation time in seconds) for estimating the expectation in~\cref{thm:classwise-ot} using different numbers of sampled permutations (1, 4, 16, 64). Results are reported for SGC and GCN on the Cora and Amazon-Ratings datasets under both the OT-oracle and OT bound settings.}
    \centering
    \resizebox{0.7\linewidth}{!}{
    \begin{tabular}{ccccccc} 
         \toprule 
         & & & \multicolumn{4}{c}{\# sampled permutations} \\
         \cmidrule(lr){4-7}
         & & & 1 & 4 & 16 & 64 
         \\
         \midrule\midrule
         \multirow{4}{*}{SGC} & \multirow{2}{*}{Cora} & OT-oracle & 0.84 (0.50s) & 0.89 (1.99s) & 0.9 (8.22s) & 0.88 (23.7s) \\
         && OT bound & 0.84 (0.36s) & 0.87 (1.67s) & 0.87 (6.21s) & 0.86 (20.1s) \\
         \cmidrule(lr){2-7}
         & Amazon- & OT-oracle & 0.76 (9.69s) & 0.93 (39.8s) & 0.89 (147s) & 0.91 (611s) \\
         & ratings & OT bound & 0.76 (0.83s) & 0.91 (3.31s) & 0.87 (13.2s) &  0.92 (49.2s)  \\
          \midrule
         \multirow{4}{*}{GCN} & \multirow{2}{*}{Cora} & OT-oracle & 0.7 (0.69s) & 0.81 (2.58s) & 0.75 (10.5s) & 0.78 (43.9s) \\
         && OT bound & 0.67 (0.17s) & 0.78 (0.68s) & 0.72 (2.84s) & 0.77 (12.1s)  \\
         \cmidrule(lr){2-7}
         & Amazon- & OT-oracle & 0.9 (11.4s) & 0.91 (45.0s) & 0.88 (180s) & 0.94 (728s) \\
         & ratings & OT bound & 0.88 (0.72s) & 0.91 (2.89s) & 0.87 (11.5s) & 0.94 (46.8s)  \\
         \bottomrule
    \end{tabular}
    }
    \label{tab:sample_efficiency}
\end{table}

\newpage

\section{Additional Results on Global Bound}\label{appdx:sec:gb_results}

\subsection{Correlation with the generalization error}\label{appdx:subsec:gb_cor}
\begin{table*}[h!]
    \caption{Correlation between empirical error gap and generalization bounds across datasets and GNNs. $1.0$, $0.9$ and $0.5$ correspond to $1.0$, $0.9$, and $0.5$ percentile of the change rate among all combination sets of $(i, j, y)$ for $M(f,\phi)$ from \cref{thm:global-ot}. Cells are shaded green only for positive correlations, with deeper green indicating stronger alignment. Negative correlations are left unshaded.}
    \centering
    \resizebox{1.0\linewidth}{!}{
    \begin{tabular}{ll *{9}{>{\centering\arraybackslash}p{0.075\linewidth}}}
         \toprule
          & $p$-percentile & Cora & Cite. & PubMed & Comp. & Photo & Squirrel & Chamel. & Roman. & Amazon. \\
         \midrule\midrule
         \multirow{3}{*}{SGC} & 1.0 & -0.03 & \cellcolor[rgb]{0.933,0.966,0.943}\phantom{-}0.19 & \cellcolor[rgb]{0.958,0.978,0.964}\phantom{-}0.12 & -0.51 & \cellcolor[rgb]{0.979,0.989,0.982}\phantom{-}0.06 & \cellcolor[rgb]{0.996,0.998,0.997}\phantom{-}0.01 & -0.37 & \cellcolor[rgb]{0.765,0.879,0.798}\phantom{-}0.67 & -0.12 \\
          & 0.9 & \cellcolor[rgb]{0.713,0.852,0.753}\phantom{-}0.82 & \cellcolor[rgb]{0.678,0.834,0.723}\phantom{-}0.92 & \cellcolor[rgb]{0.660,0.825,0.708}\phantom{-}0.97 & \cellcolor[rgb]{0.954,0.977,0.961}\phantom{-}0.13 & \cellcolor[rgb]{0.702,0.847,0.744}\phantom{-}0.85 & \cellcolor[rgb]{0.660,0.825,0.708}\phantom{-}0.97 & \cellcolor[rgb]{0.713,0.852,0.753}\phantom{-}0.82 & \cellcolor[rgb]{0.702,0.847,0.744}\phantom{-}0.85 & \cellcolor[rgb]{0.653,0.821,0.702}\phantom{-}0.99 \\
          & 0.5 & \cellcolor[rgb]{0.678,0.834,0.723}\phantom{-}0.92 & \cellcolor[rgb]{0.671,0.830,0.717}\phantom{-}0.94 & \cellcolor[rgb]{0.664,0.827,0.711}\phantom{-}0.96 & \cellcolor[rgb]{0.937,0.968,0.946}\phantom{-}0.18 & \cellcolor[rgb]{0.695,0.843,0.738}\phantom{-}0.87 & \cellcolor[rgb]{0.657,0.823,0.705}\phantom{-}0.98 & \cellcolor[rgb]{0.685,0.838,0.729}\phantom{-}0.90 & \cellcolor[rgb]{0.702,0.847,0.744}\phantom{-}0.85 & \cellcolor[rgb]{0.657,0.823,0.705}\phantom{-}0.98 \\
         \midrule
         \multirow{3}{*}{GCN} & 1.0 & \cellcolor[rgb]{0.975,0.987,0.979}\phantom{-}0.07 & \cellcolor[rgb]{0.926,0.962,0.937}\phantom{-}0.21 & \cellcolor[rgb]{0.972,0.986,0.976}\phantom{-}0.08 & -0.63 & -0.62 & \cellcolor[rgb]{0.972,0.986,0.976}\phantom{-}0.08 & \cellcolor[rgb]{0.951,0.975,0.958}\phantom{-}0.14 & -0.44 & -0.42 \\
          & 0.9 & \cellcolor[rgb]{0.692,0.841,0.735}\phantom{-}0.88 & \cellcolor[rgb]{0.667,0.829,0.714}\phantom{-}0.95 & \cellcolor[rgb]{0.664,0.827,0.711}\phantom{-}0.96 & \cellcolor[rgb]{0.688,0.839,0.732}\phantom{-}0.89 & \cellcolor[rgb]{0.734,0.863,0.771}\phantom{-}0.76 & \cellcolor[rgb]{0.692,0.841,0.735}\phantom{-}0.88 & \cellcolor[rgb]{0.793,0.894,0.822}\phantom{-}0.59 & \cellcolor[rgb]{0.678,0.834,0.723}\phantom{-}0.92 & \cellcolor[rgb]{0.709,0.850,0.750}\phantom{-}0.83 \\
          & 0.5 & \cellcolor[rgb]{0.678,0.834,0.723}\phantom{-}0.92 & \cellcolor[rgb]{0.674,0.832,0.720}\phantom{-}0.93 & \cellcolor[rgb]{0.664,0.827,0.711}\phantom{-}0.96 & \cellcolor[rgb]{0.713,0.852,0.753}\phantom{-}0.82 & \cellcolor[rgb]{0.734,0.863,0.771}\phantom{-}0.76 & \cellcolor[rgb]{0.681,0.836,0.726}\phantom{-}0.91 & \cellcolor[rgb]{0.716,0.854,0.756}\phantom{-}0.81 & \cellcolor[rgb]{0.681,0.836,0.726}\phantom{-}0.91 & \cellcolor[rgb]{0.671,0.830,0.717}\phantom{-}0.94 \\
         \midrule
         \multirow{3}{*}{GCNII} & 1.0 & \cellcolor[rgb]{0.863,0.930,0.882}\phantom{-}0.39 & \cellcolor[rgb]{0.849,0.922,0.870}\phantom{-}0.43 & \cellcolor[rgb]{0.695,0.843,0.738}\phantom{-}0.87 & -0.49 & -0.51 & \cellcolor[rgb]{0.968,0.984,0.973}\phantom{-}0.09 & \cellcolor[rgb]{0.989,0.995,0.991}\phantom{-}0.03 & \cellcolor[rgb]{0.709,0.850,0.750}\phantom{-}0.83 & \cellcolor[rgb]{0.779,0.886,0.810}\phantom{-}0.63 \\
          & 0.9 & \cellcolor[rgb]{0.695,0.843,0.738}\phantom{-}0.87 & \cellcolor[rgb]{0.734,0.863,0.771}\phantom{-}0.76 & \cellcolor[rgb]{0.758,0.875,0.792}\phantom{-}0.69 & \cellcolor[rgb]{0.688,0.839,0.732}\phantom{-}0.89 & \cellcolor[rgb]{0.685,0.838,0.729}\phantom{-}0.90 & \cellcolor[rgb]{0.954,0.977,0.961}\phantom{-}0.13 & -0.31 & \cellcolor[rgb]{0.814,0.904,0.840}\phantom{-}0.53 & \cellcolor[rgb]{0.692,0.841,0.735}\phantom{-}0.88 \\
          & 0.5 & \cellcolor[rgb]{0.716,0.854,0.756}\phantom{-}0.81 & \cellcolor[rgb]{0.755,0.874,0.789}\phantom{-}0.70 & \cellcolor[rgb]{0.755,0.874,0.789}\phantom{-}0.70 & \cellcolor[rgb]{0.671,0.830,0.717}\phantom{-}0.94 & \cellcolor[rgb]{0.695,0.843,0.738}\phantom{-}0.87 & \cellcolor[rgb]{0.961,0.980,0.967}\phantom{-}0.11 & -0.33 & \cellcolor[rgb]{0.825,0.910,0.849}\phantom{-}0.50 & \cellcolor[rgb]{0.671,0.830,0.717}\phantom{-}0.94 \\
         \midrule
         \multirow{3}{*}{GAT} & 1.0 & \cellcolor[rgb]{0.919,0.958,0.931}\phantom{-}0.23 & \cellcolor[rgb]{0.870,0.933,0.888}\phantom{-}0.37 & \cellcolor[rgb]{0.842,0.919,0.864}\phantom{-}0.45 & \cellcolor[rgb]{0.800,0.897,0.828}\phantom{-}0.57 & \cellcolor[rgb]{0.821,0.908,0.846}\phantom{-}0.51 & \cellcolor[rgb]{0.772,0.883,0.804}\phantom{-}0.65 & \cellcolor[rgb]{0.706,0.848,0.747}\phantom{-}0.84 & \cellcolor[rgb]{0.716,0.854,0.756}\phantom{-}0.81 & \cellcolor[rgb]{0.783,0.888,0.813}\phantom{-}0.62 \\
          & 0.9 & \cellcolor[rgb]{0.800,0.897,0.828}\phantom{-}0.57 & \cellcolor[rgb]{0.723,0.857,0.762}\phantom{-}0.79 & \cellcolor[rgb]{0.741,0.866,0.777}\phantom{-}0.74 & \cellcolor[rgb]{0.702,0.847,0.744}\phantom{-}0.85 & \cellcolor[rgb]{0.713,0.852,0.753}\phantom{-}0.82 & \cellcolor[rgb]{0.685,0.838,0.729}\phantom{-}0.90 & \cellcolor[rgb]{0.779,0.886,0.810}\phantom{-}0.63 & \cellcolor[rgb]{0.688,0.839,0.732}\phantom{-}0.89 & \cellcolor[rgb]{0.664,0.827,0.711}\phantom{-}0.96 \\
          & 0.5 & \cellcolor[rgb]{0.804,0.899,0.831}\phantom{-}0.56 & \cellcolor[rgb]{0.678,0.834,0.723}\phantom{-}0.92 & \cellcolor[rgb]{0.702,0.847,0.744}\phantom{-}0.85 & \cellcolor[rgb]{0.685,0.838,0.729}\phantom{-}0.90 & \cellcolor[rgb]{0.685,0.838,0.729}\phantom{-}0.90 & \cellcolor[rgb]{0.664,0.827,0.711}\phantom{-}0.96 & \cellcolor[rgb]{0.678,0.834,0.723}\phantom{-}0.92 & \cellcolor[rgb]{0.699,0.845,0.741}\phantom{-}0.86 & \cellcolor[rgb]{0.657,0.823,0.705}\phantom{-}0.98 \\
         \midrule
         \multirow{3}{*}{SAGE} & 1.0 & \cellcolor[rgb]{0.821,0.908,0.846}\phantom{-}0.51 & \cellcolor[rgb]{0.790,0.892,0.819}\phantom{-}0.60 & \cellcolor[rgb]{0.954,0.977,0.961}\phantom{-}0.13 & -0.72 & -0.59 & \cellcolor[rgb]{0.828,0.912,0.852}\phantom{-}0.49 & \cellcolor[rgb]{0.734,0.863,0.771}\phantom{-}0.76 & \cellcolor[rgb]{0.916,0.957,0.928}\phantom{-}0.24 & \cellcolor[rgb]{0.783,0.888,0.813}\phantom{-}0.62 \\
          & 0.9 & \cellcolor[rgb]{0.699,0.845,0.741}\phantom{-}0.86 & \cellcolor[rgb]{0.748,0.870,0.783}\phantom{-}0.72 & \cellcolor[rgb]{0.685,0.838,0.729}\phantom{-}0.90 & \cellcolor[rgb]{0.744,0.868,0.780}\phantom{-}0.73 & \cellcolor[rgb]{0.804,0.899,0.831}\phantom{-}0.56 & \cellcolor[rgb]{0.772,0.883,0.804}\phantom{-}0.65 & \cellcolor[rgb]{0.699,0.845,0.741}\phantom{-}0.86 & \cellcolor[rgb]{0.888,0.942,0.904}\phantom{-}0.32 & \cellcolor[rgb]{0.664,0.827,0.711}\phantom{-}0.96 \\
          & 0.5 & \cellcolor[rgb]{0.727,0.859,0.765}\phantom{-}0.78 & \cellcolor[rgb]{0.699,0.845,0.741}\phantom{-}0.86 & \cellcolor[rgb]{0.692,0.841,0.735}\phantom{-}0.88 & \cellcolor[rgb]{0.667,0.829,0.714}\phantom{-}0.95 & \cellcolor[rgb]{0.748,0.870,0.783}\phantom{-}0.72 & \cellcolor[rgb]{0.748,0.870,0.783}\phantom{-}0.72 & \cellcolor[rgb]{0.723,0.857,0.762}\phantom{-}0.79 & \cellcolor[rgb]{0.930,0.964,0.940}\phantom{-}0.20 & \cellcolor[rgb]{0.664,0.827,0.711}\phantom{-}0.96 \\
         \bottomrule
    \end{tabular}
    }
    \label{tab:apdx:global}
\end{table*}


We repeat the correlation experiment in \cref{subsec:exp_corr_gnn} with the global bound in \cref{thm:global-ot}, using the same $18$ configurations, nine datasets, and five GNN architectures. When computing $M(f, \phi)$, we consider the $1.0$, $0.9$, and $0.5$ percentiles of the change rates over all pairs $(i, j, y)$, and \cref{tab:apdx:global} reports the rank correlation for each. The $1.0$ percentile corresponds to the maximum in the definition of $M(f, \phi)$.
In this case, however, the maximum is dominated by rare extreme pairs, which inflates the bound and often leads to low correlation. Using a lower percentile instead avoids this sensitivity and yields consistently strong correlations: with the $0.9$ percentile, $\rho \geq 0.7$ in $34$ out of $45$ cases, negative only for GCNII on Chameleon, and the results remain similar with the $0.5$ percentile. Note that with a percentile below $1.0$, the reported quantity is a robust estimate of the bound rather than a strict upper bound.

\subsection{Application as a training regularizer}
\begin{table*}[t]
    \caption{Effect of using the global bound as a regularizer on test accuracy across datasets and GNNs. Mean $\pm$ std over multiple runs; $\Delta$ is the gain over the unregularized baseline.}
    \centering
    \resizebox{1.0\linewidth}{!}{
    \begin{tabular}{lccccccccc}
         \toprule
          & Cora & CiteSeer & PubMed & Computers & Photo & Squirrel & Chamel. & Roman. & Amazon. \\
          Homophily  &    0.81    &    0.74    &    0.80    &   0.78     &    0.83    &    0.21    &   0.24   &   0.05     &    0.38  \\
         \midrule\midrule
         GCN & 83.70{\scriptsize $\pm$1.16} & 69.55{\scriptsize $\pm$1.19} & 87.56{\scriptsize $\pm$0.34} & 89.96{\scriptsize $\pm$0.25} & 93.25{\scriptsize $\pm$0.57} & 43.14{\scriptsize $\pm$1.32} & 43.54{\scriptsize $\pm$3.59} & 56.06{\scriptsize $\pm$2.75} & 49.38{\scriptsize $\pm$0.52} \\
         + Global reg. & 85.36{\scriptsize $\pm$1.02} & 70.87{\scriptsize $\pm$1.70} & 87.72{\scriptsize $\pm$0.45} & 91.10{\scriptsize $\pm$0.30} & 93.75{\scriptsize $\pm$0.52} & 45.36{\scriptsize $\pm$1.78} & 42.54{\scriptsize $\pm$2.58} & 63.97{\scriptsize $\pm$2.08} & 49.52{\scriptsize $\pm$0.18} \\
         $\Delta$ & +1.66 & +1.32 & +0.16 & +1.14 & +0.50 & +2.22 & $-$1.00 & +7.91 & +0.14 \\
         \midrule
         GCNII & 84.44{\scriptsize $\pm$1.01} & 71.02{\scriptsize $\pm$1.85} & 87.37{\scriptsize $\pm$0.71} & 89.10{\scriptsize $\pm$1.10} & 92.14{\scriptsize $\pm$1.72} & 41.26{\scriptsize $\pm$1.61} & 42.38{\scriptsize $\pm$3.32} & 70.88{\scriptsize $\pm$0.96} & 48.38{\scriptsize $\pm$0.63} \\
         + Global reg. & 85.29{\scriptsize $\pm$1.11} & 71.62{\scriptsize $\pm$1.22} & 88.94{\scriptsize $\pm$0.24} & 89.79{\scriptsize $\pm$0.37} & 93.31{\scriptsize $\pm$0.28} & 41.80{\scriptsize $\pm$1.73} & 42.76{\scriptsize $\pm$1.70} & 76.23{\scriptsize $\pm$0.90} & 47.97{\scriptsize $\pm$0.49} \\
         $\Delta$ & +0.85 & +0.60 & +1.57 & +0.69 & +1.17 & +0.54 & +0.38 & +5.35 & $-$0.41 \\
         \midrule
         GAT & 82.99{\scriptsize $\pm$1.28} & 69.97{\scriptsize $\pm$1.92} & 85.03{\scriptsize $\pm$0.81} & 90.49{\scriptsize $\pm$0.65} & 93.14{\scriptsize $\pm$0.64} & 38.63{\scriptsize $\pm$3.19} & 40.38{\scriptsize $\pm$2.48} & 57.83{\scriptsize $\pm$1.65} & 50.13{\scriptsize $\pm$0.49} \\
         + Global reg. & 84.73{\scriptsize $\pm$0.43} & 70.84{\scriptsize $\pm$1.55} & 86.80{\scriptsize $\pm$0.84} & 91.70{\scriptsize $\pm$0.51} & 93.95{\scriptsize $\pm$0.22} & 39.18{\scriptsize $\pm$2.76} & 39.40{\scriptsize $\pm$2.17} & 61.25{\scriptsize $\pm$0.91} & 50.09{\scriptsize $\pm$0.55} \\
         $\Delta$ & +1.74 & +0.87 & +1.77 & +1.21 & +0.81 & +0.55 & $-$0.98 & +3.42 & $-$0.04 \\
         \midrule
         SAGE & 84.66{\scriptsize $\pm$0.56} & 70.03{\scriptsize $\pm$2.42} & 87.06{\scriptsize $\pm$0.62} & 88.92{\scriptsize $\pm$0.14} & 94.13{\scriptsize $\pm$0.44} & 36.10{\scriptsize $\pm$1.49} & 39.79{\scriptsize $\pm$3.58} & 76.79{\scriptsize $\pm$0.89} & 47.03{\scriptsize $\pm$0.63} \\
         + Global reg. & 85.21{\scriptsize $\pm$0.93} & 71.68{\scriptsize $\pm$1.74} & 87.79{\scriptsize $\pm$0.34} & 90.17{\scriptsize $\pm$0.62} & 94.52{\scriptsize $\pm$0.40} & 37.07{\scriptsize $\pm$1.90} & 39.59{\scriptsize $\pm$4.40} & 78.71{\scriptsize $\pm$0.86} & 47.25{\scriptsize $\pm$0.60} \\
         $\Delta$ & +0.55 & +1.65 & +0.73 & +1.25 & +0.39 & +0.97 & $-$0.20 & +1.92 & +0.22 \\
         \bottomrule
    \end{tabular}
    }
    \label{tab:apdx:reg_global}
\end{table*}

Following \cref{subsec:expo_reg_gnn}, we use the global bound as a regularizer, $\gL = \gL_\mathrm{CE} + \lambda \gR_\mathrm{global}$, where $\gR_\mathrm{global} \coloneq M(f,\phi)\, \gW_1(\phi_\#\mu_{\gI_{\mathrm{train}}^{(\pi)}}, \phi_\#\mu_{\gI_{\mathrm{test}}^{(\pi)}})$ with $\gamma = 1$. 
As in the class-wise case, the Wasserstein distance is approximated by the Sinkhorn algorithm. For $M(f,\phi)$, we follow \cref{appdx:subsec:gb_cor} and use the $0.9$ percentile of the change rates over all pairs $(i,j,y)$ rather than the maximum, since the maximum is dominated by rare extreme pairs; the order statistic is made differentiable by a temperature-smoothed log-sum-exp over the change rates at or below that percentile.
\cref{tab:apdx:reg_global} shows that the global regularizer improves over the baseline in $31$ out of $36$ setups, with the largest gain of $+7.91\%$ on Roman-empire with GCN.

\section{Additional Results on GNN}
\subsection{Correlation on link prediction}\label{appdx:subsec:link_prediction}
\begin{table*}[!ht]
    \caption{Spearman rank correlation between each bound and the empirical generalization error on link prediction, over 18 architectural configurations. Higher is better; deeper green indicates stronger alignment.}
    \centering
    \resizebox{0.6\linewidth}{!}{
    \begin{tabular}{ll *{3}{>{\centering\arraybackslash}p{0.13\linewidth}}}
         \toprule
          &  & Cora & CiteSeer & PubMed \\
         \midrule\midrule
         \multirow{2}{*}{GCN} & OT-oracle & \cellcolor[rgb]{0.735,0.863,0.772}\phantom{-}0.756 & \cellcolor[rgb]{0.681,0.836,0.726}\phantom{-}0.909 & \cellcolor[rgb]{0.782,0.888,0.813}\phantom{-}0.621 \\
           & OT bound & \cellcolor[rgb]{0.735,0.863,0.772}\phantom{-}0.756 & \cellcolor[rgb]{0.700,0.845,0.742}\phantom{-}0.856 & \cellcolor[rgb]{0.785,0.889,0.816}\phantom{-}0.612 \\
         \midrule
         \multirow{2}{*}{GCN2} & OT-oracle & \cellcolor[rgb]{0.718,0.855,0.759}\phantom{-}0.802 & \cellcolor[rgb]{0.730,0.861,0.769}\phantom{-}0.769 & \cellcolor[rgb]{0.752,0.872,0.787}\phantom{-}0.707 \\
           & OT bound & \cellcolor[rgb]{0.713,0.852,0.754}\phantom{-}0.818 & \cellcolor[rgb]{0.730,0.861,0.769}\phantom{-}0.769 & \cellcolor[rgb]{0.767,0.880,0.800}\phantom{-}0.664 \\
         \midrule
         \multirow{2}{*}{GAT} & OT-oracle & \cellcolor[rgb]{0.734,0.863,0.772}\phantom{-}0.758 & \cellcolor[rgb]{0.771,0.882,0.804}\phantom{-}0.652 & \cellcolor[rgb]{0.741,0.866,0.778}\phantom{-}0.739 \\
           & OT bound & \cellcolor[rgb]{0.734,0.863,0.772}\phantom{-}0.758 & \cellcolor[rgb]{0.776,0.884,0.808}\phantom{-}0.639 & \cellcolor[rgb]{0.718,0.855,0.758}\phantom{-}0.803 \\
         \midrule
         \multirow{2}{*}{SAGE} & OT-oracle & \cellcolor[rgb]{0.765,0.879,0.798}\phantom{-}0.670 & \cellcolor[rgb]{0.829,0.912,0.853}\phantom{-}0.487 & \cellcolor[rgb]{0.709,0.850,0.751}\phantom{-}0.829 \\
           & OT bound & \cellcolor[rgb]{0.765,0.879,0.798}\phantom{-}0.670 & \cellcolor[rgb]{0.822,0.908,0.847}\phantom{-}0.507 & \cellcolor[rgb]{0.710,0.851,0.751}\phantom{-}0.826 \\
         \bottomrule
    \end{tabular}
    }
    \label{tab:link_correlation}
\end{table*}

To verify that our bound generalizes beyond node classification, we measure the rank correlation on a transductive link prediction task. Following the protocol in \cref{subsec:exp_corr_gnn}, we generate $18$ configurations by varying the encoder $\phi$ (GNN) over depths $\{1,2,4,8,16,32\}$ and the predictor $f$ over one-, two-, and four-layer MLPs with ReLU activations, applied to the concatenation of source and target node embeddings. The same OT-oracle and OT bound variants are reported. 
  
We use $30\%$ of observed edges for training and the remaining $70\%$ for testing, with negative edges sampled uniformly at random from non-edges; the model is trained to distinguish positive from negative edges. We evaluate on three datasets,  Cora, CiteSeer, and PubMed, and four GNN encoders (GCN, GCNII, GAT, GraphSAGE). Models are trained with Adam ($\text{lr}=0.01$, no weight decay) for $500$ epochs.
 
\cref{tab:link_correlation} shows that across all four encoders, both OT-oracle and OT bound exhibit strong positive correlation with the empirical generalization error, with most cells above $\rho = 0.7$. The OT bound tracks the oracle closely throughout. These results, together with the node classification experiments in \cref{tab:main_result}, indicate that the proposed bound transfers to other transductive tasks with dependent features.

\subsection{Depth experiments on heterophilic datasets}
\label{apdx_sec:depth_hetero}
\begin{figure*}[ht!]
\centering
\resizebox{\textwidth}{!}{%
\begin{tikzpicture}

\pgfplotsset{
    every axis/.append style={
        width=8cm, height=6cm,
        scale only axis=true,
        axis line style={line width=1.5pt},
        ymajorgrids=true,
        grid style={line width=0.2pt, dashed, gray!50},
        label style={font=\huge},
        tick label style={font=\huge},
        legend style={font=\huge, draw=gray!40, rounded corners=1pt, inner sep=2pt},
        title style={at={(0.5,-0.40)}, anchor=north, font=\Huge, align=center},
        xtick={1,8,16,24,32},
        xmin=0, xmax=33,
        every axis plot/.append style={line width=2.0pt},
        major tick length=0pt,
        minor tick length=0pt,
        yminorgrids=false,
        scaled y ticks=false,
        legend pos=north east,
    },
}

\begin{groupplot}[
    group style={
        group size=2 by 2,
        horizontal sep=1.8cm,
        vertical sep=2.8cm,
    },
]

\nextgroupplot[
    ymin=0, ymax=0.8,
    ytick={0.20,0.40,0.60}, yticklabels={.20,.40,.60}]
\addplot[color=colorOracle, mark=square*, mark size=5.0pt] coordinates {
    (1, 0.0288) (4, 0.0229) (8, 0.02) (12, 0.0182) (16, 0.0169) (20, 0.0159) (24, 0.0151) (28, 0.0145) (32, 0.0139)
};
\addplot[color=colorGlobal, mark=triangle*, mark size=5.5pt] coordinates {
    (1, 0.711) (4, 0.452) (8, 0.363) (12, 0.32) (16, 0.294) (20, 0.275) (24, 0.261) (28, 0.25) (32, 0.241)
};
\legend{$\mathcal{W}_C$, $\mathcal{W}_S$}

\nextgroupplot[
    ymin=0.40, ymax=0.80,
    ytick={0.50,0.60,0.70}, yticklabels={.50,.60,.70}]
\addplot[name path=upB, draw=none, forget plot] coordinates {
    (1, 0.557) (4, 0.705) (8, 0.721) (12, 0.665) (16, 0.606) (20, 0.558) (24, 0.515) (28, 0.477) (32, 0.455)
};
\addplot[name path=loB, draw=none, forget plot] coordinates {
    (1, 0.551) (4, 0.698) (8, 0.699) (12, 0.634) (16, 0.571) (20, 0.521) (24, 0.481) (28, 0.461) (32, 0.435)
};
\addplot[colorGen, fill opacity=0.18, forget plot] fill between[of=upB and loB];
\addplot[color=colorGen, mark=diamond*, mark size=5.5pt] coordinates {
    (1, 0.554) (4, 0.702) (8, 0.71) (12, 0.65) (16, 0.588) (20, 0.539) (24, 0.498) (28, 0.469) (32, 0.445)
};
\legend{Gen. error}

\nextgroupplot[
    title={(a) SGC: Wasserstein}, xlabel={Depth},
    ymin=0, ymax=0.24,
    ytick={0.06,0.12,0.18}, yticklabels={.06,.12,.18}]
\addplot[color=colorOracle, mark=square*, mark size=5.0pt] coordinates {
    (1, 0.0553) (4, 0.033) (8, 0.0282) (12, 0.0256) (16, 0.0237) (20, 0.0223) (24, 0.0212) (28, 0.0202) (32, 0.0194)
};
\addplot[color=colorGlobal, mark=triangle*, mark size=5.5pt] coordinates {
    (1, 0.214) (4, 0.144) (8, 0.125) (12, 0.114) (16, 0.107) (20, 0.101) (24, 0.0965) (28, 0.0927) (32, 0.0895)
};
\legend{$\mathcal{W}_C$, $\mathcal{W}_S$}

\nextgroupplot[
    title={(b) SGC: Gen. error}, xlabel={Depth},
    ymin=0.15, ymax=0.55,
    ytick={0.25,0.35,0.45}, yticklabels={.25,.35,.45}]
\addplot[name path=upF, draw=none, forget plot] coordinates {
    (1, 0.518) (4, 0.272) (8, 0.233) (12, 0.198) (16, 0.184) (20, 0.178) (24, 0.166) (28, 0.16) (32, 0.152)
};
\addplot[name path=loF, draw=none, forget plot] coordinates {
    (1, 0.509) (4, 0.254) (8, 0.207) (12, 0.191) (16, 0.175) (20, 0.164) (24, 0.152) (28, 0.143) (32, 0.136)
};
\addplot[colorGen, fill opacity=0.18, forget plot] fill between[of=upF and loF];
\addplot[color=colorGen, mark=diamond*, mark size=5.5pt] coordinates {
    (1, 0.514) (4, 0.263) (8, 0.22) (12, 0.195) (16, 0.18) (20, 0.171) (24, 0.159) (28, 0.152) (32, 0.144)
};
\legend{Gen. error}

\end{groupplot}


\begin{axis}[name=panelC,
    at={($(group c2r1.east)+(3.5cm,0)$)}, anchor=west,
    ymode=log, ymin=0.0003, ymax=10,
    ytick={0.01,0.1,1}, yticklabels={.01,.1,1}]
\addplot[color=colorOracle, mark=square*, mark size=5.0pt] coordinates {
    (1, 0.115) (4, 0.0946) (8, 0.0362) (12, 0.0131) (16, 0.00282) (20, 0.000988) (24, 0.00111) (28, 0.00112) (32, 0.00112)
};
\addplot[color=colorGlobal, mark=triangle*, mark size=5.5pt] coordinates {
    (1, 2.93) (4, 3.2) (8, 2.01) (12, 0.969) (16, 0.313) (20, 0.212) (24, 0.211) (28, 0.211) (32, 0.21)
};
\legend{$\mathcal{W}_C$, $\mathcal{W}_S$}
\end{axis}

\begin{axis}[name=panelD,
    at={($(panelC.east)+(1.8cm,0)$)}, anchor=west,
    ymin=0, ymax=0.6,
    ytick={0.15,0.30,0.45}, yticklabels={.15,.30,.45}]
\addplot[name path=upD, draw=none, forget plot] coordinates {
    (1, 0.573) (4, 0.534) (8, 0.303) (12, 0.159) (16, 0.0674) (20, 0.00908) (24, 0.00884) (28, 0.00511) (32, 0.0042)
};
\addplot[name path=loD, draw=none, forget plot] coordinates {
    (1, 0.557) (4, 0.432) (8, 0.259) (12, 0.0681) (16, -0.0142) (20, 0.00174) (24, -0.000524) (28, 0.000547) (32, 0.00206)
};
\addplot[colorGen, fill opacity=0.18, forget plot] fill between[of=upD and loD];
\addplot[color=colorGen, mark=diamond*, mark size=5.5pt] coordinates {
    (1, 0.565) (4, 0.483) (8, 0.281) (12, 0.114) (16, 0.0266) (20, 0.00541) (24, 0.00416) (28, 0.00283) (32, 0.00313)
};
\legend{Gen. error}
\end{axis}

\begin{axis}[name=panelG,
    at={($(group c2r2.east)+(3.5cm,0)$)}, anchor=west,
    title={(c) GCN: Wasserstein}, xlabel={Depth},
    ymode=log, ymin=0.001, ymax=10,
    ytick={0.01,0.1,1}, yticklabels={.01,.1,1}]
\addplot[color=colorOracle, mark=square*, mark size=5.0pt] coordinates {
    (1, 0.122) (4, 0.0671) (8, 0.0964) (12, 0.0581) (16, 0.0473) (20, 0.0193) (24, 0.012) (28, 0.00975) (32, 0.00968)
};
\addplot[color=colorGlobal, mark=triangle*, mark size=5.5pt] coordinates {
    (1, 0.471) (4, 0.733) (8, 1.49) (12, 1.24) (16, 1.2) (20, 0.483) (24, 0.259) (28, 0.121) (32, 0.109)
};
\legend{$\mathcal{W}_C$, $\mathcal{W}_S$}
\end{axis}

\begin{axis}[name=panelH,
    at={($(panelG.east)+(1.8cm,0)$)}, anchor=west,
    title={(d) GCN: Gen. error}, xlabel={Depth},
    ymin=0, ymax=0.48,
    ytick={0.12,0.24,0.36}, yticklabels={.12,.24,.36}]
\addplot[name path=upH, draw=none, forget plot] coordinates {
    (1, 0.479) (4, 0.242) (8, 0.146) (12, 0.0986) (16, 0.0795) (20, 0.0658) (24, 0.0444) (28, 0.00862) (32, 0.006)
};
\addplot[name path=loH, draw=none, forget plot] coordinates {
    (1, 0.422) (4, 0.181) (8, 0.128) (12, 0.0805) (16, 0.056) (20, 0.0142) (24, -0.00499) (28, 0.00392) (32, 0.00112)
};
\addplot[colorGen, fill opacity=0.18, forget plot] fill between[of=upH and loH];
\addplot[color=colorGen, mark=diamond*, mark size=5.5pt] coordinates {
    (1, 0.45) (4, 0.212) (8, 0.137) (12, 0.0895) (16, 0.0678) (20, 0.04) (24, 0.0197) (28, 0.00627) (32, 0.00356)
};
\legend{Gen. error}
\end{axis}

\node[rotate=90, font=\Huge]
    at ($(group c1r1.outer west |- group c1r1.center)+(-0.7cm,0)$) {Roman};
\node[rotate=90, font=\Huge]
    at ($(group c1r2.outer west |- group c1r2.center)+(-0.7cm,0)$) {Amazon};

\node[rotate=90, font=\Huge, opacity=0]
    at ($(panelD.outer east |- panelD.center)+(0.7cm,0)$) {Roman};
\node[rotate=90, font=\Huge, opacity=0]
    at ($(panelH.outer east |- panelH.center)+(0.7cm,0)$) {Amazon};

\end{tikzpicture}
}
\caption{Depth analysis on Roman-empire (top row) and Amazon-ratings (bottom row) for SGC and GCN, showing Wasserstein distances $\gW_C, \gW_S$ and the empirical generalization error denoted Gen. error in the figure at varying depths. Shaded regions indicate $\pm 1$ standard deviation across runs.}
\label{fig:depth_analysis_2x4_hetero}
\end{figure*}

We additionally provide the depth analysis on heterophilic graph datasets, Roman-empire and Amazon-ratings. \cref{fig:depth_analysis_2x4_hetero} shows the Wasserstein distances $\gW_C, \gW_S$ and the empirical generalization error at varying depths for SGC and GCN. 

For SGC, the Wasserstein distances $\gW_C$ and $\gW_S$ exponentially decay with depth, consistent with \cref{prop:w1-sgc}; for GCN, their trajectories additionally reflect the $\beta^\ell$ scaling in \cref{prop:w1-gcn}. On Roman-empire with SGC, the error gap exhibits a non-monotonic dependence on depth: it initially increases as the loss of inter-class separation dominates, then decreases as the strengthening of intra-class concentration outweighs. The other three settings exhibit a monotonically decreasing error, where intra-class concentration dominates at every depth.

Interestingly, GCN on Amazon-ratings offers a particularly clean 
illustration of the trade-off in \cref{prop:w1-gcn}. Over depths $1 \to 4$, $\gW_S$ increases while $\gW_C$ decreases, i.e., both terms move in the favorable direction, and the error gap drops sharply over exactly this range. The transient growth of $\gW_S$ here is consistent with the $\beta^\ell$ factor in our GCN bound, which can amplify the feature scale.

\section{Additional analyses of the OT regularizer}\label{apdx:subsec:overhead}
\subsection{Optimization stability of Amazon-ratings}
\label{app:opt-stability}
\begin{figure}[ht]
    \centering
    \includegraphics[width=\linewidth]{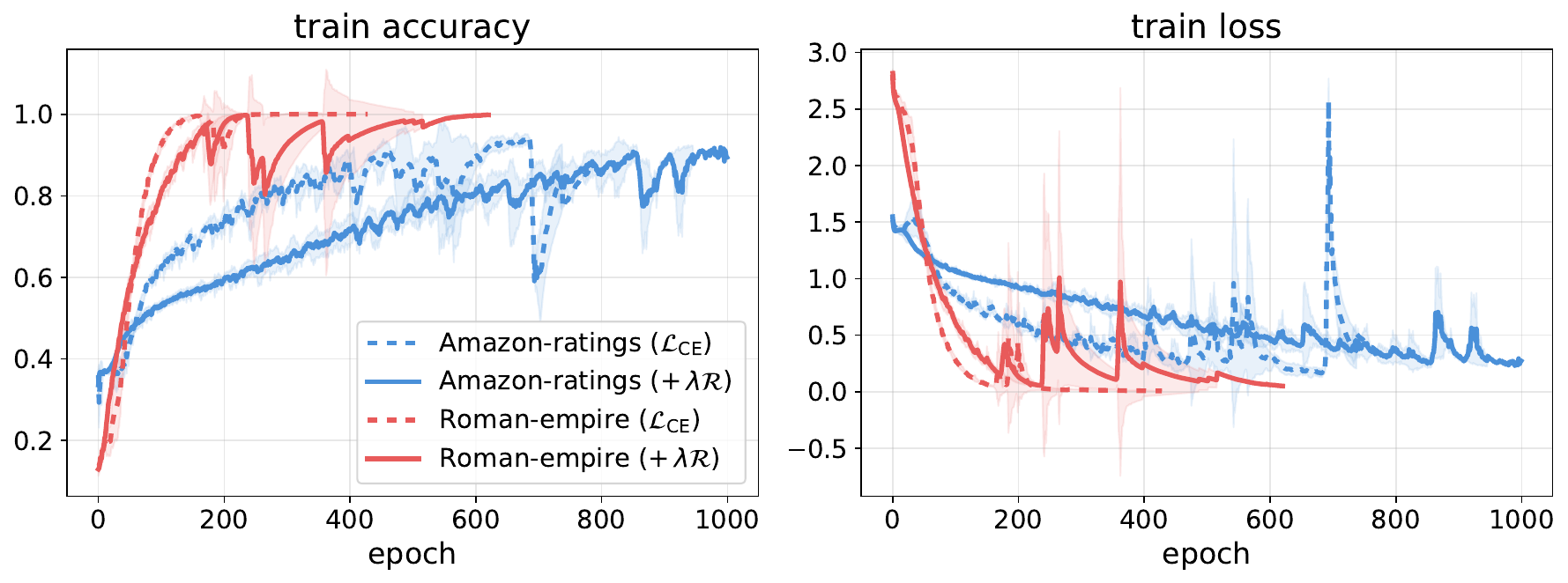}
    \caption{Training trajectories of an 8-layer GCN encoder with a 2-layer MLP classifier on Roman-Empire and Amazon-ratings, with ($\lambda = 10^{-3}$) and without the OT regularizer. Curves show the mean across 5 random seeds; shaded bands indicate standard deviation. Training is run for up to 1000 epochs with early stopping (patience 200 on validation accuracy). Roman-empire converges smoothly to near-perfect training accuracy in both settings, whereas Amazon-ratings plateaus below $1.0$ and exhibits persistent oscillations regardless of regularization, indicating that the optimization difficulty is intrinsic to this dataset rather than induced by the regularizer.}
    \label{fig:opt-stability-curves}
\end{figure}

Among the nine datasets in our regularization experiments (\cref{tab:regularization_result}), Amazon-ratings is the only one in which adding the OT regularizer fails to improve upon the cross-entropy baseline. We attribute this to instability in the underlying optimization rather than to a limitation of the regularizer itself. To make this concrete, \cref{fig:opt-stability-curves} compares the training trajectories of an 8-layer GCN encoder with a 2-layer MLP classifier on Amazon-ratings and Roman-Empire, the latter chosen as a reference case showing the largest gain in our experiments. For both datasets, we compare the baseline against the regularized model with $\lambda = 10^{-3}$. Curves show the mean across five random seeds with one-standard-deviation bands.

On Roman-empire, both the baseline and the regularized model converge smoothly to near-perfect training accuracy and near-zero training loss within a few hundred epochs. On Amazon-ratings, neither setting reaches training accuracy of $1.0$: trajectories plateau in the $0.85$ to $0.9$ range with persistent oscillations throughout training, and the loss curves exhibit corresponding fluctuations rather than the smooth decay observed on Roman-empire. Importantly, this behavior is present \emph{with and without} the regularizer, indicating that the difficulty stems from the dataset-architecture interaction rather than from the regularization signal.

This observation has two consequences. First, when the underlying optimization cannot stably fit the training data even with cross-entropy alone, there is no well-fit representation geometry on which the regularizer can operate; the regularizer is designed to refine an already-fit encoder, not to compensate for failures of the base optimization. Second, our hyperparameter search reflects this: across all four GNN architectures, the validation procedure selects the smallest regularization weight in our search grid ($\lambda = 10^{-5}$) for Amazon-ratings, whereas other datasets typically prefer $\lambda = 10^{-3}$ or $10^{-2}$. With $\lambda = 10^{-5}$, training trajectories are nearly indistinguishable from the CE-only baseline (not shown), so the validation procedure is effectively choosing to turn the regularizer off rather than risk amplifying the existing instability.

\subsection{Computational overhead of the regularizer}
\label{apdx:subsec:overhead}

We analyze the computational cost of the OT regularizer relative to standard cross-entropy training. With $N$ nodes, $K$ classes, embedding dimension $d$, $T$ sampled permutations, $m$ labeled nodes, and $I$ Sinkhorn iterations, computing the regularizer takes $O(NK + TKd + Tm^2(d+I)/K)$ time and $O(NK + TK + Tm^2/K)$ peak memory per epoch.

\begin{table}[t]
\centering
\caption{Per-epoch cost (time / peak GPU memory) of training with and without the OT regularizer, using an 8-layer GCN encoder and a 2-layer MLP classifier on a single NVIDIA GeForce RTX 3090. We report the median of three runs; timing covers the training step only, as the evaluation pass is identical with and without the regularizer.}
\label{tab:overhead}
\begin{tabular}{lccc}
\toprule
 & Cross entropy & + OT reg. & Increase \\
\midrule
Cora           & 13.1 ms / 40 MB  & 33.1 ms / 58 MB  & +20.0 ms / +18 MB \\
CiteSeer       & 13.1 ms / 77 MB  & 31.5 ms / 90 MB  & +18.4 ms / +13 MB \\
Amazon-ratings & 16.7 ms / 301 MB & 34.3 ms / 302 MB & +17.6 ms / +1 MB  \\
Roman-empire   & 14.3 ms / 208 MB & 75.5 ms / 281 MB & +61.2 ms / +74 MB \\
\bottomrule
\end{tabular}
\end{table}

\cref{tab:overhead} reports the measured per-epoch overhead on four representative datasets, including the worst case. Across all nine datasets, the regularizer adds 14.5--32.1 ms and at most 27 MB per epoch, except on Roman-empire (+61.2 ms, +74 MB), where the full 500-epoch training still completes in under 40 seconds. Overall, the overhead remains modest even on the largest graphs.

\subsection{Statistical significance of the regularization gains}
\label{apdx:subsec:ttest}

To assess whether the gains in \cref{tab:regularization_result} are statistically significant, we conduct paired t-tests between the regularized and unregularized models on two homophilic datasets (Cora, CiteSeer) and two heterophilic datasets (Squirrel, Roman-empire).
\cref{tab:ttest} reports the p-values, pairing the five random seeds used in \cref{tab:regularization_result}.
The OT regularizer yields statistically significant improvements ($p < 0.05$) in 11 out of 16 settings, across both homophilic and heterophilic graphs.

\begin{table}[t]
\centering
\caption{P-values of paired t-tests comparing training with and without
the OT regularizer. Bold indicates statistically significant improvement
($p < 0.05$).}
\label{tab:ttest}
\resizebox{0.5\linewidth}{!}{%

\begin{tabular}{lcccc}
\toprule
Model & Cora & CiteSeer & Squirrel & Roman-emp. \\
\midrule
GCN   & 0.2162          & \textbf{0.0229} & \textbf{0.0480} & \textbf{0.0031} \\
GCNII & 0.9997          & 0.2491          & \textbf{0.0092} & \textbf{0.0026} \\
GAT   & \textbf{0.0312} & \textbf{0.0185} & 0.3099          & \textbf{0.0136} \\
SAGE  & \textbf{0.0197} & \textbf{0.0321} & 0.2532          & \textbf{0.0204} \\
\bottomrule
\end{tabular}
}
\end{table}

\subsection{Comparison with standard regularizers}
\label{apdx:subsec:reg_comparison}

We compare the OT regularizer against three standard regularization techniques, weight decay~\citep{krogh1991simple}, Dropout~\citep{srivastava2014dropout}, and DropEdge~\citep{Rong2020DropEdge}, on GCN under the same experimental setting as \cref{subsec:expo_reg_gnn}. We use a weight decay of $10^{-2}$, a dropout rate of 0.1, and an edge drop rate of 0.5 for all datasets. Importantly, our regularizer is complementary to these methods: it is a loss term and can be added on top of any of them. For each baseline, we therefore also report the gain from adding the OT regularizer on top.

\cref{tab:reg_comparison} shows the results. The OT regularizer alone is competitive with the standard techniques, and adding it improves the corresponding baseline in all 12 combinations.

\begin{table}[t]
\centering
\caption{Test accuracy (\%) of each regularizer on GCN. The + OT reg.
rows report the gain from adding the OT regularizer on top of the
corresponding baseline.}
\label{tab:reg_comparison}
\resizebox{0.75\linewidth}{!}{%

\begin{tabular}{lcccc}
\toprule
Method & Cora & CiteSeer & Squirrel & Roman-emp. \\
\midrule
OT reg.      & 84.55 $\pm$ 1.00 & 71.89 $\pm$ 1.29 & 45.27 $\pm$ 1.52 & 64.68 $\pm$ 1.86 \\
\midrule
Weight decay & 83.47 $\pm$ 1.54 & 69.94 $\pm$ 1.00 & 44.46 $\pm$ 1.87 & 60.19 $\pm$ 0.96 \\
+ OT reg.    & +1.07            & +1.20            & +0.23            & +1.02            \\
\midrule
Dropout      & 84.92 $\pm$ 1.02 & 70.39 $\pm$ 1.32 & 43.31 $\pm$ 1.40 & 69.22 $\pm$ 0.52 \\
+ OT reg.    & +0.37            & +0.60            & +0.13            & +2.22            \\
\midrule
DropEdge     & 84.88 $\pm$ 1.47 & 71.26 $\pm$ 1.21 & 43.27 $\pm$ 1.06 & 58.58 $\pm$ 0.37 \\
+ OT reg.    & +0.44            & +0.36            & +1.43            & +2.70            \\
\bottomrule
\end{tabular}
}
\end{table}

\subsection{Ablation of the regularizer terms}
\label{apdx:subsec:ablation}

Our regularizer $\mathcal{R}$ consists of two factors, the transport term and the margin-sensitivity term $M_c(f, \phi)$. To isolate their contributions, we train with each factor alone:
\begin{align*}
\text{full regularization:}\quad
&\mathcal{R} = \sum_{c=1}^{K} M_c(f, \phi)\,
\mathbb{E}_{\pi'}\!\left[\tfrac{m_c}{m}
\mathcal{W}\big(\phi_\#\mu_{I^{(\pi')}_{\mathrm{train},c}},
\phi_\#\mu_{I^{(\pi')}_{\mathrm{test},c}}\big)\right], \\
\text{transport-only:}\quad
&\mathcal{R}_t = \sum_{c=1}^{K}
\mathbb{E}_{\pi'}\!\left[\tfrac{m_c}{m}
\mathcal{W}\big(\phi_\#\mu_{I^{(\pi')}_{\mathrm{train},c}},
\phi_\#\mu_{I^{(\pi')}_{\mathrm{test},c}}\big)\right], \\
\text{margin-sensitivity-only:}\quad
&\mathcal{R}_m = \sum_{c=1}^{K} M_c(f, \phi).
\end{align*}

All variants follow the identical tuning protocol of \cref{subsec:expo_reg_gnn}. \cref{tab:ablation} reports mean $\pm$ std over five seeds on three homophilic datasets, Cora, CiteSeer, and PubMed, and three heterophilic datasets, Squirrel, Chameleon, and Roman-empire, with GCN.

Three observations follow. 
First, the transport term is the primary driver of the gains: it improves over the baseline on all six datasets, accounting for most of the improvement reported in \cref{tab:regularization_result}. 
Second, the effect of the margin-sensitivity term depends on homophily. In isolation, it consistently helps on the homophilic datasets but consistently hurts on the heterophilic ones. We attribute this to the fact that under heterophily, features are heavily mixed, so a sharp classifier is what makes them separable; penalizing $M_c$ alone forces the classifier to vary smoothly across feature space and thus hinders fitting. 
Third, the two terms are complementary only on heterophilic graphs: combined with the transport term, the margin term slightly offsets the transport-only gain on the homophilic datasets but helps on the heterophilic ones.

\begin{table}[t]
\centering
\caption{Ablation of the regularizer factors on GCN. Mean $\pm$ std over
5 seeds.}
\label{tab:ablation}
\resizebox{\textwidth}{!}{%
\begin{tabular}{lcccccc}
\toprule
 & Cora & CiteSeer & PubMed & Squirrel & Chameleon & Roman-emp. \\
\midrule
GCN                        & 83.70 $\pm$ 1.16 & 69.55 $\pm$ 1.19 & 87.56 $\pm$ 0.34 & 43.14 $\pm$ 1.32 & 43.54 $\pm$ 3.59 & 56.06 $\pm$ 2.75 \\
+ transport-only           & 85.36 $\pm$ 1.04 & 72.28 $\pm$ 0.98 & 87.80 $\pm$ 0.32 & 45.27 $\pm$ 1.52 & 44.08 $\pm$ 3.15 & 62.11 $\pm$ 1.75 \\
+ margin-sensitivity-only  & 84.14 $\pm$ 1.30 & 70.18 $\pm$ 1.84 & 88.01 $\pm$ 0.54 & 42.29 $\pm$ 0.85 & 42.08 $\pm$ 1.68 & 54.81 $\pm$ 2.31 \\
+ full regularization      & 84.55 $\pm$ 1.00 & 71.89 $\pm$ 1.29 & 87.68 $\pm$ 0.24 & 45.27 $\pm$ 1.52 & 44.28 $\pm$ 3.31 & 64.68 $\pm$ 1.86 \\
\bottomrule
\end{tabular}%
}
\end{table}

\subsection{Sensitivity analysis of the regularizer}
\label{apdx:subsec:sensitivity}

We examine the sensitivity of the regularization gains to each hyperparameter on four representative (dataset, architecture)
configurations. In \cref{tab:sens_lambda,tab:sens_sinkhorn,tab:sens_perm}, bold indicates configurations that improve over the unregularized baseline.

\paragraph{Regularization weight $\lambda$.}
\cref{tab:sens_lambda} shows that performance is stable across $\lambda \in \{10^{-5}, \ldots, 10^{-2}\}$ and degrades only at $\lambda = 10^{-1}$, where the regularizer overwhelms the cross-entropy loss.

\paragraph{Margin threshold $\gamma$.}
The training objective, with the $1/\gamma$ factor of \cref{thm:classwise-ot} written out explicitly, is $\mathcal{L} = \mathcal{L}_{\mathrm{CE}} + \frac{\lambda}{\gamma}\mathcal{R}$.
We set $\gamma = 1$ and did not vary it: since $\gamma$ enters the regularizer only through the multiplicative factor $1/\gamma$, its effect is fully absorbed into the choice of $\lambda$, so the $\lambda$ sensitivity above also covers $\gamma$.

\paragraph{Sinkhorn regularization strength.}
We vary the entropic coefficient $\varepsilon$ of the Sinkhorn algorithm used to approximate the Wasserstein term over $\{0.01, 0.05, 0.1, 0.5\}$. As shown in \cref{tab:sens_sinkhorn}, the gains are largely insensitive to $\varepsilon$, with only a mild drop on CiteSeer at the largest value $\varepsilon = 0.5$.

\paragraph{Number of sampled permutations.}
\cref{tab:sens_perm} shows that the gains are insensitive to the number of sampled permutations around our default of $T = 1$, consistent with our observation in \cref{apdx_sec:approximate} for the bound estimation.

In summary, the gains are robust across a wide range of all hyperparameters; the only sensitive regime is an excessively large $\lambda$, which is naturally avoided by validation-based selection.

\begin{table}[t]
\centering
\caption{Sensitivity to the regularization weight $\lambda$.}
\label{tab:sens_lambda}
\resizebox{0.9\linewidth}{!}{%
\begin{tabular}{lccccc}
\toprule
$\lambda$ & $10^{-5}$ & $10^{-4}$ & $10^{-3}$ & $10^{-2}$ & $10^{-1}$ \\
\midrule
Cora/GCN         & \textbf{85.36 $\pm$ 1.08} & \textbf{84.62 $\pm$ 0.75} & \textbf{84.55 $\pm$ 1.00} & 81.85 $\pm$ 1.68 & 50.68 $\pm$ 18.60 \\
Cora/GAT         & \textbf{84.07 $\pm$ 0.94} & \textbf{84.81 $\pm$ 0.69} & \textbf{84.51 $\pm$ 1.15} & \textbf{84.70 $\pm$ 1.53} & 67.62 $\pm$ 19.31 \\
Roman-empire/GCN & 54.25 $\pm$ 1.17 & \textbf{56.32 $\pm$ 1.27} & \textbf{64.68 $\pm$ 1.86} & \textbf{58.05 $\pm$ 1.49} & 22.49 $\pm$ 5.47 \\
CiteSeer/GCN     & \textbf{70.84 $\pm$ 1.14} & \textbf{70.00 $\pm$ 1.29} & \textbf{71.89 $\pm$ 1.29} & \textbf{70.63 $\pm$ 1.14} & 53.66 $\pm$ 18.98 \\
\bottomrule
\end{tabular}%
}
\end{table}

\begin{table}[t]
\centering
\caption{Sensitivity to the Sinkhorn entropic coefficient $\varepsilon$.}
\label{tab:sens_sinkhorn}
\resizebox{0.8\linewidth}{!}{%
\begin{tabular}{lcccc}
\toprule
$\varepsilon$ & 0.01 & 0.05 & 0.1 & 0.5 \\
\midrule
Cora/GCN         & \textbf{84.70 $\pm$ 0.97} & \textbf{84.55 $\pm$ 1.00} & \textbf{85.29 $\pm$ 0.91} & \textbf{84.77 $\pm$ 0.79} \\
Cora/GAT         & \textbf{84.66 $\pm$ 1.12} & \textbf{84.51 $\pm$ 1.15} & \textbf{85.18 $\pm$ 0.46} & \textbf{84.58 $\pm$ 1.21} \\
CiteSeer/GCN     & \textbf{70.72 $\pm$ 1.46} & \textbf{71.89 $\pm$ 1.29} & \textbf{71.14 $\pm$ 1.32} & 68.98 $\pm$ 1.30 \\
Roman-empire/GCN & \textbf{64.00 $\pm$ 1.26} & \textbf{64.68 $\pm$ 1.86} & \textbf{63.17 $\pm$ 2.00} & \textbf{64.72 $\pm$ 0.78} \\
\bottomrule
\end{tabular}
}
\end{table}

\begin{table}[t]
\centering
\caption{Sensitivity to the number of sampled permutations $T$.}
\label{tab:sens_perm}
\resizebox{0.8\linewidth}{!}{%
\begin{tabular}{lcccc}
\toprule
$T$ & 1 & 2 & 4 & 8 \\
\midrule
Cora/GCN         & \textbf{84.55 $\pm$ 1.00} & \textbf{85.03 $\pm$ 0.80} & \textbf{84.88 $\pm$ 0.85} & \textbf{85.03 $\pm$ 1.17} \\
Cora/GAT         & \textbf{84.51 $\pm$ 1.15} & \textbf{84.58 $\pm$ 0.59} & \textbf{84.81 $\pm$ 0.91} & \textbf{84.70 $\pm$ 1.24} \\
CiteSeer/GCN     & \textbf{71.89 $\pm$ 1.29} & \textbf{70.24 $\pm$ 1.14} & \textbf{71.59 $\pm$ 1.07} & \textbf{70.27 $\pm$ 1.51} \\
Roman-empire/GCN & \textbf{64.68 $\pm$ 1.86} & \textbf{62.48 $\pm$ 1.01} & \textbf{63.71 $\pm$ 2.44} & \textbf{62.97 $\pm$ 1.68} \\
\bottomrule
\end{tabular}
}
\end{table}

\section{Additional results beyond GNN}\label{appdx:sec:experiment_others}
\begin{table}[ht!]
    \caption{Spearman rank correlation between generalization loss and each bound on tabular datasets with MLP. Cells are shaded green only for positive correlations, with deeper green indicating stronger alignment with performance and negative correlations are left unshaded.}
    \centering
    \resizebox{0.5\linewidth}{!}{%

    \begin{tabular}{llccc}
         \toprule
          &  & Breast & Letter & CovType \\
         \midrule\midrule
         \multirow{3}{*}{MLP} & PAC bound & $-$0.440 & $-$0.351 & $-$0.519 \\
           & OT-oracle & \cellcolor[rgb]{0.700,0.845,0.742}\phantom{-}0.856 & \cellcolor[rgb]{0.718,0.855,0.759}\phantom{-}0.802 & \cellcolor[rgb]{0.893,0.945,0.908}\phantom{-}0.305 \\
           & OT bound & \cellcolor[rgb]{0.703,0.847,0.746}\phantom{-}0.845 & \cellcolor[rgb]{0.718,0.855,0.759}\phantom{-}0.802 & \cellcolor[rgb]{0.897,0.947,0.912}\phantom{-}0.293 \\
         \bottomrule
    \end{tabular}
    }
    \label{tab:tabular_mlp_correlation}
\end{table}

\begin{table}[ht!]
    \caption{Spearman rank correlation between generalization loss and each bound on image-shaped tabular datasets.}
    \centering
    \resizebox{0.5\linewidth}{!}{%
    \begin{tabular}{llccc}
         \toprule
          &  & Digits & MNIST & F-MNIST \\
         \midrule\midrule
         \multirow{3}{*}{MLP} & PAC bound & \cellcolor[rgb]{0.950,0.974,0.957}\phantom{-}0.142 & $-$0.520 & $-$0.863 \\
           & OT-oracle & \cellcolor[rgb]{0.764,0.878,0.797}\phantom{-}0.673 & \cellcolor[rgb]{0.740,0.866,0.777}\phantom{-}0.742 & \cellcolor[rgb]{0.728,0.860,0.767}\phantom{-}0.775 \\
           & OT bound & \cellcolor[rgb]{0.768,0.880,0.801}\phantom{-}0.661 & \cellcolor[rgb]{0.747,0.869,0.783}\phantom{-}0.722 & \cellcolor[rgb]{0.740,0.866,0.777}\phantom{-}0.740 \\
         \midrule
         \multirow{3}{*}{CNN} & PAC bound & $-$0.577 & $-$0.776 & $-$0.594 \\
           & OT-oracle & \cellcolor[rgb]{0.676,0.833,0.722}\phantom{-}0.923 & \cellcolor[rgb]{0.662,0.826,0.710}\phantom{-}0.962 & \cellcolor[rgb]{0.691,0.841,0.735}\phantom{-}0.880 \\
           & OT bound & \cellcolor[rgb]{0.673,0.832,0.719}\phantom{-}0.932 & \cellcolor[rgb]{0.662,0.826,0.710}\phantom{-}0.962 & \cellcolor[rgb]{0.691,0.841,0.735}\phantom{-}0.880 \\
         \bottomrule
    \end{tabular}
    }
    \label{tab:image_correlation}
\end{table}

To examine the generality of our bound, we extend the experiments in 
\cref{sec:beyond} to two additional transductive settings: 
tabular classification, and image classification.

\subsection{Correlation beyond GNN}
We extend the correlation analysis of \cref{subsec:corr_image} along two axes. First, we evaluate our bound on three tabular datasets: Breast, Letter, and CovType, using MLP encoders. Second, we evaluate it on three additional image classification benchmarks: Digits, MNIST, and Fashion-MNIST, using both MLP and CNN encoders. 

Following the protocol of \cref{subsec:exp_corr_gnn}, we randomly select $30\%$ of the samples as labeled (training) and use the remaining $70\%$ as unlabeled (test). Models are trained with Adam (learning rate $0.01$, hidden dimension $64$) for $500$ epochs. For each dataset and architecture pair, we generate $18$ architectural configurations by varying the encoder depth in $ \{1, 2, 4, 8, 16, 32\}$ and the classifier depth in $\{1, 2, 4\}$, and compute the Spearman rank correlation between our bound and the empirical generalization gap over the $18$ architectural configurations.

\cref{tab:image_correlation,tab:tabular_mlp_correlation} report the Spearman rank correlation between each bound and the empirical generalization gap. Across both tabular and image settings, OT-oracle and OT bound consistently align with the gap, while the PAC baseline often exhibits strong negative correlation. These results indicate that our bound generalizes reliably beyond graph and image classification to tabular settings as well.

\subsection{Regularization beyond GNN}
\label{apdx_sec:reg-tabular}

\begin{table*}[ht!]
    \caption{Effect of using our OT bound as a regularizer on test accuracy across image-shaped tabular datasets and architectures (MLP, CNN). Mean $\pm$ std over 5 runs; $\Delta$ is the gain over the unregularized baseline.}
    \centering
    \resizebox{0.5\linewidth}{!}{%

    \begin{tabular}{lcccc}
         \toprule
          & Digits & MNIST & F-MNIST  \\
         \midrule\midrule
         MLP & 96.67{\scriptsize $\pm$0.88} & 96.25{\scriptsize $\pm$0.17} & 88.40{\scriptsize $\pm$0.25} \\
         + OT bound & 97.06{\scriptsize $\pm$1.03} & 96.24{\scriptsize $\pm$0.24} & 88.38{\scriptsize $\pm$0.26} \\
         $\Delta$ & +0.39 & $-$0.01 & $-$0.02 \\
         \midrule
         CNN & 96.67{\scriptsize $\pm$1.22} & 99.16{\scriptsize $\pm$0.11} & 91.63{\scriptsize $\pm$0.23} \\
         + OT bound & 98.50{\scriptsize $\pm$0.48} & 99.27{\scriptsize $\pm$0.06} & 91.80{\scriptsize $\pm$0.25} \\
         $\Delta$ & +1.83 & +0.11 & +0.17  \\
         \bottomrule
    \end{tabular}
    }
    \label{tab:tab_reg}
\end{table*}

We further evaluate our OT regularizer on the three image classification benchmarks: Digits, MNIST, Fashion-MNIST, with MLP and CNN encoders. We fix the encoder depth to $8$ and the classifier depth to $2$, with hidden dimension $64$. To ensure stable optimization and isolate the geometric benefits of the regularizer from general underfitting, all encoders are augmented with skip connections and batch normalization. The regularizer $\mathcal{R}$ replaces the two non-smooth components of the original bound with smooth approximations: the Wasserstein distance is approximated by the entropy-regularized Sinkhorn algorithm~\citep{cuturi2013sinkhorn}, and the maximum over pairwise change rates by a temperature-smoothed log-sum-exp. The regularization weight $\lambda$ is selected from $\{1e^{-1}, 1e^{-2}, 1e^{-3}, 1e^{-4}\}$ using a validation set. We report mean test accuracy and standard deviation over $5$ random seeds.

\cref{tab:tab_reg} reports the test accuracy with and without the regularizer. Adding the regularizer yields consistent improvements over the baseline in most settings, with the largest gains on CNN.

\section{Experimental details}

\subsection{Details on GNN node classification experiments}
\label{apdx_sec:exp_gnn}

\paragraph{Rank correlation experiments}
 For~\cref{subsec:exp_corr_gnn}, we follow the standard transductive setting~\citep{tang2023towards}: for each run, we randomly select $30\%$ of the nodes as the training set and use the remaining $70\%$ as the test set. We use Adam with learning rate $0.01$, hidden dimension $64$, and train for $500$ epochs. We generate $18$ configurations per (architecture, dataset) pair by varying the encoder $\phi$ over GNNs with depth $\{1, 2, 4, 8, 16, 32\}$ and the classifier $f$ over one-, two-, and four-layer MLPs with ReLU activations. For each configuration, we compute our bound and the empirical generalization gap, and take the rank correlation over the $18$ configurations.

\paragraph{Regularization experiments}
For all results in~\cref{tab:regularization_result}, we fix the encoder depth to $8$ and the classifier depth to $2$ across all datasets. To ensure stable optimization and isolate the geometric benefits of the regularizer from general underfitting, all encoders are augmented with skip connections and batch normalization. The regularizer $\mathcal{R}$ replaces the two non-smooth pieces of the original bound with smooth approximations: the Wasserstein distance is approximated by the entropy-regularized Sinkhorn algorithm~\citep{cuturi2013sinkhorn}, and the maximum over pairwise change rates is replaced by a temperature-smoothed log-sum-exp. The expectation over $\pi'$ in $\mathcal{R}$ is estimated with a single sampled permutation ($T = 1$) at each training step. This differs from the bound estimation in \cref{subsec:exp_corr_gnn}, where we use $T = 4$. The regularization weight $\lambda$ is selected from $\{10^{-1}, 10^{-2}, 10^{-3}, 10^{-4}, 10^{-5}\}$ using a validation set. We report mean test accuracy and standard deviation over $5$ random seeds.

\subsection{Details on depth analysis experiments}\label{apdx_sec:exp_depth}
For the depth analysis in \cref{sec:depth-ana}, we evaluate two encoders, SGC~\citep{wu2019simplifying} and GCN~\citep{kipf2016semi}, on four datasets: Cora, CiteSeer, Roman-empire, and Amazon-ratings, covering both homophilic and heterophilic regimes. We train each encoder with depth $\{1,4,8,12,16,20,24,28,32\}$  with a fixed two-layer MLP classifier with ReLU activations. The hidden dimension is set to $64$. Each configuration is trained with Adam (learning rate $0.01$, no weight decay) for $500$ epochs on a $30\%/70\%$ train/test transductive split, with no validation set; reported quantities are means and standard deviations across $5$ random splits. 

\subsection{Details on image classification experiments}
\label{apdx_sec:exp-image}

We use CIFAR-10 and SVHN for both supervised and semi-supervised experiments. For the semi-supervised setting, we treat the entire dataset as a single pool, partitioned into a labeled subset and an unlabeled subset to follow the transductive setting. We use only $2\%$ of the data labeled (1{,}200 labels for CIFAR-10 and 2{,}000 labels for SVHN), sampled in a class-balanced manner so that each class contains the same number of labeled samples. The remaining samples are used as unlabeled data. We consider LP-DeepSSL~\citep{iscen2019label}, Mean Teacher~\citep{tarvainen2017mean}, FixMatch~\citep{sohn2020fixmatch}, and SimMatch~\citep{zheng2022simmatch}, each with the base architecture and default training hyperparameters (optimizer, learning rate, batch size) reported in its original paper. We use the same $2\%$ labeled budget across all semi-supervised algorithms. For the supervised setting, we use a CNN. 

\paragraph{Rank correlation experiments}
 For~\cref{subsec:corr_image}, we generate $24$ configurations per (algorithm, dataset) pair by varying three architectural hyperparameters: encoder depth, width, and classifier depth ($c_{\text{depth}} \in \{1, 2, 4\}$, the number of classifier layers). For LP-DeepSSL and Mean Teacher, the encoder is a CNN with $n_{\text{blocks}} \in \{1, 2, 3, 4\}$ blocks and $\text{base\_channels} \in \{32, 64\}$; for FixMatch and SimMatch, it is a WideResNet with depth $\in \{10, 16, 22, 28\}$ and width $\in \{1, 2\}$. For each configuration, we compute our bound and the empirical generalization gap, defined as the absolute difference between labeled and unlabeled error, and take the Spearman rank correlation over the $24$ configurations.
 

\paragraph{Regularization experiments}
For~\cref{subsec:reg_image}, we use the base architecture of each algorithm without varying configurations. In the semi-supervised setting, each mini-batch contains only a handful of labeled samples per class, which is too few for a stable per-class $\mathcal{W}_1$ estimate. To approximate this, we maintain a buffer of (detached) labeled features from the most recent four mini-batches and concatenate it with the current batch when computing $\mathcal{R}$. Following the convention of the Unified SSL Benchmark~\citep{wang2022usb}, we report the best performance as the median accuracy over the last several checkpoints, averaged over $5$ runs.

\subsection{Compute resources}
\label{apdx_sec:compute}

Most experiments were conducted on NVIDIA RTX 3090, A5000, and A6000 GPUs. Each individual run completes within a day on a single GPU, and no experiment requires multi-GPU training.

\section{Dataset Statistics}\label{apdx:datastat}
\label{appendix:dataset_statistics}
We provide detailed statistics and explanations about the graph datasets used for the experiments in \Cref{tab:appendix_dataset_statistics}.
\begin{table*}[ht]
\begin{center}
\caption{Statistics of the datasets utilized in the experiments.}
\resizebox{0.6\linewidth}{!}{
\setlength{\tabcolsep}{10pt}
\begin{tabular}{crrrrr}
\toprule
Dataset & \# nodes & \# edges & \# features & \# classes \\
\midrule
Cora          & 2,708   & 5,278     & 1,433 & 7 \\
CiteSeer      & 3,327   & 4,552     & 3,703 & 6 \\
PubMed        & 19,717  & 44,324    & 500   & 3 \\
Computers     & 13,752  & 245,861   & 767   & 10 \\
Photo         & 7,650   & 119,081   & 745   & 8 \\
\midrule[0.8pt] 
Squirrel      & 2,223   & 46,998    & 2,089 & 5 \\
Chameleon     & 890     & 8,854     & 2,325 & 5 \\
Roman-empire  & 22,662  & 32,927    & 300   & 18\\
Amazon-ratings& 24,492  & 93,050    & 300   & 5 \\
\bottomrule
\end{tabular}
}
\label{tab:appendix_dataset_statistics}
\end{center}
\end{table*}
\paragraph{Cora, CiteSeer, and PubMed}
Each node represents a paper, and an edge indicates a reference relationship between two papers. The task is to predict the research subjects of the papers.
\paragraph{Computers and Photo}
Each node represents a product, and an edge indicates a high frequency of concurrent purchases of the two products. The task is to predict the product category.
\paragraph{Squirrel and Chameleon}
Each node represents a Wikipedia page, and an edge indicates a link between two pages. The task is to predict the monthly traffic for each page. We use the classification version of the dataset, where labels are converted by dividing monthly traffic into five bins. We adopted the filtering process to prevent train-test data
leakage as recommended by \citep{platonov2023critical}.
\paragraph{Roman-empire}
Each node represents a word extracted from the English Wikipedia article on the Roman Empire, and an edge indicates a grammatical or sequential relationship between words. The task is to predict the part-of-speech tag of each word.
\paragraph{Amazon-ratings}
Each node represents a product from the Amazon co-purchasing network, and an edge indicates a frequent co-purchase between two products. The task is to predict the product category based on user co-purchasing patterns.

\newpage

\section*{NeurIPS Paper Checklist}

\begin{enumerate}

\item {\bf Claims}
    \item[] Question: Do the main claims made in the abstract and introduction accurately reflect the paper's contributions and scope?
    \item[] Answer: \answerYes{} 
    \item[] Justification: We accurately write main claims and contributions in both the abstract and introduction.

\item {\bf Limitations}
    \item[] Question: Does the paper discuss the limitations of the work performed by the authors?
    \item[] Answer: \answerYes{} 
    \item[] Justification: We discuss settings where our approach does not apply to in the section of the appendix.

\item {\bf Theory assumptions and proofs}
    \item[] Question: For each theoretical result, does the paper provide the full set of assumptions and a complete (and correct) proof?
    \item[] Answer:\answerYes{} 
    \item[] Justification: We provide full assumptions and complete proofs for all theoretical results in the appendix.

    \item {\bf Experimental result reproducibility}
    \item[] Question: Does the paper fully disclose all the information needed to reproduce the main experimental results of the paper to the extent that it affects the main claims and/or conclusions of the paper (regardless of whether the code and data are provided or not)?
    \item[] Answer: \answerYes{} 
    \item[] Justification: We provide all experimental details in the appendix.

\item {\bf Open access to data and code}
    \item[] Question: Does the paper provide open access to the data and code, with sufficient instructions to faithfully reproduce the main experimental results, as described in supplemental material?
    \item[] Answer: \answerYes{} 
    \item[] Justification: We use publicly available datasets and add experimental details needed to reproduce in Appendix. We will release the code upon acceptance.

\item {\bf Experimental setting/details}
    \item[] Question: Does the paper specify all the training and test details (e.g., data splits, hyperparameters, how they were chosen, type of optimizer) necessary to understand the results?
    \item[] Answer:\answerYes{} 
    \item[] Justification: We specify all training and test details in the experiment section and the appendix.

\item {\bf Experiment statistical significance}
    \item[] Question: Does the paper report error bars suitably and correctly defined or other appropriate information about the statistical significance of the experiments?
    \item[] Answer: \answerYes{}{} 
    \item[] Justification: We report mean and standard deviation over multiple random seeds in tables and figures.

\item {\bf Experiments compute resources}
    \item[] Question: For each experiment, does the paper provide sufficient information on the computer resources (type of compute workers, memory, time of execution) needed to reproduce the experiments?
    \item[] Answer:\answerYes{} 
    \item[] Justification: We explicitly indicate the type of GPU for experiments in Appendix.
    
\item {\bf Code of ethics}
    \item[] Question: Does the research conducted in the paper conform, in every respect, with the NeurIPS Code of Ethics \url{https://neurips.cc/public/EthicsGuidelines}?
    \item[] Answer: \answerYes{} 
    \item[] Justification: Our research conforms with the NeurIPS Code of Ethics in every respect.

\item {\bf Broader impacts}
    \item[] Question: Does the paper discuss both potential positive societal impacts and negative societal impacts of the work performed?
    \item[] Answer: \answerNA{} 
    \item[] Justification: We believe our experiments are based on theory we propose, so there is no societal impact of the work performed.
    
\item {\bf Safeguards}
    \item[] Question: Does the paper describe safeguards that have been put in place for responsible release of data or models that have a high risk for misuse (e.g., pre-trained language models, image generators, or scraped datasets)?
    \item[] Answer: \answerNA{} 
    \item[] Justification: The paper does not release any data or models that pose a high risk for misuse.

\item {\bf Licenses for existing assets}
    \item[] Question: Are the creators or original owners of assets (e.g., code, data, models), used in the paper, properly credited and are the license and terms of use explicitly mentioned and properly respected?
    \item[] Answer: \answerYes{} 
    \item[] Justification: All datasets and model architectures used in our experiments are publicly available and properly cited.

\item {\bf New assets}
    \item[] Question: Are new assets introduced in the paper well documented and is the documentation provided alongside the assets?
    \item[] Answer: \answerNA{} 
    \item[] Justification: The paper does not introduce or release new assets.

\item {\bf Crowdsourcing and research with human subjects}
    \item[] Question: For crowdsourcing experiments and research with human subjects, does the paper include the full text of instructions given to participants and screenshots, if applicable, as well as details about compensation (if any)? 
    \item[] Answer: \answerNA{} 
    \item[] Justification: The paper does not involve crowdsourcing or research with human subjects.

\item {\bf Institutional review board (IRB) approvals or equivalent for research with human subjects}
    \item[] Question: Does the paper describe potential risks incurred by study participants, whether such risks were disclosed to the subjects, and whether Institutional Review Board (IRB) approvals (or an equivalent approval/review based on the requirements of your country or institution) were obtained?
    \item[] Answer: \answerNA{} 
    \item[] Justification: The paper does not involve research with human subjects.

\item {\bf Declaration of LLM usage}
    \item[] Question: Does the paper describe the usage of LLMs if it is an important, original, or non-standard component of the core methods in this research? Note that if the LLM is used only for writing, editing, or formatting purposes and does \emph{not} impact the core methodology, scientific rigor, or originality of the research, declaration is not required.
    \item[] Answer: \answerNA{} 
    \item[] Justification: LLMs were used only for writing and editing assistance, not as a component of the core methodology.

\end{enumerate}

\end{document}